\documentclass[]{article}

\usepackage[utf8]{inputenc}
\usepackage{babel}  
\usepackage[T1]{fontenc}  
\usepackage{booktabs}       
\usepackage{amsfonts}     
\usepackage{nicefrac}       
\usepackage{microtype}      
\usepackage{xcolor}         
\usepackage{caption}
\usepackage{float}
\usepackage{graphicx}
\usepackage{subcaption}
\usepackage{amssymb}
\usepackage{fancyhdr}
\usepackage{bookmark}
\usepackage{bbm}
\usepackage{boldline}

\usepackage[utf8]{inputenc}   
\usepackage[T1]{fontenc}      

\usepackage{comment}
\usepackage{geometry}
\usepackage{natbib}

\usepackage[tbtags]{amsmath}
\usepackage{amsthm}
\allowdisplaybreaks
\usepackage{amssymb,mathrsfs}
\usepackage{nccmath}
\usepackage{amsfonts}
\usepackage{upgreek}
\usepackage{xspace}

\usepackage{graphicx}
\usepackage{color}
\usepackage{algorithm, algpseudocode}

\usepackage{stmaryrd}
\usepackage[inline]{enumitem}
\usepackage{url}

\usepackage{tikz}
\usetikzlibrary{calc}

\usepackage{pgfplots}
\usepackage{xcolor}
\usepackage{bbm}
\usepackage{ifthen}
\usepackage{xargs}
\usepackage[textwidth=1.8cm]{todonotes}

\usepackage{aliascnt}
\usepackage{cleveref}
\usepackage{autonum}
\makeatletter
\newtheorem{theorem}{Theorem}

\newtheorem*{lemma_nonumber*}{Lemma}

\newaliascnt{lemma}{theorem}
\newtheorem{lemma}[lemma]{Lemma}
\aliascntresetthe{lemma}

\newaliascnt{corollary}{theorem}

\aliascntresetthe{corollary}

\newaliascnt{proposition}{theorem}
\newtheorem{proposition}[proposition]{Proposition}
\aliascntresetthe{proposition}

\newaliascnt{definition}{theorem}
\newtheorem{definition}[definition]{Definition}
\aliascntresetthe{definition}

\newaliascnt{remark}{theorem}

\aliascntresetthe{remark}

\newtheorem{example}[theorem]{Example}

\newaliascnt{lemmaB}{lemma}
\newtheorem{lemmaB}[lemmaB]{Lemma}
\aliascntresetthe{lemmaB}

\newaliascnt{lemmaC}{lemma}

\aliascntresetthe{lemmaC}

\newaliascnt{propositionC}{proposition}

\aliascntresetthe{propositionC}

\crefformat{assumption}{{\textbf{A}}#2#1#3}

\newtheorem{assumptionF}{\textbf{F}\hspace{-3pt}}
\crefformat{assumptionF}{{\textbf{F}}#2#1#3}

\Crefname{assumptionB}{\textbf{B}\hspace{-3pt}}{\textbf{B}\hspace{-3pt}}
\crefname{assumptionB}{\textbf{B}}{\textbf{B}}

\Crefname{assumptionC}{\textbf{C}\hspace{-3pt}}{\textbf{C}\hspace{-3pt}}
\crefname{assumptionC}{\textbf{C}}{\textbf{C}}

\Crefname{assumptionH}{\textbf{H}\hspace{-3pt}}{\textbf{H}\hspace{-3pt}}
\crefname{assumptionH}{\textbf{H}}{\textbf{H}}

\Crefname{assumptionT}{\textbf{T}\hspace{-3pt}}{\textbf{T}\hspace{-3pt}}
\crefname{assumptionT}{\textbf{T}}{\textbf{T}}

\Crefname{assumptionT}{\textbf{T}\hspace{-3pt}}{\textbf{T}\hspace{-3pt}}
\crefname{assumptionT}{\textbf{T}}{\textbf{T}}

\Crefname{assumptionL}{\textbf{L}\hspace{-3pt}}{\textbf{L}\hspace{-3pt}}
\crefname{assumptionL}{\textbf{L}}{\textbf{L}}

\Crefname{assumptionQ}{\textbf{Q}\hspace{-3pt}}{\textbf{Q}\hspace{-3pt}}
\crefname{assumptionQ}{\textbf{Q}}{\textbf{Q}}

\Crefname{assumptionAR}{\textbf{AR}\hspace{-3pt}}{\textbf{AR}\hspace{-3pt}}
\crefname{assumptionAR}{\textbf{AR}}{\textbf{AR}}

\newtheoremstyle{cited}
  {3pt}
  {3pt}
  {\itshape}
  {}
  {\bfseries}
  {.}
  {.5em}
  {\thmname{#1} \thmnumber{#2} \thmnote{\normalfont#3}}

\theoremstyle{cited}

\newtheorem{citedlemma}[lemma]{Lemma}

\usepackage{bm}
\usepackage{wrapfig}

\def\d{d}
\def\di{{d'}}
\def\dii{{d''}}

\newcommand{\tta}{\mathtt{A}}

\newcommand{\Capprox}{\tta}

\newcommandx\ctun[1][1=T]{\Capprox_{#1,1}}

\newcommandx{\expec}[2]{{\mathbb E}\left[#1 \middle \vert #2  \right]} 

\newcommandx{\norm}[2][1=]{\ifthenelse{\equal{#1}{}}{\left\Vert #2 \right\Vert}{\left\Vert #2 \right\Vert^{#1}}}
\newcommandx{\normLigne}[2][1=]{\ifthenelse{\equal{#1}{}}{\Vert #2 \Vert}{\Vert #2\Vert^{#1}}}

\def\bfc{\mathbf{c}}

\def\msa{\mathsf{A}}

\def\rset{\mathbb{R}}

\def\rmd{\mathrm{d}}

\newcommandx{\functionspace}[2][1=+]{\mathbb{F}_{#1}(#2)}

\newcommand{\argmin}{\operatorname*{arg\,min}}

\newcommandx{\VarDeux}[3][3=]{\operatorname{Var}^{#3}_{#1}\left\{#2 \right\}}

\newcommand{\LeftEqNo}{\let\veqno\@@leqno}

\newcommand{\N}{\ensuremath{\mathbb{N}}}

\newcommandx{\Vnorm}[2][1=V]{\| #2 \|_{#1}}
\newcommandx{\VnormEq}[2][1=V]{\left\| #2 \right\|_{#1}}

\newcommandx\probaMarkovTilde[2][2=]
{\ifthenelse{\equal{#2}{}}{{\widetilde{\mathbb{P}}_{#1}}}{\widetilde{\mathbb{P}}_{#1}\left[ #2\right]}}

\def\eqsp{\;}

\newcommandx{\weight}[2][2=n]{\omega_{#1,#2}^N}

\newcommandx\sequence[3][2=,3=]
{\ifthenelse{\equal{#3}{}}{\ensuremath{\{ #1_{#2}\}}}{\ensuremath{\{ #1_{#2}, \eqsp #2 \in #3 \}}}}

\newcommandx\sequenceD[3][2=,3=]
{\ifthenelse{\equal{#3}{}}{\ensuremath{\{ #1_{#2}\}}}{\ensuremath{( #1)_{ #2 \in #3} }}}

\newcommandx{\sequencen}[2][2=n\in\N]{\ensuremath{\{ #1_n, \eqsp #2 \}}}
\newcommandx\sequenceDouble[4][3=,4=]
{\ifthenelse{\equal{#3}{}}{\ensuremath{\{ (#1_{#3},#2_{#3}) \}}}{\ensuremath{\{  (#1_{#3},#2_{#3}), \eqsp #3 \in #4 \}}}}
\newcommandx{\sequencenDouble}[3][3=n\in\N]{\ensuremath{\{ (#1_{n},#2_{n}), \eqsp #3 \}}}

\newcommand{\opnorm}[1]{{\left\vert\kern-0.25ex\left\vert\kern-0.25ex\left\vert #1
    \right\vert\kern-0.25ex\right\vert\kern-0.25ex\right\vert}}

\def\Id{\operatorname{Id}}

\newcommandx{\CPE}[3][1=]{{\mathbb E}_{#1}\left[#2 \middle \vert #3  \right]} 
\newcommandx{\CPELigne}[3][1=]{{\mathbb E}_{#1}[#2  \vert #3  ]} 
\newcommandx{\CPEsq}[3][1=]{{\mathbb{E}^{1/2}}_{#1}\left[#2 \middle \vert #3  \right]} 
\newcommandx{\CPVar}[3][1=]{\mathrm{Var}^{#3}_{#1}\left\{ #2 \right\}}
\newcommand{\CPP}[3][]
{\ifthenelse{\equal{#1}{}}{{\mathbb P}\left(\left. #2 \, \right| #3 \right)}{{\mathbb P}_{#1}\left(\left. #2 \, \right | #3 \right)}}

\newcommandx{\osc}[2][1=]{\mathrm{osc}_{#1}(#2)}

\def\Id{\operatorname{Id}}

\def\a{a}

\newcommand{\ensembleLigne}[2]{\{#1\,:\eqsp #2\}}

\newcommand\coupling[2]{\Gamma(\mu,\nu)}

\newcommandx{\KL}[2]{\operatorname{KL}\left( #1 | #2 \right)}
\newcommandx{\KLsqrt}[2]{\operatorname{KL}^{1/2}\left( #1 | #2 \right)}
\newcommandx{\Jef}[2]{\operatorname{J}\left( #1 , #2 \right)}
\newcommandx{\JefLigne}[2]{\operatorname{J}( #1 , #2 )}
\newcommandx{\KLLigne}[2]{\operatorname{KL}( #1 | #2 )}
\newcommandx{\KLLignesqrt}[2]{\operatorname{KL}^{1/2}( #1 | #2 )}

\def\gaStep
\def\QKer{Q}

\def\distance{\mathbf{d}}
\newcommandx{\wasserstein}[3][1=\distance,3=]{\mathbf{W}_{#1}^{#3}\left(#2\right)}
\newcommandx{\wassersteinLigne}[3][1=\distance,3=]{\mathbf{W}_{#1}^{#3}(#2)}
\newcommandx{\wassersteinD}[1][1=\distance]{\mathbf{W}_{#1}}
\newcommandx{\wassersteinDLigne}[1][1=\distance]{\mathbf{W}_{#1}}

\def\sigmaD{\sigma^2}

\newcommandx{\phibfs}[1][1=]{\pmb{\varphi}_{\sigmaD_{#1}}}

\newcommandx\sequenceg[3][2=,3=]
{\ifthenelse{\equal{#3}{}}{\ensuremath{( #1_{#2})}}{\ensuremath{( #1_{#2})_{ #2 \geq #3}}}}

\newcommandx{\distV}[1][1=\bfc]{\mathbf{W}_{#1}}
\newcommandx{\distVdeux}[1][1=W_2]{\mathbf{d}_{#1}}

\usepackage{comment}
\usepackage{hyperref}
\hypersetup{colorlinks,citecolor=blue,linkcolor=blue}

\newcommand{\spa}[1]{\mathcal{#1}}
\newcommand{\op}{\pi}

\title{Gromov-Wasserstein-like Distances in the Gaussian Mixture Models Space}

\author{Antoine Salmona$^1$,  Julie Delon$^{2}$, Agnès Desolneux$^1$.}
\date{%
    $^1$ ENS Paris-Saclay, CNRS, Centre Borelli UMR 9010\\%
    $^2$ Universit\'e de Paris, CNRS, MAP5 UMR 8145 and Institut Universitaire de France \\[2ex]%
    \today
}

\pgfplotsset{compat=1.18}

\newcounter{Hequation}

\makeatletter
\g@addto@macro\equation{\stepcounter{Hequation}}
\makeatother

\begin{document}

\maketitle

\begin{abstract}
The Gromov-Wasserstein (GW) distance is frequently used in machine learning to compare distributions across distinct metric spaces. Despite its utility, it remains computationally intensive, especially for large-scale problems. 
 Recently, a novel Wasserstein distance specifically tailored for Gaussian mixture models (written GMMs in the paper for the sake of brevity) and known as $MW_2$  (mixture Wasserstein) has been introduced by several authors.  
  In scenarios where data exhibit clustering, this approach simplifies to a small-scale discrete optimal transport problem, which complexity depends solely on the number of Gaussian components in the GMMs. 
This paper aims to incorporate invariance properties into $MW_2$. This is done by introducing new Gromov-type distances, designed to be isometry-invariant in Euclidean spaces and  applicable for comparing GMMs across different dimensional spaces.
Our first contribution is the Mixture Gromov Wasserstein distance ($MGW_2$), which can be viewed as a 'Gromovized' version of 
 $MW_2$. This new distance has a straightforward discrete formulation, making it highly efficient for estimating distances between GMMs in practical applications. To facilitate the derivation of a transport plan between GMMs, we present a second distance, the Embedded Wasserstein distance ($EW_2$). This distance turns out to be closely related to several recent alternatives to Gromov-Wasserstein. We show that $EW_2$ can be adapted to derive a distance as well as optimal transportation plans between GMMs. We demonstrate the efficiency of these newly proposed distances on medium to large-scale problems, including shape matching and hyperspectral image color transfer.

\end{abstract}

\section{Introduction}

The goal of optimal transport (OT) theory is to design meaningful ways to compare probability distributions. It provides 
very useful mathematical tools for diverse imaging sciences and machine learning tasks including generative modeling 
\cite[]{arjovsky2017wasserstein,genevay2017learning,tolstikhin2018wasserstein}, domain adaptation \cite[]{courty2016optimal}, image processing \cite[]{rabin2012wasserstein,rabin2014adaptive},
and embedding learning \cite[]{courty2018learning,xu2018distilled}. For two probability distributions 
$ \mu $ and $ \nu $, respectively on two Polish (i.e complete, separable, metrizable) spaces  $ \spa{X} $ and $ \spa{Y} $,  and given a lower semi-continuous function 
$ c \colon \spa{X} \times \spa{Y} \rightarrow \rset_+ $ called \emph{cost}, optimal transport in its most classic form aims at solving the following
optimization problem,
\begin{equation}\label{eq:generalOT}
\inf_{\op \in \Pi(\mu,\nu)} \int_{\spa{X} \times \spa{Y}} c(x,y) \rmd \op(x,y) \eqsp,
\end{equation} 
where $ \Pi(\mu,\nu) $ is the set of probability measures on $ \spa{X} \times \spa{Y} $ with marginals $ \mu $ and $ \nu $. 
When $ \spa{Y} $ is equal to $ \spa{X} $, 
the choice of cost $ c_p(x,y) = d_{\spa{X}}(x,y)^p $, with $ p \geq 1 $ and $ d_{\spa{X}} $ 
the metric of the space $ \spa{X} $, induces a distance between probability distributions with
finite $ p $-th moments, called the Wasserstein distance $ W_p $. In the discrete setting, Problem \eqref{eq:generalOT} becomes
\begin{equation}
\inf_{\omega \in \Pi(a,b)} \sum_{k,l}C_{k,l}\omega_{k,l} \eqsp,
\end{equation}
where $  a = (a_1,\dots,a_m)^T $ and $ b = (b_1,\dots,b_n)^T $ are 
respectively in the $ \rset^m $ and $ \rset^n $  simplexes $ \Delta_m $ and $ \Delta_n $,\footnote{The
simplex $ \Delta_m $ is the subset of $ \rset^m $ of $ x = (x_1,\dots,x_m)^T  $ such that for all $ 1\le k \leq m $, $ x_k \geq 0 $, and $ \sum_{k=1}^m x_k = 1 $.}
$ \Pi(a,b) = \ensembleLigne{ \omega \in \Delta_{m \times n} } {  \omega\mathbbm{1}_n = a \text{ and }  \omega^T \mathbbm{1}_m = b }  $  and $ C $ is a non-negative matrix of size $ m \times n $, called cost matrix.

Optimal transport is known to be computationally challenging. Between discrete distributions, its computation involves solving
a linear program that rapidly becomes costly as soon as the number of points is moderately large. Between two sets of $ n $ points, its computation complexity is
 $ O(n^3\mathrm{log}(n)) $ \cite[]{seguy2017large}, which compromises its usability for settings with more than a few tens of thousand of points. 
To lighten OT computational cost, a large number of
works have developped efficient computational tools. 
In particular,  \cite{cuturi2013sinkhorn} proposes to solve an entropic regularized OT problem using the Sinkhorn-Knopp algorithm \cite[]{sinkhorn1967concerning}, reducing the 
cost of the problem to $ O(n^2) $. Over the last past years, a large body of works have focused 
on speeding up the Sinkhorn-Knopp algorithm, building mostly on 
diverse low-rank approximations \cite[]{solomon2015convolutional,altschuler2018approximating,altschuler2019massively,forrow2019statistical,scetbon2020linear,scetbon2021low}.
These approaches have helped to reduce the computational cost of the problem from cubic (for the non-regularized problem)
to linear complexity. Another type of commonly used solvers are building
on sliced mechanisms \cite[]{rabin2012wasserstein,kolouri2019generalized}. These solvers average Wasserstein distances between several one dimensional projections of the high-dimensional distributions, leveraging the 
fact that the OT problem between one-dimentional distributions can be 
solved using a simple sorting algorithm.
Alternatively, \cite{delon2020wasserstein} have proposed an OT distance between \emph{Gaussian mixture models} (GMM), called Mixture Wasserstein (MW), where
the admissible  couplings $ \op $  are themselves constrained to be GMMs. 
They demonstrated that this specific continuous OT problem could be equivalently reformulated into a discrete version (which had been also proposed independently by ~\cite{chen2018optimal}): for two GMM with respectively $K_0$ and $K_1$ components, solving this formulation boils down to solve a small scale $K_0\times K_1$ discrete OT problem. This distance can be applied to real data by first fitting GMMs on each distribution, making it particularly suited for scenarios where a clustering structure already exists in the data. The main advantage of this approach is that its computational cost arises almost exclusively from fitting the GMMs to the data, since the complexity of the composite OT problem depends neither on the dimension nor on the number of points, but solely on the number of components in the GMMs.
This approach offers a scalable and computationally efficient OT distance, which has been used for instance for texture synthesis \cite[]{leclaire2022optimal}, 
evaluating generative models \cite[]{luzi2023evaluating}, 
Gaussian Mixture reduction \cite[]{zhang2020unified} or approximate Bayesian computation~\cite[]{forbes2021approximate}. 

One weakness of the classical optimal transport approach lies in the fact that it implicitly assumes that the spaces $ \spa{X} $ 
and $ \spa{Y} $ are \emph{comparable}, i.e. that there exists 
a relevant cost function $ c \colon \spa{X} \times \spa{Y} \rightarrow \rset_+ $ 
to compare them. Yet, this assumption is not always 
verified. For instance, if $ \spa{X} = \rset^\d $ and  $ \spa{Y} = \rset^\di $
with $ \d \neq \di $, the definition of a meaningful cost function
$ c \colon \rset^\d \times \rset^\di \rightarrow \rset_+ $ is not straightforward. 
Furthermore, some applications such as shape matching require having
an OT distance that is invariant to a given family of transformations, such
as translations or rotations, or more generally to \emph{isometries}\footnote{We say that $ \phi \colon \spa{X} \rightarrow \spa{Y} $ is an isometry if for all $ (x,x') \in \spa{X}^2 $, $ d_{\spa{Y}}(\phi(x),\phi(x'))=d_{\spa{X}}(x,x') $.}. Even if 
the two distributions involved in these applications do live in the same 
ground space, it is not straightforward to design a cost function such 
that the resulting OT distance will be invariant to these families of transformations. 
To overcome those limitations, several non-convex variants of Problem \eqref{eq:generalOT} have been proposed 
\cite[]{cohen1999earth,pele2013tangent,alvarez2019towards,cai2022distances}. Among these, the Gromov-Wasserstein (GW)~\cite[]{memoli} distance is perhaps the most frequently utilized, recently gaining significant attention for the versatility it provides.
Indeed, it only requires modeling topological aspects of the distributions within each domain to compare them without having 
to specify first a subset of invariances nor to design a relevant cost function between the spaces the distributions lie on.
The GW problem between two measures $\mu$ and $\nu$ living respectively on $\spa{X}$ and $\spa{Y}$ aims at solving
\begin{equation}\label{eq:gromov_intro}
 \inf_{\op \in \Pi(\mu,\nu)}  \int_{\spa{X} \times \spa{Y}} \int_{\spa{X} \times \spa{Y}}  | c_{\spa{X}}(x,x') - c_{\spa{Y}}(y,y')|^p \rmd \op(x,y) \rmd \op(x',y'),
\end{equation}
where $ c_{\spa{X}} : \spa{X} \times \spa{X} \shortrightarrow \rset $ and $ c_{\spa{Y}} : \spa{Y} \times \spa{Y} \shortrightarrow \rset $ are two cost functions. 
Since this optimization problem only requires to define cost functions in each respective space, it remains very versatile and can be defined with very little assumptions on the spaces  $\spa{X}$ and $\spa{Y}$ . This approach has been applied to shape matching \cite[]{memoli2009spectral}, or more generally to correspondence problems \cite[]{solomon2016entropic}, 
word embedding \cite[]{alvarez2018gromov}, graph classification \cite[]{titouan2019optimal}, graph prediction \cite[]{brogat2022learning}, and 
generative modeling \cite[]{bunne2019learning}.

Computationally speaking, the Gromov-Wasserstein problem is known 
to be much more costly to solve than the classic linear OT problem. Indeed, the problem is non convex, quadratic with respect to $\op$ and known to be NP-hard. 
One possible approach to solve GW consists in linearizing the cost and to solve iteratively several classic OT problems. Entropic regularization of GW has also been proposed in \cite[]{peyre2016gromov,solomon2016entropic} and results in a still non convex problem which can be solved by a projected gradient algorithm, where each projection is itself an entropic linear optimal transport problem. In recent years, several practical approximations of GW have been proposed in the literature to reduce its computational complexity and solve it efficiently, either
through quantization of input measures \cite[]{chowdhury2021quantized}, recursive clustering approches \cite[]{xu2019scalable,blumberg2020mrec}, or using 
a minibatch scheme \cite[]{fatras2021minibatch}. Specifically to the Euclidean setting, \cite{titouan2019sliced} has introduced 
a solver buiding on a sliced mechanism, and leveraging the observation that the GW problem seems most of the time easy to solve between one-dimensional distributions. 
More recently, \cite{scetbon2022linear} have shown that the low-rank approximations 
used to speed-up the Sinkhorn-Knopp algorithm were particularly suited for the regularized GW problem, resulting in a much more 
computationally efficient solver.
In this work, we propose to build on the ideas of~\cite{delon2020wasserstein} in order to construct OT distances between GMMs that are 
invariant to isometries and that stay relevant between GMMs of different dimensions. These distances share similarities with the one defined in~\cite[]{chowdhury2021quantized}, since they rely on a form of quantization of the original data through the GMM representation. One of these distances is a ``Gromovization'' of the Mixture Wasserstein distance, that we call MGW. 
 We will see that the structured representation of MGW makes it very robust in practice,  and permits to design an efficient and scalable solver using a fixed small number of Gaussian components, while keeping competitive precision and running times (when compared to the state-of-the-art methods described above) when the number of points of the underlying data increases. 

\paragraph{Contributions of the paper.}
In this paper, we introduce two Gromov-Wasserstein type OT distances
between GMMs that are designed to be invariant (at least) to isometries. More precisely, we introduce in \Cref{sec:mgw2} 
a natural Gromov version of the distance introduced by \cite{chen2018optimal} and \cite{delon2020wasserstein}, 
that we call MGW for \emph{Mixture Gromov Wasserstein}. 
This distance can be used for applications which only 
require to evaluate how far the distributions are from each other, 
without having to identify correspondences between points. 
However, this formulation does not directly allow to derive 
an optimal transportation plan between the points. To design a
way to define such a transportation plan, we define in \Cref{sec:ew2} 
another distance that we call EW for \emph{Embedded Wasserstein}. This latter
turns out to be closely related to the Gromov-Wasserstein distance and 
coincides with the OT distance introduced by \cite{alvarez2019towards}.
We show that EW  can be adapted to derive a distance and optimal transportation plans between GMMs and we then define a heuristic transportation plan for  MGW  by analogy with EW. Finally, in \Cref{sec:expe}, we illustrate the pratical use of our distances
on medium-to-large scale problems such as shape matching and hyperspectral image color transfer and we compare the performance of our methods with other recent GW based approaches, both on assessing distances between clouds on points and drawing 
correspondences between points. All the proofs are postponed to the appendix.

\section*{Notation}
We define in the following some of the notation that will be used in the paper. 
\begin{itemize}[itemsep=0pt,parsep=0pt]
\item $ \langle x,x' \rangle_\d $ stands for the Euclidean inner-product in $ \rset^\d $ between $ x $ and $ x' $.
We will use the notation $ \langle x,x' \rangle $ when the dimension is clear and unambiguous.
\item $ \|x\|_{\rset^\d} $ stands for the Euclidean norm of $ x \in \rset^\d $. We will  use the notation $ \|x\|$ when the dimension is clear and unambiguous.
\item $ \text{tr}(M) $ denotes the trace of a matrix $ M $. 
\item $ \|M\|_{\spa{F}} $ stands for the Frobenius norm of a matrix $ M $, i.e. $ \|M\|_{\spa{F}} = \sqrt{\text{tr}(M^TM)} $. 
\item $ \|M\|_{*}  $ stands for the nuclear norm of a matrix $ M $, i.e. $ \|M\|_{*} = \mathrm{tr}((M^T M)^{\frac{1}{2}}) $.
\item The notation $ \boldsymbol{\sigma}(M) $ denotes the vector of singular values of the matrix $ M $. 
\item $ \Id_\d $ is the identity matrix of size $ \d \times \d $. 
\item For any  $ x \in \rset^\d $, $ \mathrm{diag}(x) $ denotes the matrix 
of size $ \d \times \d $ with diagonal vector $ x $. 
\item $\widetilde{I}_\d $ stands for any matrix of size $ \d \times \d $ of the form $ \textup{diag}((\pm 1)_{1 \leq i\leq \d}) $
\item Suppose $ \d \geq \di $. For any matrix $ M $ of 
size $ \d \times \d $, we denote by $ M^{(\di)} $ the submatrix of size $ \di \times \di $ containing the $ \di $ first rows and
the $ \di $ first columns of $ A $.
\item Let $ r \leq \d $ and $ s \leq \di $. 
For any matrix $ M $ of  size $ r \times s $, we denote by $ M^{[\d,\di]} $ the matrix
of size $d\times d'$ of the form $  \begin{pmatrix} M & 0 \\ 0 & 0 \end{pmatrix} $. When $ \d = \di $, we will write $ M^{[\d]} $.
\item We use the notation $ \mathbb{S}^\d $ for the set of symmetric matrices of size $ \d \times \d $, $ \mathbb{S}_+^\d $  the set of symmetric positive semi-definite matrices, 
and $ \mathbb{S}_{++}^\d $ the set of symmetric positive definite matrices. 
\item $ \mathbbm{1}_{\di,\d} = (1)_{\substack{1 \leq i \leq \di \\ 1 \leq j \leq \d}} $ denotes the matrix of ones with $ \di $ rows and $ \d $ columns. 
\item The notation $  X \sim \mu $ means that $ X $ is a random variable with probability distribution $ \mu $.
\item If $ \mu $ is a positive measure on $ \spa{X} $ and $ \phi \colon \spa{X} \rightarrow \spa{Y} $ is a mapping, 
$ \phi_{\#}\mu $ stands for the push-forward measure of $ \mu $ by $\phi$, 
i.e. the measure on $ \spa{Y} $ such that for any measurable set $ \msa $ of $ \spa{Y} $, $ \phi_{\#}\mu(\msa) = \mu(\phi^{-1}(\msa)) $.
\item If $ \mu $ is a positive measure on $ \spa{X} $ , $ \mathrm{supp}(\mu) $ denotes its support, i.e. 
the subset of $ \spa{X} $ defined as $ \mathrm{supp}(\mu) = \{ x \in \spa{X} \ | \ \text{for all open set } N_x \text{ such that } x \in N_x, \ \mu(N_x) > 0 \} $. \item If $ X $ and $ Y $  are random vectors on $ \rset^\d $ and $ \rset^\di $, we use the notation $ \text{Cov}(X,Y) $ for the matrix of size $ \d \times \di $ of the form  $ \mathbb{E} \left[(X - \mathbb{E}[X])(Y - \mathbb{E}[Y])^T \right] $.
\item For any positive measure $ \mu $, we denote by $ \bar{\mu} $ its associated centered measure, i.e. the measure such that if $ X \sim \mu $, 
we have $ X - \mathbb{E}_{X \sim \mu}[X] \sim \bar{\mu} $. 
\item For any $ m \in \rset^\d $ and any $ \Sigma \in \mathbb{S}^\d_+ $, we denote by $ \mathrm{N}(m,\Sigma) $
the Gaussian measure of mean $ m $ and covariance matrix $ \Sigma $. 
\item For $ x \in \spa{X} $, $ \delta_x $ denotes the Dirac distribution at $ x $. 
\end{itemize}

\section{Background : Mixture-Wasserstein and Gromov-Wasserstein-type distances}
We recall in this section the definitions and some important properties of the different OT distances used throughout the paper. For any Polish space $ \spa{X} $, we write $ \mathcal{P}(\spa{X}) $ 
the set probability measures on $ \spa{X} $. For $ \d \geq 1 $ and $ p \geq 1 $, 
the Wasserstein space $ \mathcal{W}_p(\rset^\d) $ is defined as the 
set of probability measures $ \mu $ on $ \rset^\d $ with finite moment of order $ p $, i.e. such that
\begin{equation}
\int_{\rset^m} \|x\|^p \rmd \mu(x) < +\infty \eqsp,
\end{equation}
with $ \|.\| $ being the Euclidean norm on $ \rset^\d $.

\subsection{Mixture-Wasserstein distance between GMMs}
We present here the distance introduced in  \cite{delon2020wasserstein}, 
as well as some results that will be useful in the rest of the paper.
We denote $ GMM_K(\rset^\d) $ 
the set of Gaussian mixtures on $ \rset^\d $  with less than $ K $ components, i.e. 
the set of measures in $ \mathcal{P}(\rset^\d) $  which can be written 
\begin{equation}
\mu = \textstyle{\sum\limits_{k=1}^{K'} \a_k\mu_k \eqsp,}
\end{equation}
where $ K' \leq K $, $ \a = (\a_1,\dots,\a_{K'})^T $ is in  $ \Delta_{K'} $, and $ \{\mu_k\}_k $ is a family 
of pairwise distinct Gaussian distributions, each of mean $ m_k \in \rset^\d $ 
and covariance matrix $ \Sigma_k \in \mathbb{S}^\d_{+} $. Again, 
to avoid degeneracy issues where locations with no mass are accounted for, we will assume
that the elements of $ \a $ are all positive.
The set of all finite Gaussian mixture distributions on $ \rset^\d $  is then written
\begin{equation}
\textstyle{GMM_\infty(\rset^\d) = \underset{K \geq 0}{\bigcup}GMM_K(\rset^\d) \eqsp. }
\end{equation}
Note that the condition that the Gaussian components are pairwise 
distinct ensures the identifiability of the elements of $ GMM_\infty(\rset^\d) $ \cite[]{yakowitz1968identifiability},
in the sense that two GMMs $ \mu = \sum^K_k a_k \mu_k $ and $ \nu = \sum^L_l b_l\nu_l $
are equal if and only if $ K = L $, and we can reorder the indices such that 
for all $ k $, $ a_k = b_k $ and $ \mu_k = \nu_k $. It can been shown that 
$ GMM_\infty(\rset^\d) $ is dense in $ \mathcal{W}_p(\rset^\d) $ for the metric $ W_p $, 
meaning that any measure in $ \mathcal{W}_p(\rset^\d) $ 
can be approximated with any precision for the distance $ W_p $ by 
a finite Gaussian mixture distribution. Let $ \mu \in GMM_K(\rset^\d) $ 
and $ \nu \in GMM_L(\rset^\d) $. 
The Mixture-Wasserstein distance of order $ 2 $ is defined as 
\begin{equation}\label{eq:mw2}\tag{$MW_2$}
MW_2(\mu,\nu) = \left(\inf_{\op \in \Pi(\mu,\nu)\cap GMM_\infty(\rset^{2\d})} \int_{\rset^\d \times \rset^\d} \|x - y\|^2 \rmd \op(x,y) \right)^{\frac{1}{2}} \eqsp.
\end{equation}
As for $ W_2 $ with $ \mathcal{W}_2(\rset^\d) $, $ MW_2 $ defines a metric on $ GMM_\infty(\rset^\d) $ \cite[]{delon2020wasserstein}. 
In general, the transportation plan solution of the $ W_2 $ problem is not a Gaussian mixture, thus by restricting the set
of admissible couplings, we most of the time have $ MW_2(\mu,\nu) > W_2(\mu,\nu) $.  It can be shown that
the difference between $ MW_2(\mu,\nu) $ and $ W_2(\mu,\nu) $ is 
upper-bounded by a term that only depends on the weights and the covariances 
matrices of the components of the two mixtures.
An important property of $MW_2 $ is that it can be written in an equivalent form, which had already been introduced in \cite{chen2018optimal}: if $ \mu = \sum_k^{K}a_k\mu_k $ and $ \nu = \sum_l^{L}b_l\nu_l $, then
\begin{equation}\label{eq:mw2discret}
MW_2^2(\mu,\nu) = \inf_{\omega \in \Pi(a,b)}\sum_{k,l} \omega_{k,l}W_2^2(\mu_k,\nu_l) \eqsp,
\end{equation}
where $ a = (a_1,\dots,a_K)^T $, $ b = (b_1,\dots,b_L)^T  $. 
From a computational point of view, this latter formulation
reduces the problem to a simple small-scale discrete optimal transport 
problem since the $ W_2 $ distance between Gaussian distributions has 
a closed form: indeed, recall that if $ \mu_k = \mathrm{N}(m_k,\Sigma_k) $ and $ \nu_l = \mathrm{N}(m_l,\Sigma_l) $, then
\begin{equation}\label{eq:w2gaussian}
W_2^2(\mu_k,\nu_l) = \|m_k - m_l\|^2 + \mathrm{tr}\left(\Sigma_k + \Sigma_l - 2\left(\Sigma_l^{\frac{1}{2}}\Sigma_k\Sigma_l^{\frac{1}{2}}\right)^{\frac{1}{2}}\right) \eqsp.
\end{equation}
Finally, the respective solutions $ \op^* $ and $ \omega^* $ of 
Problems \eqref{eq:mw2} and \eqref{eq:mw2discret} are linked by the following relationship, for all  $ (x,y) \in \rset^d \times \rset^d $,
\begin{equation}\label{eq:optiplan}
\op^*(x,y) = \sum\limits_{k,l}\omega^*_{k,l}p_{\mu_k}(x)\delta_{y = T_{W_2}^{k,l}(x)} \eqsp,
\end{equation}
where $ T_{W_2}^{k,l} $ is the optimal $ W_2 $ transport map between $ \mu_k $ and $ \nu_l$ and $ p_{\mu_k} $ is the probability density function of $ \mu_k $.

\subsection{Gromov-Wasserstein distance}
\label{sec:gromov}
The Gromov-Wasserstein problem \cite[]{memoli} can be defined as the following:
given two \emph{network measure spaces}, i.e. triplets of the form 
$ (\spa{X},c_{\spa{X}},\mu) $ where $ \spa{X} $ is a Polish space, 
$ c_{\spa{X}} : \spa{X} \times \spa{X} \shortrightarrow \rset $ is a measurable 
function and $ \mu \in \mathcal{P}(\spa{X})  $, it aims at finding 

\begin{equation}\label{eq:gromov}
GW_p((\spa{X},c_{\spa{X}},\mu),(\spa{Y},c_{\spa{Y}},\nu)) =  \left( \inf_{\op \in \Pi(\mu,\nu)}  \int_{\spa{X} \times \spa{Y}} \int_{\spa{X} \times \spa{Y}}  | c_{\spa{X}}(x,x') - c_{\spa{Y}}(y,y')|^p \rmd \op(x,y) \rmd \op(x',y') \right)^{\frac{1}{p}} \eqsp,
\end{equation}
with $ p \geq 1 $.
The fundamental metric properties of $GW_p $ have been studied 
in depth in \cite[]{memoli,sturm2012space,chowdhury2019gromov}. 
When $c_{\spa{X}} $ and $c_{\spa{Y}} $ are powers of the metrics $ d_{\spa{X}} $ 
and $ d_{\spa{Y}}$ of the base spaces $ \spa{X} $ and $ \spa{Y} $, 
$GW_p$ induces a metric over the space of \emph{metric measure spaces} (i.e. the triplets $ (\spa{X}, d_{\spa{X}}, \mu) $)
quotiented by the \emph{strong isomorphisms} \cite[]{sturm2012space}, where one says that two metric measure spaces $ (\spa{X},d_{\spa{X}},\mu) $ and $ (\spa{Y},d_{\spa{Y}},\mu) $ are strongly isomorphic if there exists an isometric bijection $ \phi \colon \mathrm{supp}(\mu) \rightarrow \mathrm{supp}(\nu) $
that transports $ \mu $ into $ \nu $.  When $ c_{\spa{X}} $ and $ c_{\spa{Y}} $ are not powers of 
the metrics of the base spaces, $ GW_p $ still defines a metric, but this time over 
the space of network measure spaces  quotiented by the \emph{weak isomorphisms} \cite[]{chowdhury2019gromov}, which are spaces isomorphic for the costs $ c_{\spa{X}} $ and $ c_{\spa{Y}} $ 
relatively to a third space, see \cite[]{chowdhury2019gromov} for details. 
Note that in both cases, the metric property of $ GW_p $ stricly holds 
only when it takes finite values, and so it is natural to restrict it to 
the following space
\begin{equation}
\mathbb{M}_p = \textstyle{\ensembleLigne{(\spa{X},c_{\spa{X}},\mu)}{\int_{\spa{X} \times \spa{X}}c^p_{\spa{X}}(x,x')\rmd\mu(x)\rmd\mu(x') < + \infty }} \eqsp.
\end{equation}
Note that when $ \spa{X} $ and $ \spa{Y} $ are fixed as well as 
$ c_{\spa{X}} $ and $ c_{\spa{Y}} $, it is natural to see $ GW_p $ as 
a distance between the two measures $ \mu $ and $ \nu $ rather than a distance 
between the two network measure spaces $ (\spa{X},c_{\spa{X}},\mu) $ and
$ (\spa{Y},c_{\spa{Y}},\nu) $. Therefore, we will denote in that case -
with a slight abuse of notations - $ GW_p(\mu,\nu) $ instead 
of $ GW_p((\spa{X},c_{\spa{X}},\mu),(\spa{Y},c_{\spa{Y}},\nu)) $.
Finally, in the discrete setting, given $  a = (a_1,\dots,a_m)^T $ and $ b = (b_1,\dots,b_n)^T $ being 
respectively in $ \Delta_m $ and $ \Delta_n $, and
given two non-negative cost matrices $ C^x $ and $ C^y $ of respective size 
$ m \times m $ and $ n \times n $, the Gromov-Wasserstein 
distance can be written as 
\begin{equation}
GW_p(a,b) = \inf_{\omega \in \Pi(a,b)} \sum_{i,j,k,l} |C^x_{i,k} - C^y_{j,l}|^p\omega_{i,j}\omega_{k,l} \eqsp. 
\end{equation}

\subsection{Other invariant distances}\label{sec:oid}
\cite{sturm2006geometry} has introduced another distance 
between metric measures spaces which takes the following form
\begin{equation}\label{eq:sturmdist}
D_p((\spa{X},d_{\spa{X}},\mu),(\spa{Y},d_{\spa{Y}},\nu)) = \inf_{\spa{Z}, \psi, \phi}  W_p(\psi_{\#}\mu,\phi_{\#}\nu) \eqsp,
\end{equation}
where $ (\spa{X},d_{\spa{X}},\mu) $ and $ (\spa{Y},d_{\spa{Y}},\nu) $ are 
two metric measure spaces as defined in \Cref{sec:gromov}, $ \spa{Z} $ is a third Polish space, and where $ \psi \colon \spa{X} \rightarrow \spa{Z} $ and $ \phi \colon \spa{Y} \rightarrow \spa{Z} $ are two isometric mappings.
More recently, \cite{alvarez2019towards} have introduced another family of invariant OT distances in the Euclidean setting which can also be used to compare distributions on spaces of different dimensions. Initially, 
\cite{alvarez2019towards} have introduced this OT distance 
in the setting where $ \mu $ and $ \nu $ are both living in the 
same Euclidean space $ \rset^\d $. Yet, it generalizes 
well to settings where $ \mu $ and $ \nu $ are living in spaces of different dimensions.
Between two measures $ \mu $ and $ \nu $ on $ \rset^\d $ and $ \rset^\di $, this reads as 
\begin{equation}\label{eq:invariant}\tag{$IW_2$}
IW_{2,\spa{H}}(\mu,\nu) = \left(\inf_{\op \in \Pi(\mu,\nu)} \inf_{h \in \spa{H}} \int_{\rset^\d \times \rset^\di} \|x - h(y)\|^2 \rmd \op(x,y)\right)^{\frac{1}{2}} \eqsp, \end{equation}
where $ \spa{H} $ 
is a class of mappings from $ \rset^\di $ to $ \rset^\d $ encoding the invariance.
This is a \emph{non-convex} optimization problem in $ \op $ and $ h $ that becomes 
convex in $ \op $ if $ h $ is fixed and becomes also convex in $ h$ if 
$ \op $ is fixed and $ \spa{H} $ is a convex set.
When $ \d $ is equal to $ \di $ and both measures are centered, 
\cite{alvarez2019towards} have notably shown that when $ \nu $ is such that 
$ \mathbb{E}_{Y \sim \nu}[YY^T] = \Id_{\d} $ and when $ \spa{H} = \spa{H}_1 := \ensembleLigne{P \in \rset^{\d \times \d}}{\|P\|_{\spa{F}} \leq \sqrt{\d}} $, 
Problem \eqref{eq:invariant} is equivalent 
to the Gromov-Wasserstein problem \eqref{eq:gromov} of order $ 2 $  with
inner-product costs. Indeed, it can 
be shown that both problems are equivalent in that case to 
\begin{equation}\label{eq:norm2}\tag{$\spa{F}$-COV}
\sup\limits_{\op \in \Pi(\mu,\nu)} \left\| \int_{\rset^\d \times \rset^\d} xy^T \rmd \op(x,y)\right\|_{\spa{F}} \eqsp,
\end{equation}
where for any matrix $ A $ of size $ \d \times \d $, $ \|A\|_{\spa{F}} $
denotes the Frobenius  norm, i.e. $ \sqrt{\mathrm{tr}(A^TA)} $.
Another interesting case is when $ \spa{H} = \spa{H}_2 := \mathbb{O}(\rset^\d) :=  \ensembleLigne{P \in \rset^{\d \times \d}}{P^TP = \Id_\d} $
is the set of orthogonal matrices of size $ \d \times \d $. In 
that case, Problem \eqref{eq:invariant} is equivalent to 
\begin{equation}\label{eq:norm1}\tag{$*$-COV}
\sup\limits_{\op \in \Pi(\mu,\nu)} \left\| \int_{\rset^\d \times \rset^\d} xy^T \rmd \op(x,y) \right\|_{*} \eqsp,
\end{equation} 
where for any matrix $ A $ of size  $ \d \times \d $, $ \|A\|_{*} $ is the nuclear norm of $ A $, 
i.e. $ \|A\|_{*} = \mathrm{tr}((A^TA)^{\frac{1}{2}}) $. Note that both Problems \eqref{eq:norm2} 
and \eqref{eq:norm1} are \emph{non-convex}. These 
results have been shown by \cite{alvarez2019towards} in the case where $ \mu $ and $ \nu $ are discrete but 
can easily be extended to continuous distributions. Observe that problem \eqref{eq:norm1} consists in 
maximizing the sum of the singular values of the 
cross-covariance matrix $\int xy^T \rmd \op(x,y) $, whereas 
the Problem \eqref{eq:norm2} consists in maximizing the sum of the squared singular values of the cross-covariance matrix. 
In general, these two problems are not equivalent 
despite being structurally similar, as the example of \Cref{fig:three_points} illustrates it.

\begin{figure}[!ht]
  \centering
  \begin{tabular}{cc}
    \includegraphics[width=0.43\textwidth]{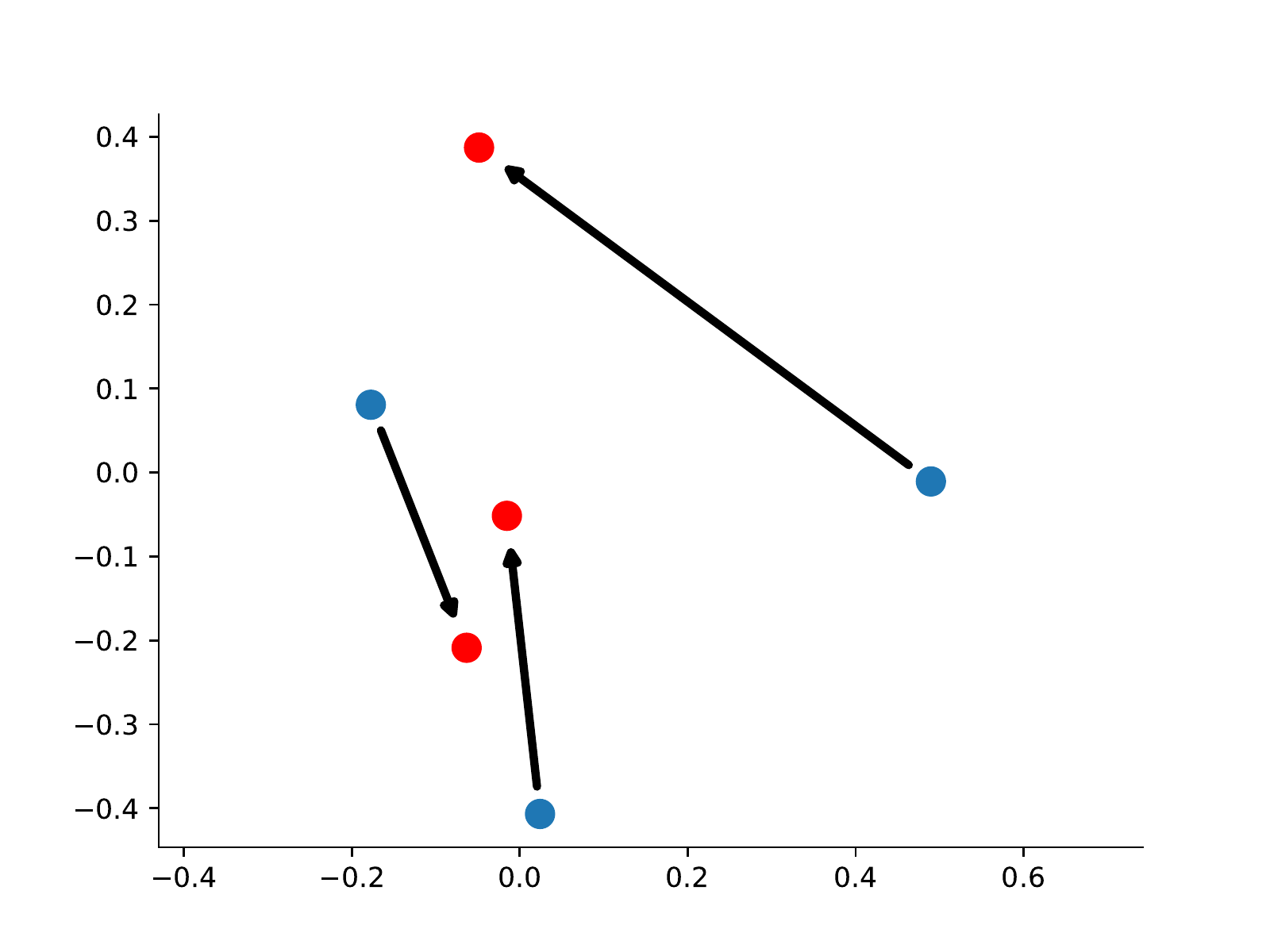} &
    \includegraphics[width=0.43\textwidth]{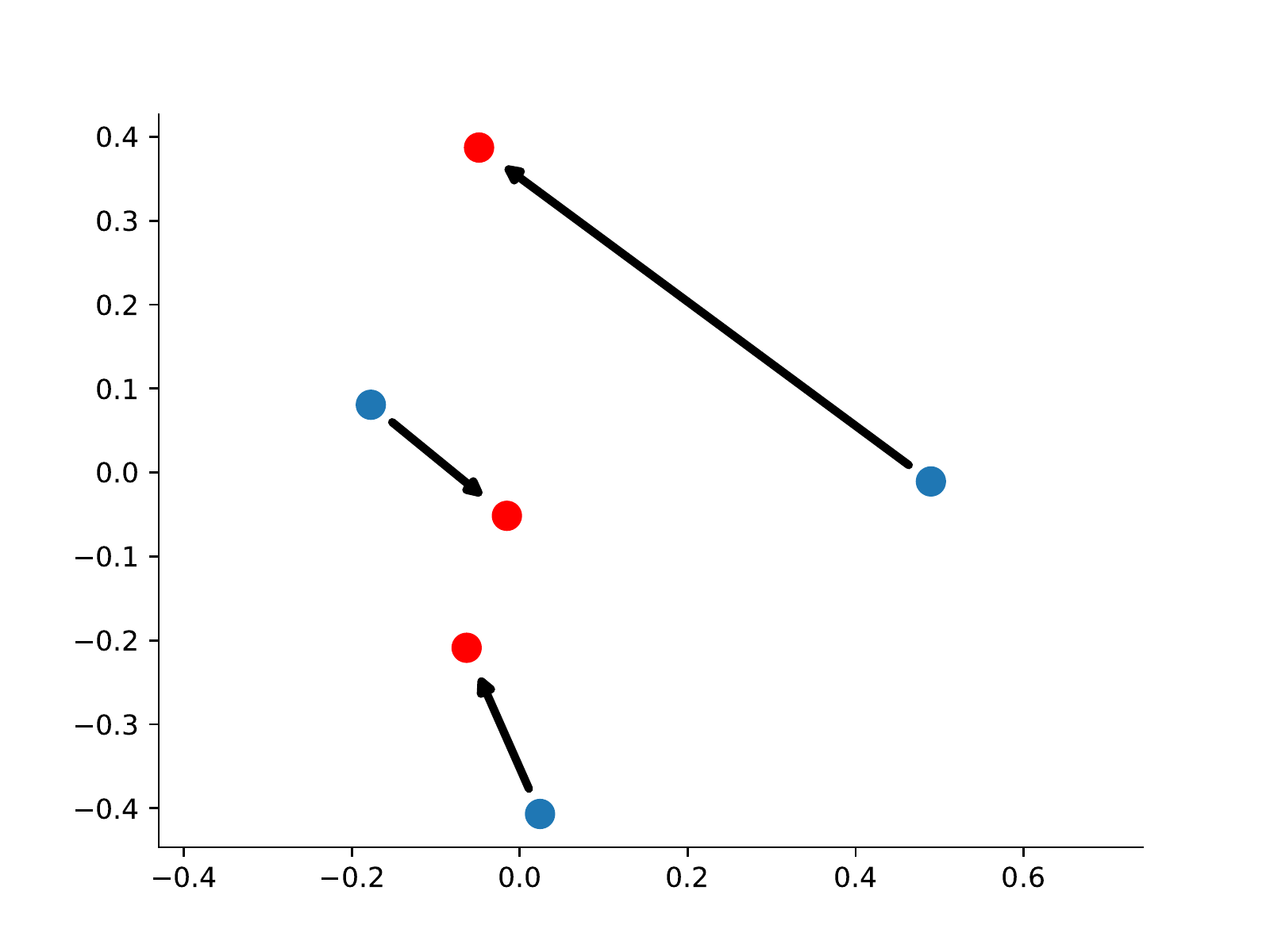}
  \end{tabular}
\caption{Transport plans between two discrete centered distributions on $ \rset^2 $ composed 
of three points. Left: optimal coupling given 
by the maximization of Problem \eqref{eq:norm2}. Right: optimal coupling given 
by the maximization of Problem \eqref{eq:norm1}.}\label{fig:three_points}
\end{figure}

\section{Gromov-Wasserstein distance between mixture of Gaussians}
\label{sec:mgw2}
In this section, we define a Gromov-Wasserstein type 
distance between Gaussian mixture distributions. This distance
is a natural "Gromovization" of Problem \eqref{eq:mw2discret}. 
Indeed, as it has already been observed  in the literature \cite[]{chen2018optimal,lambert2022variational}, 
any Gaussian mixture in dimension $ \d $ can be identified with a probability distribution on
$ \rset^\d \times \mathbb{S}^\d_{+} $, i.e. the product space of means and covariance matrices. Equivalently, a
finite Gaussian mixture can be seen as a discrete probability distribution on 
the space of Gaussian distributions $ \mathcal{N}(\rset^\d) $\footnote{$ \mathcal{N}(\rset^\d) $ includes
the degenerate Gaussian distributions, as for instance the Dirac distributions.}, which 
has been proven to be a complete metric space when endowed with 
$ W_2 $ \cite[]{takatsu2010wasserstein} and is furthermore separable 
since it is a subspace of $ \mathcal{W}_2(\rset^\d) $ 
which is itself a separable metric space when endowed with $ W_2 $ \cite[]{bolley2008separability}.  
Since the theory of optimal transport still applies on 
measures over non-Euclidean spaces \cite[]{villani2008optimal}, 
it follows that Problem \eqref{eq:mw2discret}
can formally be thought as a simple OT problem between two discrete
measures in $ \mathcal{P}(\mathcal{N}(\rset^\d)) $. Thus, one
can define directly its Gromov version.

\begin{definition}
Let $ \mu = \sum_k a_k\mu_k $ and $ \nu = \sum_l b_l\nu_l $ be two Gaussian mixtures respectively on 
$ \rset^\d $ and $ \rset^\di $, we define
\begin{equation}\label{eq:mgw2}\tag{$MGW_2$}
MGW_2^2(\mu,\nu) = \inf_{\omega \in \Pi(a,b)}\sum_{i,j,k,l} |W_2^2(\mu_i,\mu_k) - W_2^2(\nu_j,\nu_l)|^2\omega_{i,j}\omega_{k,l} \eqsp. 
\end{equation}
\end{definition}
Unlike $ MW_2 $, there is no straightforward equivalent continuous formulation of this latter problem. In particular,
it is not clear whether Problem \eqref{eq:mgw2} is equivalent or not to the continuous GW problem between $ \mu $ and $ \nu $ - seen 
as continuous measures on $ \rset^\d $ and $ \rset^\di $ - where the set of admissible couplings is restricted to Gaussian mixture distributions. 
Thanks to the identifiability property of the set of finite Gaussian mixtures, we have that each $ \mu \in GMM_{\infty}(\rset^\d) $ is associated with a unique discrete distribution $ \tilde{\mu} \in \mathcal{P}(\mathcal{N}(\rset^\d)) $
and $ MGW_2$ between $ \mu $ and $ \nu $ coincides with $ GW_2  $ with squared $ W_2 $ costs between their associated discrete measures $ \tilde{\mu} $ and $ \tilde{\nu} $
in $ \mathcal{P}(\mathcal{N}(\rset^\d)) $. Finally, note that we have defined 
$ MGW_2 $ only between GMMs with finite number of components because 
there is in general no identifiability property for infinite Gaussian mixtures. 
As an outcome, for a given infinite GMM $ \mu $ on $ \rset^\d $, 
there might be more than one associated continuous
measure $ \tilde{\mu} $ on $ \mathcal{N}(\rset^\d) $. For instance,
the standard Normal distribution $ \mathrm{N}(0,1) $ can naturally be
identified in $ \mathcal{P}(\mathcal{N}(\rset)) $ with the Dirac distribution at $ \mathrm{N}(0,1) $, but also 
with the Normal distribution $ \mathrm{N}(0,1/2) $ over the parametrized line
$ \ensembleLigne{\mathrm{N}(\theta,1/2) \in \mathcal{N}(\rset)}{\theta \in \rset} $,
or with $ \mathrm{N}(0,1) $ over the  
parametrized line $ \ensembleLigne{\delta_{\theta} \in \mathcal{N}(\rset)}{\theta \in \rset} $.

\subsection{Metric properties}
Here we study the metric property of $ MGW_2 $ that mainly arises from the Gromov-Wasserstein structure of
Problem \eqref{eq:mgw2}. Indeed, the following result is a direct
consequence of the theory developed by \cite{sturm2012space}. 

\begin{proposition}\label{prop:mgw2metric} 
In the following, $ \mu = \sum_k a_k\mu_k $ and $ \nu = \sum_l b_l\nu_l $ are two GMMs respectively in 
$ GMM_K(\rset^\d)$  and $  GMM_L(\rset^\di) $. 
\begin{itemize}
\item[(i)] $ MGW_2 $ is non-negative and symmetric.
\item[(ii)] $ MGW_2 $  satisfies the triangle inequality, i.e. for any $ \xi \in GMM_S(\rset^\dii)$, 
\begin{equation}
MGW_2(\mu,\nu) \leq MGW_2(\mu,\xi) + MGW_2(\xi,\nu) \eqsp.
\end{equation}
\item[(iii)] $ MGW_2(\mu,\nu) = 0 $ 
if and only if there exists a bijection $ \phi \colon \{\mu_k\}_{k} \rightarrow \{\nu_l\}_{l} $ such 
that $ \nu = \sum_k a_k \phi(\mu_k) $  and $  \phi $ 
is an isometry for $ W_2 $, i.e. for all $ k $ and $ i $ 
smaller than $ K $, $ W_2(\phi(\mu_k),\phi(\mu_i)) = W_2(\mu_k,\mu_i)$.
\end{itemize}
\end{proposition}

\begin{proof}[Sketch of proof.]
The proof of these results mainly consists in applying the theory of \cite{sturm2012space}, using the facts that $ \spa{N}(\rset^\d) $ 
is complete \cite[]{takatsu2010wasserstein}, separable \cite[]{bolley2008separability}, 
and metrizable with the Wasserstein distance. See \Cref{sec:proofmgw2metric}
for the full proof.
\end{proof}

$ MGW_2 $ defines thus a pseudometric on the set of all finite Gaussian mixtures of arbitrary dimensions, i.e. 
the set, 
\begin{equation}
\textstyle{ \mathcal{GMM}_\infty = \underset{\d \geq 1}{\bigsqcup} GMM_{\infty}(\rset^\d) \eqsp,}
\end{equation}  
that is invariant to the mappings $ \phi $ that transform a finite Gaussian mixture $ \sum_{k=1}a_k\mu_k $ into another
finite Gaussian mixture of the form $ \sum_{k=1}^Ka_k\nu_k $ 
such that for all $ k $ and $ i $ smaller than $ K $, 
$ W_2(\nu_k,\nu_i) = W_2(\mu_k,\mu_i) $. A question 
that arises is: 
are all these mappings $ \phi $ between $GMM_{\infty}(\rset^{\d'})$ and $\mathcal{P}(\rset^\d)$ always associated with mappings $ T \colon \rset^\di \rightarrow \rset^\d $
that are isometries for the Euclidean norm and such that $ T_{\#} \mu$ 
coincides with $ \phi (\mu)$ for every $\mu\in GMM_{\infty}(\rset^{\d'})$? We can already state the following 
converse result.

\begin{proposition}\label{prop:invcase}
Let $ \d \geq \di $, and let $ T \colon \rset^\di \rightarrow \rset^\d $ be 
a mapping that is an isometry for the Euclidean norm. Then the mapping 
$ \phi_{T} \colon GMM_\infty(\rset^\di) \rightarrow \mathcal{P}(\rset^\d) $ defined as 
$ \phi_{T}(\mu) = T_{\#}\mu $ for all $ \mu \in GMM_\infty(\rset^\di) $,
is such that for any $ \mu $ of the form $ \Sigma_{k=1}^K a_k\mu_k $, 
$ \phi_{T}(\mu) $ is in $ GMM_\infty(\rset^\di) $ and is of the form $ \Sigma_{k=1}^K a_k\nu_k $,
with $ \{\nu_k\}_{k=1}^K $ being such that, for all $ k$ and $ i $ smaller than $ K $, 
$ W_2(\nu_k,\nu_i) = W_2(\mu_k,\mu_i) $, and so $ MGW_2(\mu,T_{\#}\mu) = 0 $.
\end{proposition}

\begin{proof}[Sketch of proof.]
The proof of this result mainly consists in showing that for any $ \mu \in GMM_\infty(\rset^\di) $, 
$ \phi_{T}(\mu) $ is in $ GMM_\infty(\rset^\d) $ because $ T $ is necessarily affine, 
as a direct consequence of the Mazur-Ulam theorem \cite[]{mazur1932transformations}
which implies that any isometry from $ \rset^\di $ to $ \rset^d $ (endowed with the Euclidean norm) is necessarily affine. See \Cref{sec:proof:invcase} for the full proof. 
\end{proof}

Hence, if $ T \colon \rset^\di \rightarrow \rset^\d $ is an isometry for the Euclidean norm, 
then $ MGW_2 $ is invariant to the mapping $ \phi_T : GMM_{\infty}(\rset^\di) \rightarrow GMM_{\infty}(\rset^\d) $
given for all $ \mu \in  GMM_{\infty}(\rset^\di) $, by $ \phi_T(\mu) = T_{\#}\mu $. Yet, in general, there exist
mappings $ \phi : \spa{W}_2(\rset^\di) \rightarrow \spa{W}_2(\rset^\d) $  that are isometries for $ W_2 $ and
that are not induced by any mapping $ T \colon \rset^\di \rightarrow \rset^\d $ that is 
an isometry for the Euclidean norm. This has been proven
by \cite{kloeckner2010geometric} in the general case when considering isometries defined all over $ \spa{W}_2(\rset^\d) $, but remains
true when restricting to isometries defined over subspaces of $ \spa{N}(\rset^\d) $ as the following example suggests.

\begin{example}
let $ \mathcal{N}_{++}(\rset) $ be 
the set of one-dimensional Gaussian distributions with strictly positive mean. Let $ \phi \colon \mathcal{N}_{++}(\rset) \rightarrow \mathcal{N}_{++}(\rset) $
be the mapping that swaps the mean and the standard deviation, i.e. 
such that for any $ \gamma = \mathrm{N}(m_\gamma,\sigma_{\gamma}^2) $ 
with $ m_\gamma > 0 $ and $ \sigma_{\gamma} > 0 $, $ \phi(\gamma) = \mathrm{N}(\sigma_{\gamma},m_\gamma^2) $. Then 
$ \phi $ is an isometry for $ W_2 $. Observe indeed that for $ \gamma $ and $ \zeta $ in $ \mathcal{N}_{++}(\rset) $, we have
\begin{equation}
W_2(\phi(\gamma),\phi(\zeta)) = (\sigma_\gamma - \sigma_\zeta)^2 + (m_{\gamma} - m_{\zeta})^2 = W_2(\gamma,\zeta) \eqsp. 
\end{equation}
\end{example}
Thus $ \phi $ is an isometry for $ W_2 $, yet $ \phi $ is not induced 
by any isometry of $ \rset $. Hence there exist mappings from $ GMM_{\infty}(\rset^\di) $
to $ GMM_{\infty}(\rset^\d) $ that satisfy the conditions above 
but which are not induced by isometries for the Euclidean norm from $ \rset^\di $ to $ \rset^\d $.

\subsection[MGW_2 in practice]{\(MGW_2\) in practice}
\paragraph{Using $ MGW_2 $ on discrete data distributions.} 
Most applications of optimal transport involve discrete data that can be thought as samples
drawn from underlying distributions, which are not GMMs in general. 
In those applications, we aim to evaluate an OT distance  
between two distributions of the form $ \hat{\mu} = (1/M)\sum_i \delta_{x_i} $
and $ \hat{\nu} = (1/N)\sum_j \delta_{y_j} $ where $ \{x_i\}_i $ and $ \{y_j\}_j $
are families of respectively $ M $ and $ N $ vectors of $ \rset^\d $ and $ \rset^\di $. Though $ \hat{\mu} $ 
and $ \hat{\nu} $ can be thought as mixtures of degenerate Gaussian distributions, evaluating
directly $ MGW_2(\hat{\mu},\hat{\nu}) $ is not particularly interesting since 
we have in that case $ MGW_2(\hat{\mu},\hat{\nu}) = GW_2(\|.\|^2,\|.\|^2,\hat{\mu},\hat{\nu}) $. However, 
we can design a pseudometric $ MGW_{K,2} $ between $ \hat{\mu} $ and $ \hat{\nu} $  by fitting 
two GMMs $ \mu $ and $ \nu $ with $ K $ components on $ \hat{\mu} $ and $ \hat{\nu} $ and 
then setting $ MGW_{K,2}(\hat{\mu},\hat{\nu}) = MGW_2(\mu,\nu) $. The approximation of 
$ \hat{\mu} $ and $ \hat{\nu} $ by $ \mu $ and $ \nu $ can be done 
by maximizing the log-likelihood of the GMMs
with the EM algorithm \cite[]{dempster1977maximum}. 
Note that if $ K $ is chosen too small, 
the approximations of $ \hat{\mu} $ and $ \hat{\nu} $ will be of bad quality and we are likely 
to observe undesirable behaviors, 
as for instance having $ MGW_{K,2}(\hat{\mu},\hat{\nu}) = 0 $ despite $ \hat{\mu} $
and $ \hat{\nu} $ not being equal up to an isometry. Thus, the choice of $ K $
must be a compromise between the quality of the approximation given by the GMM
and the computational cost. To illustrate the pratical use 
of $ MGW_2 $ on a simple toy example, we draw $ 150 $ samples 
from the spiral dataset provided in the scikit-learn toolbox\footnote{The package 
is accessible here: \href{https://scikit-learn.org/stable/}{https://scikit-learn.org/stable/}.} \cite[]{pedregosa2011scikit} and we apply rotations with various angles
on this dataset. We then fit independently GMMs with $ 20 $ components
on the initial and the target rotated datasets and we compute $ MGW_2 $ between the two obtained GMMs. We also 
compute $ GW_2 $ with inner-product as cost functions, $ MW_2 $ using also $ 20 $ Gaussian components 
and $ W_2 $. The results can be found in \Cref{fig:spiral_datasets}. As expected, $ MGW_2 $ is
rotation-invariant as $ GW_2 $ which is not the case of $ MW_2 $ and $ W_2 $.

\begin{figure}[!ht]
  \centering
  \begin{tabular}{p{0.45\linewidth}p{0.49\linewidth}}
    \centering \textbf{Spiral datasets} & \centering \textbf{Evolution of OT distances}
  \end{tabular}
  \\
  \begin{minipage}[c]{0.47\linewidth}
    \begin{tabular}{cc}
      \includegraphics[width=0.43\textwidth]{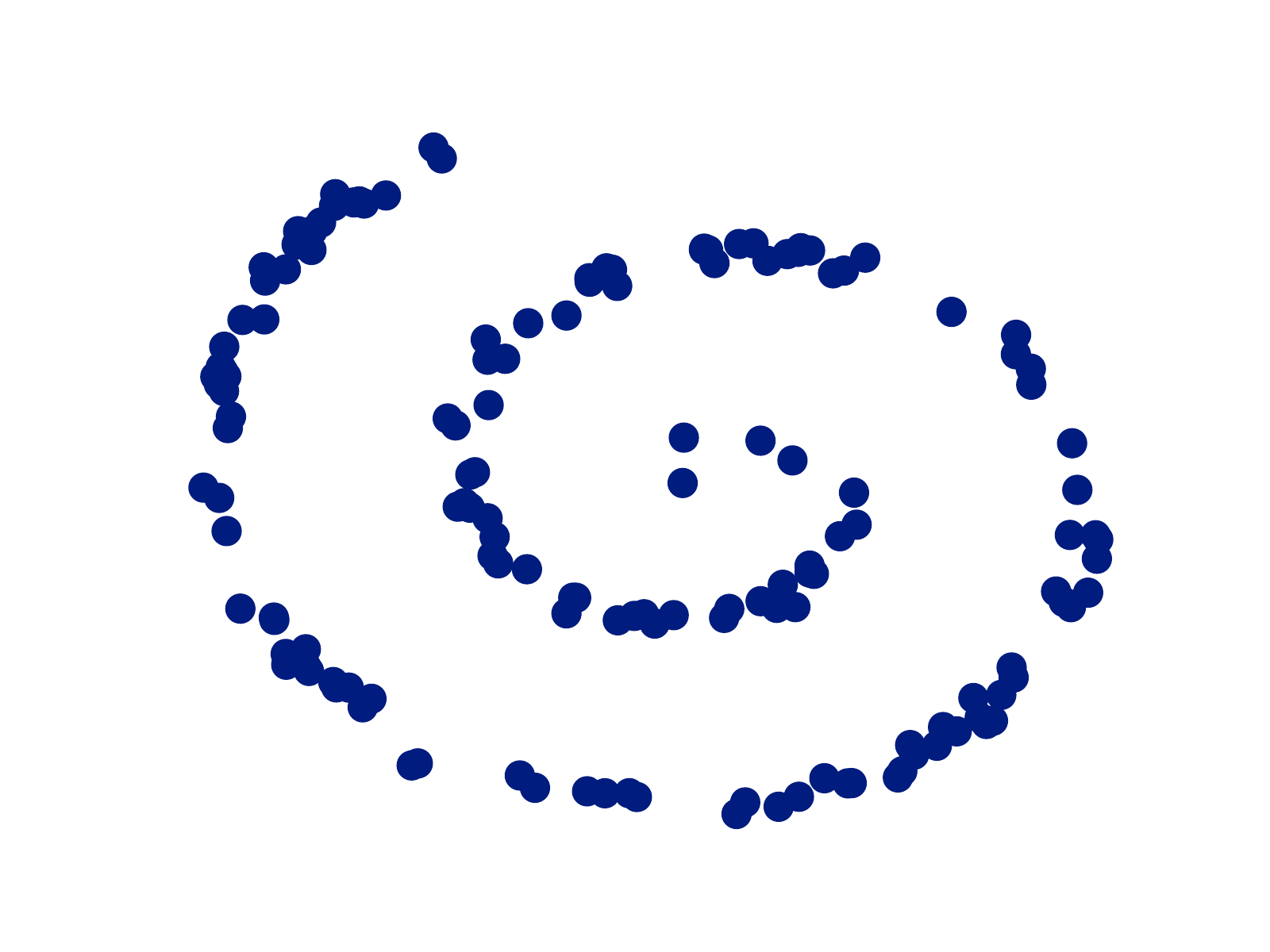} &
      \includegraphics[width=0.43\textwidth]{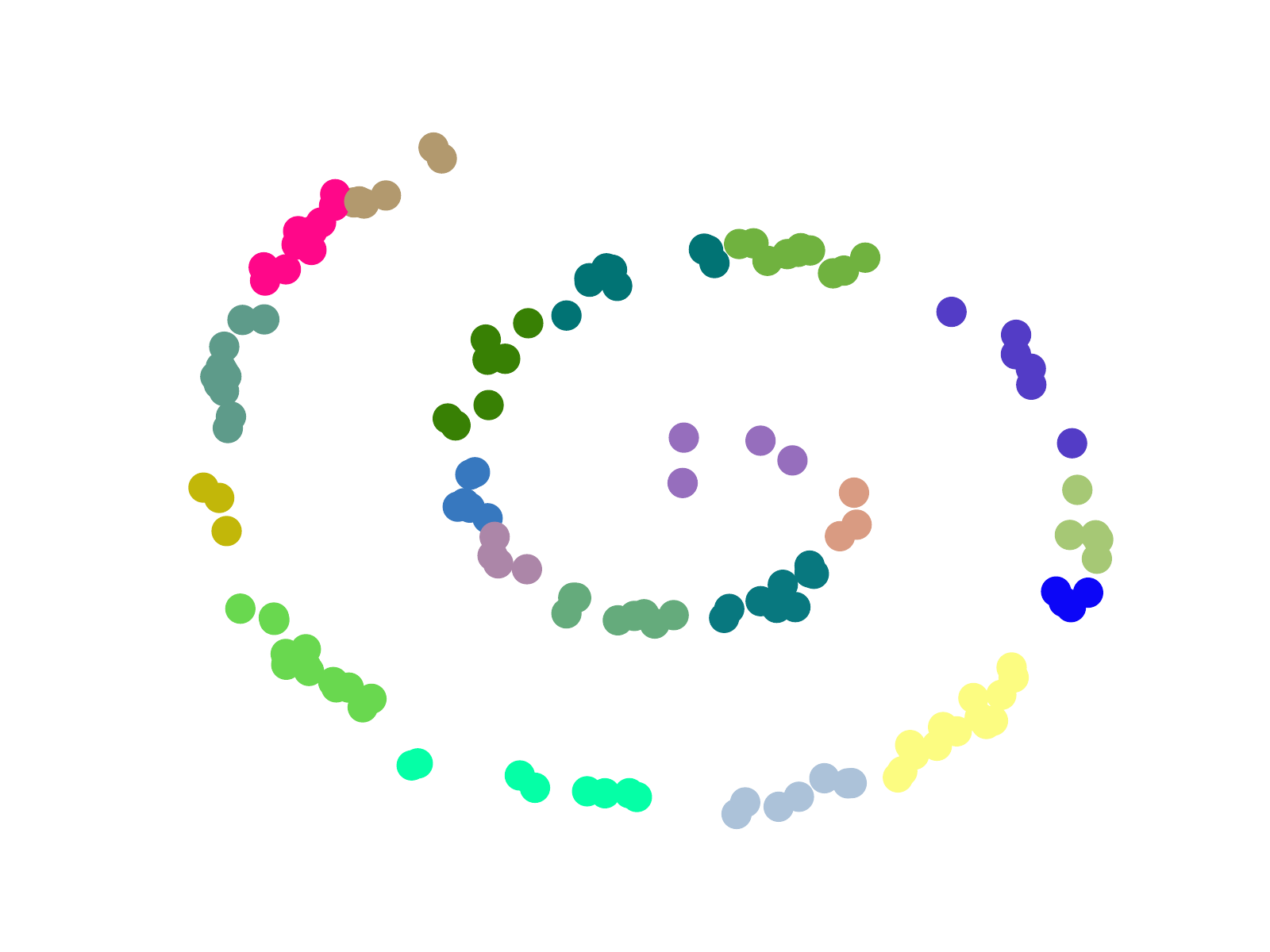} \\
      \includegraphics[width=0.43\textwidth]{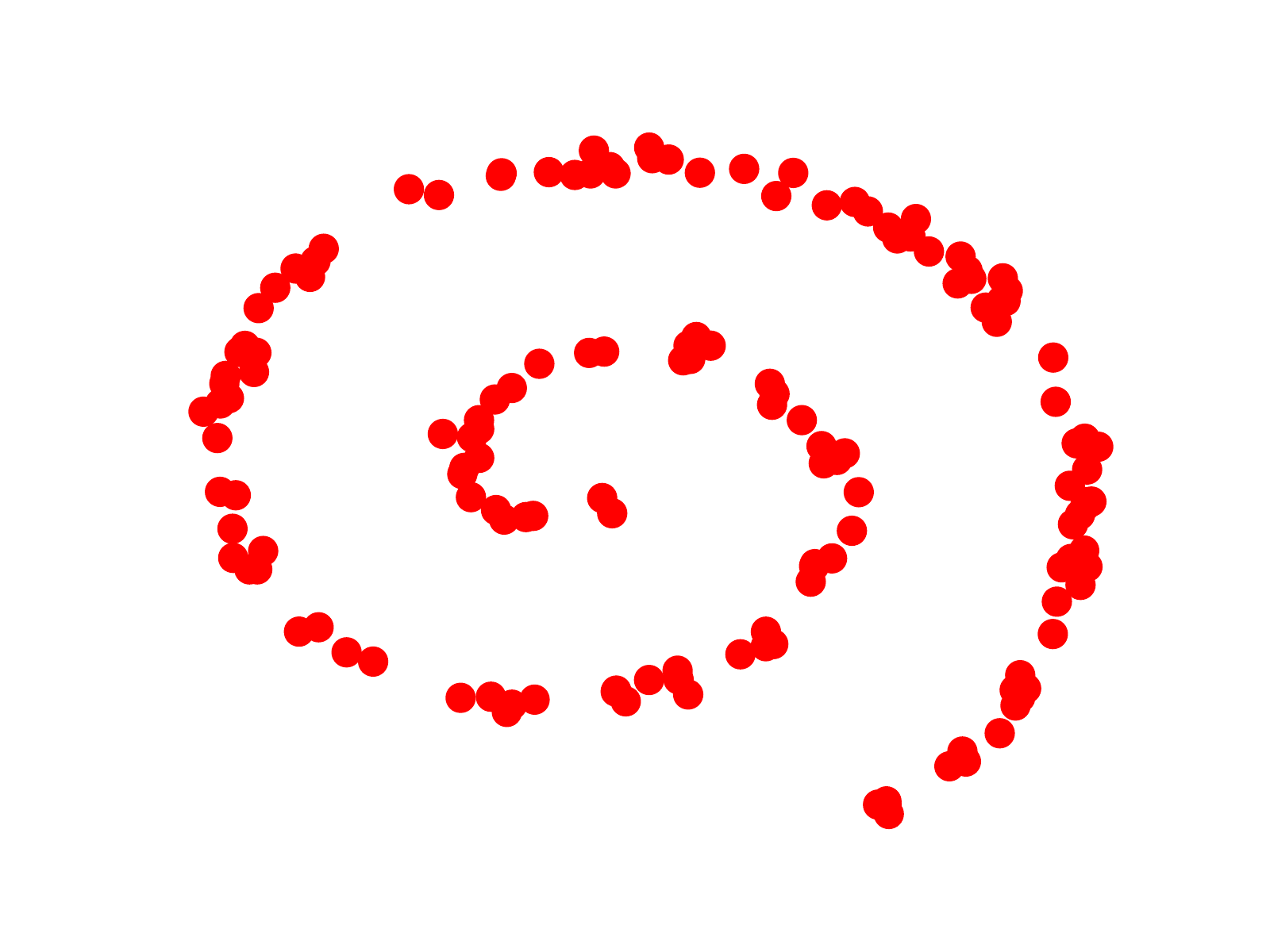} &
      \includegraphics[width=0.43\textwidth]{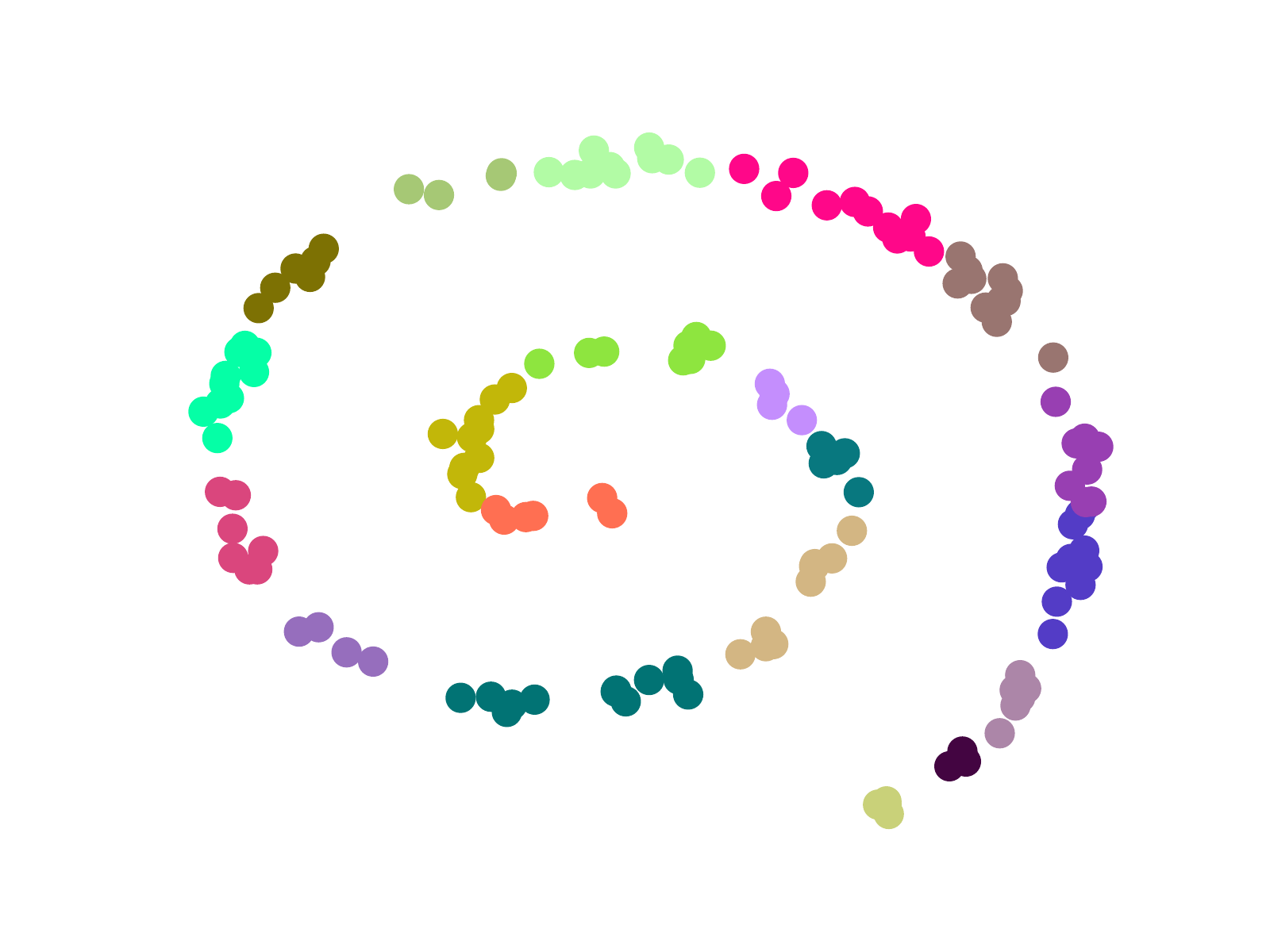} 
    \end{tabular}
  \end{minipage}
  \begin{minipage}[c]{0.51\linewidth} 
    \includegraphics[width=\textwidth]{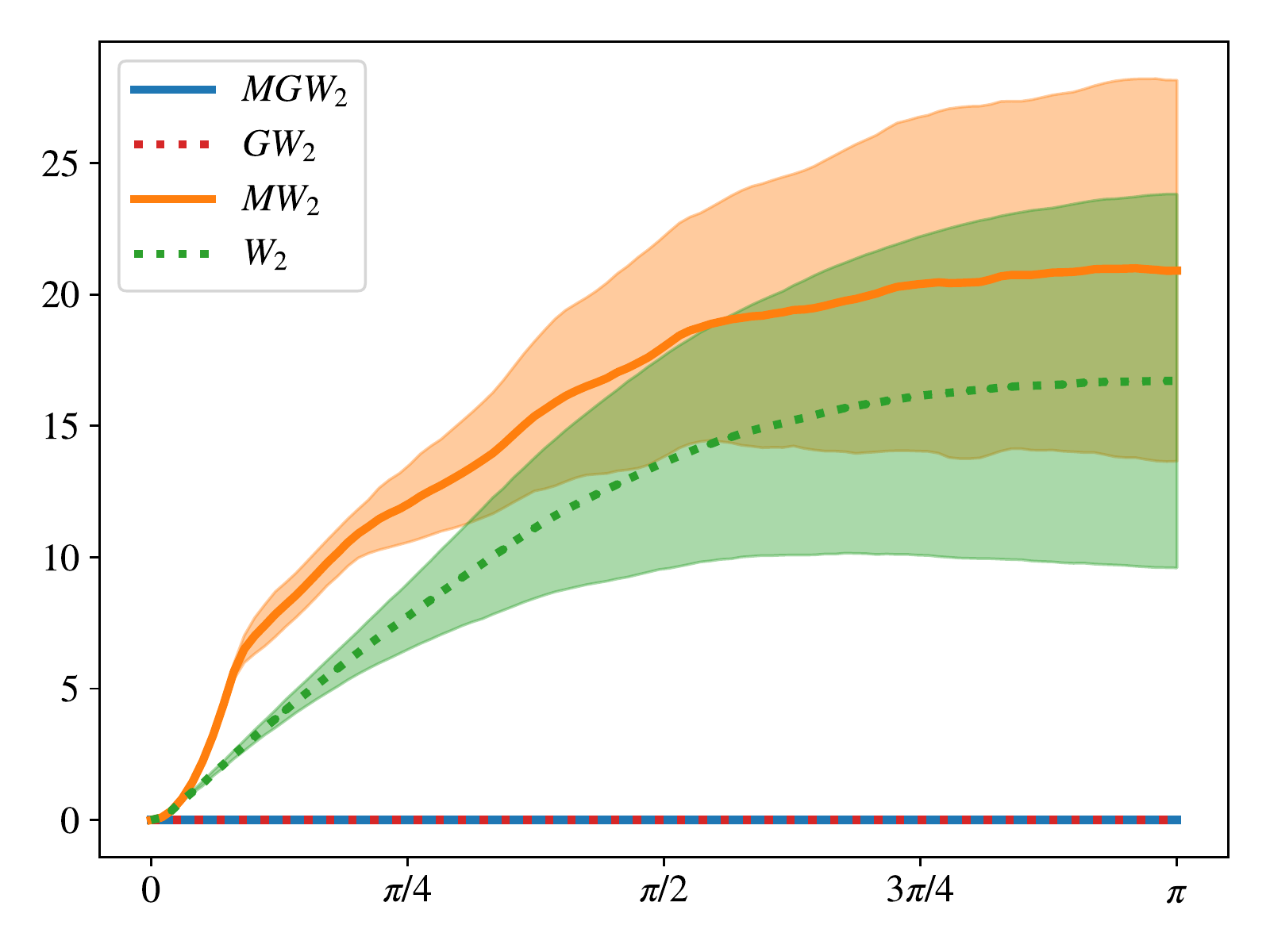}
  \end{minipage}
\caption{Left first column: spiral datasets (in blue and red) composed of $ 150 $ points of $ \rset^2 $. The 
red dataset corresponds to points sampled from the distribution of 
the blue dataset rotated from $ \uppi $. Left second column: The two corresponding learned GMMs 
with $ 20 $ components via EM algorithm (each color corresponds to a Gaussian component of the GMMs).
Right: evolution of $ MGW_2 $, $ GW_2 $, $ MW_2 $,  and $ W_2 $ between the initial distribution (in blue) and the rotated ones in function of the angle of rotation.
Experiments are averaged over $ 10 $ runs and the colored bands correspond to $ +/- $
the standard deviation. This experiment is inspired from \cite{titouan2019sliced}.}\label{fig:spiral_datasets}
\end{figure}

\paragraph{Difficulty of designing a transportation plan.}
The $ MGW_2 $ problem can be used on discrete data to provide an optimal coupling $ \omega^* $ between the Gaussian components of the two Gaussian mixtures $ \mu $ and $ \nu $ that approximate the discrete data distributions $ \hat{\mu} $ and $ \hat{\nu} $.  However, some applications require an coupling $ \op $ between the points 
that compose $ \hat{\mu} $ and $ \hat{\nu} $. 
It is not straightforward to derive such a transportation plan $ \op $ from the plan $ \omega^* $ that minimizes the $ MGW_2 $ problem. A naive heuristic approach would be to define $ \op $ in a similar way to \eqref{eq:optiplan}, transporting the Gaussian components using restricted-$GW_2$ transport maps \cite[]{salmona2021gromov} instead of $ W_2 $ transport maps. Yet this approach introduces 
too many degrees of freedom as it consists in transporting the Gaussian components independently, without taking into account the global structure of the mixture, see \Cref{fig:gaussian_example} for an illustrative example.
Since selecting the solution that preserves the global structure of the mixture among all the candidates seems to be a difficult combinatorial problem, a better solution to design such plan would be to derive explicitely the isometric transformation that has been implicitely applied to one of the two measures when solving the $ MGW_2 $ problem. This is the idea behind the embedded Wasserstein distance that we introduce in the following section.

\begin{figure}[!ht]
  \centering
  \begin{tabular}{c|c|c}
    \includegraphics[width=0.30\textwidth]{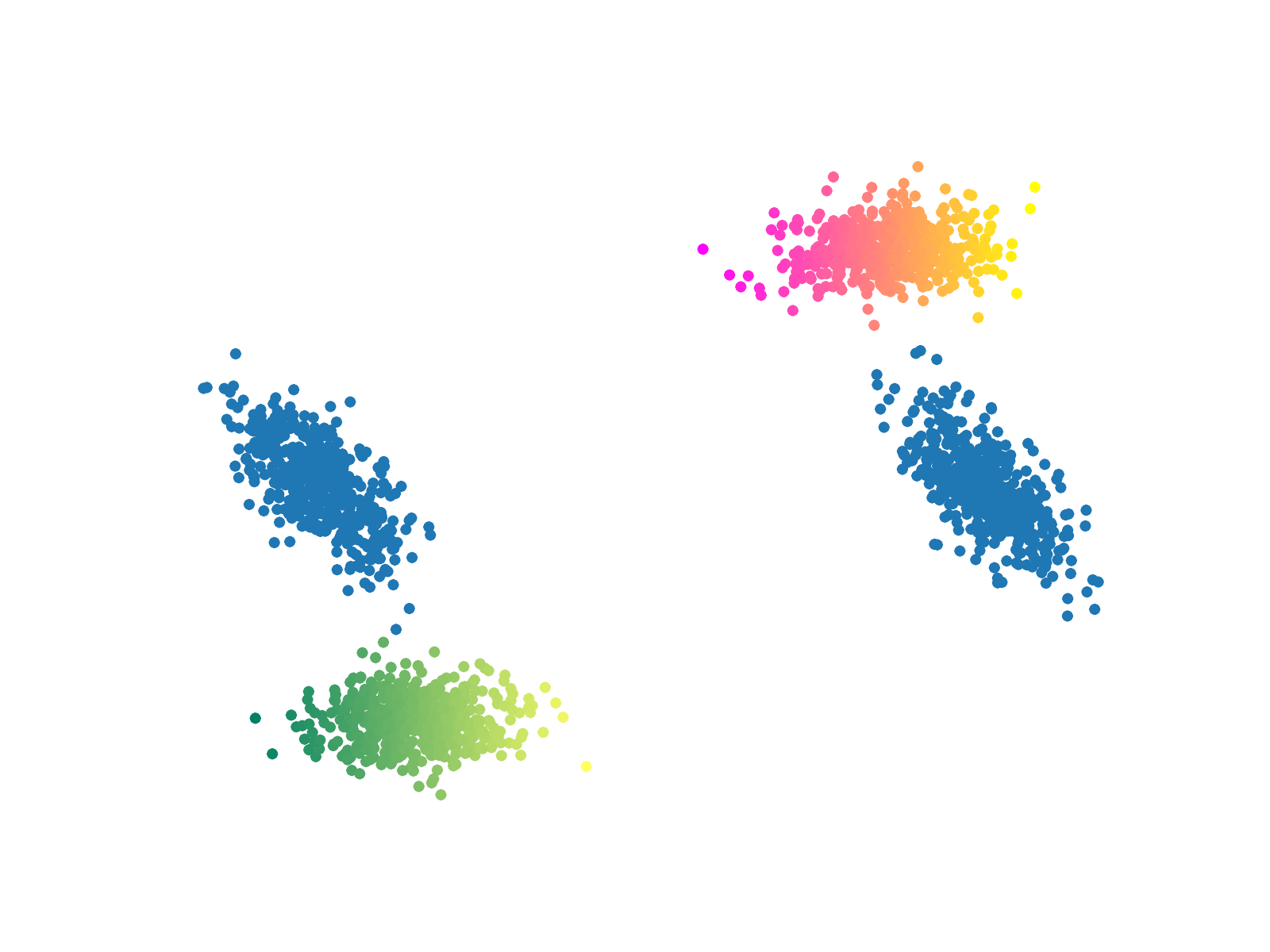} &
    \includegraphics[width=0.30\textwidth]{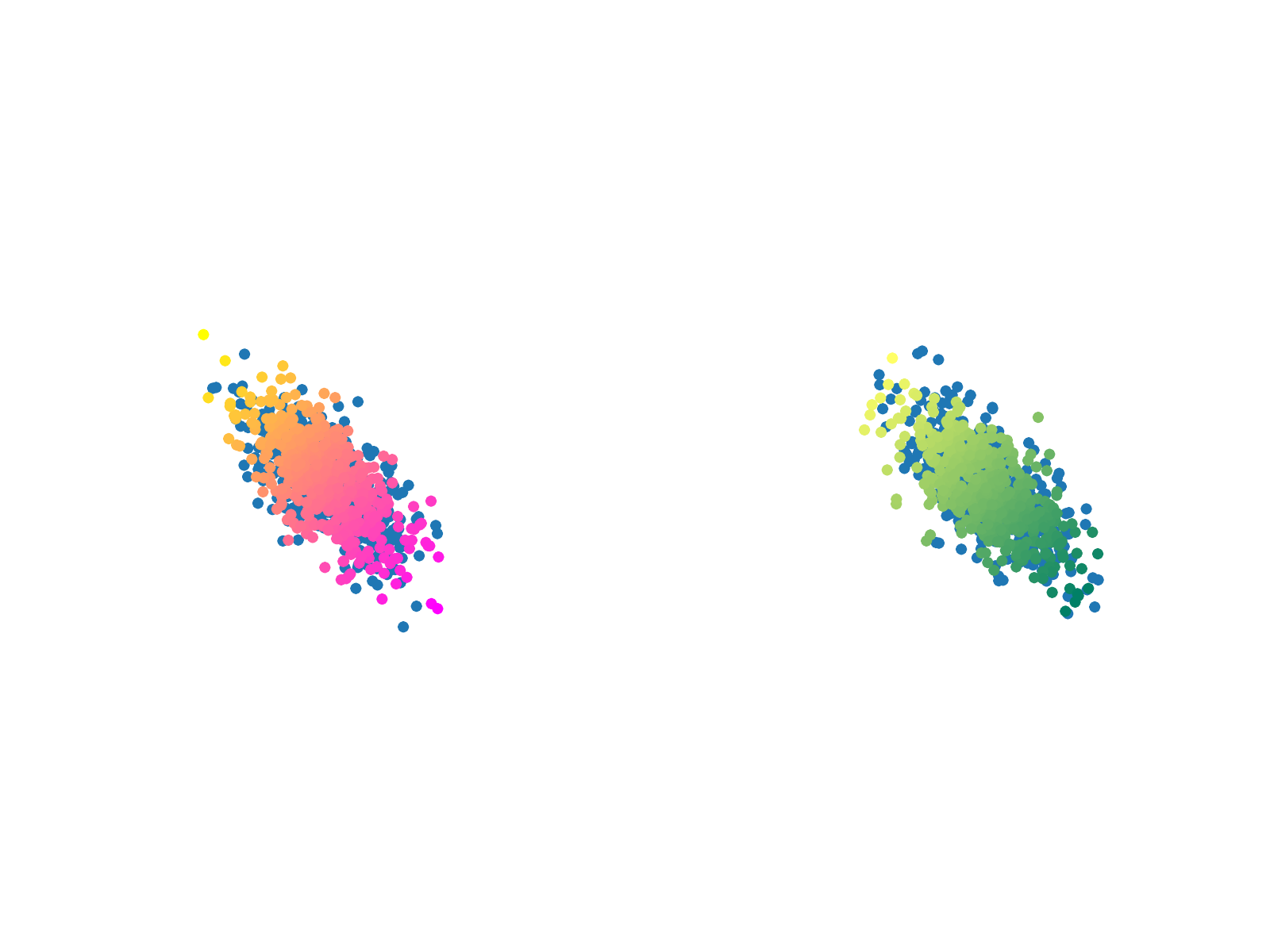} &
    \includegraphics[width=0.30\textwidth]{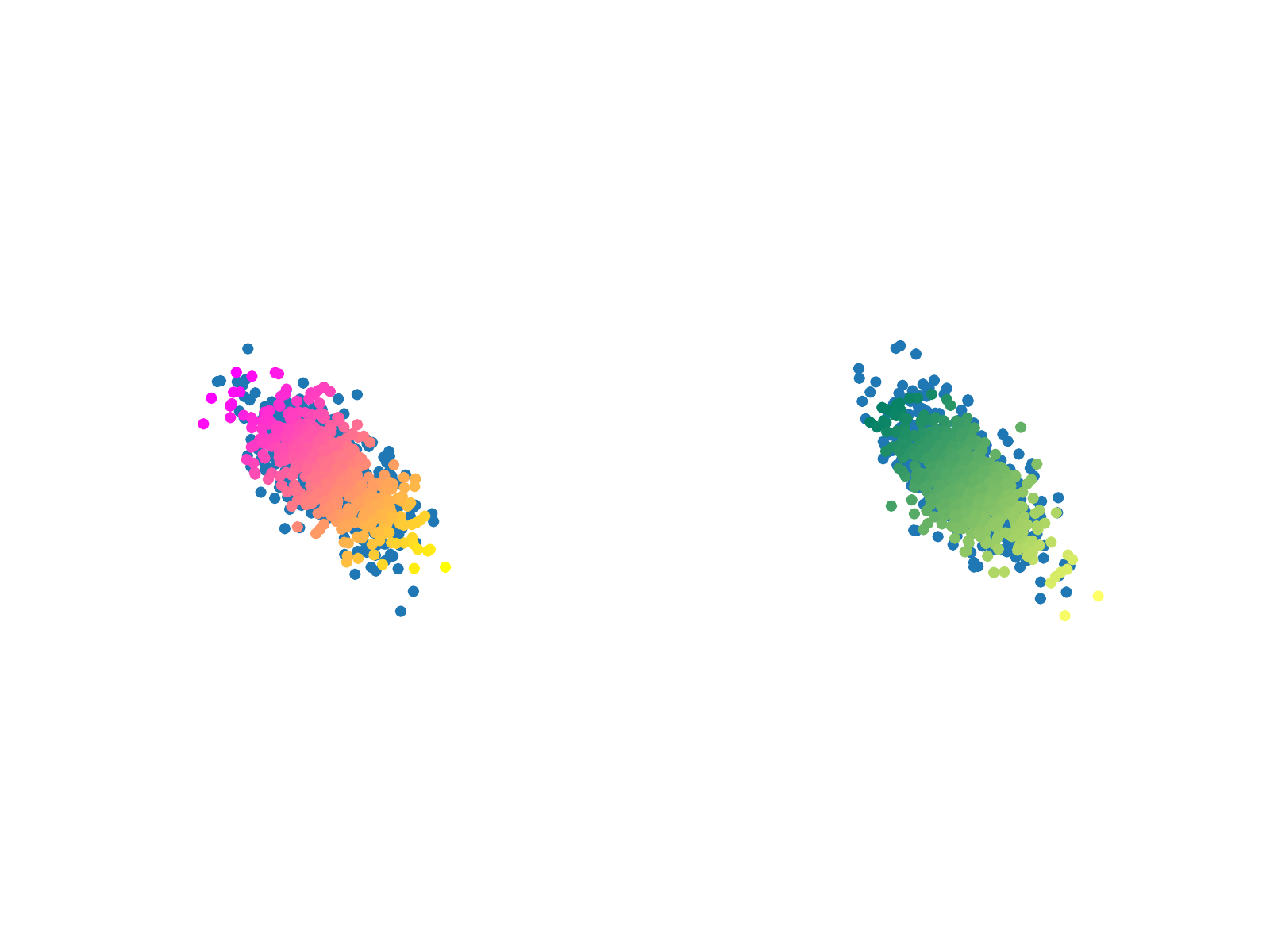} 
  \end{tabular}
\caption{Left: two discrete distributions $ \hat{\mu} $ (in gradient of colors) and
$ \hat{\nu} $ (in blue) that have been drawn from two GMMs. The colors have been added
to $ \hat{\mu} $ in order to visualize the couplings between $ \hat{\mu} $ and $ \hat{\nu} $. 
Middle and right: two possible solutions of transport of $ \hat{\mu} $ obtained by plugging the discrete plan that minimizes $ MGW_2 $
in \eqref{eq:optiplan}, using restricted-$ GW_2 $ transport maps \cite[]{salmona2021gromov} to transport the Gaussian components. Observe that the middle solution preserves the global structure of the mixture, in the sense
that points that are close to each other 
but associated with different Gaussian components remain close when tranported. This is not the case for the right solution.}\label{fig:gaussian_example}
\end{figure}

\section{Embedded Wasserstein distance}\label{sec:ew2}
In this section, we define an alternative distance to Gromov-Wasserstein 
also invariant to isometries which specifies the isometric transformation 
applied to one of the measure when computing the distance.

\begin{definition}
Let $ \mu \in \spa{P}(\rset^\d) $ and $ \nu \in \spa{P}(\rset^\di) $. For $ r \geq 1 $ and $ s \geq 1 $, let us
denote $ \mathrm{Isom}_s(\rset^r) $ the set of all isometries - for the Euclidean norm - 
from $ \rset^s $  to $ \rset^r $. We define
\begin{equation}\label{eq:ew2def}\tag{$ EW_2 $}
EW_2(\mu,\nu) =  \inf\left\{\inf_{\phi \in \mathrm{Isom}_\di(\rset^\d)} W_2(\mu,\phi_{\#}\nu),\inf_{\psi \in \mathrm{Isom}_\d(\rset^\di)} W_2(\psi_{\#}\mu,\nu)\right\}  \eqsp, 
\end{equation}
with the convention that the infinimum over an empty set is equal to $ +\infty $. 
\end{definition}
Observe that if $ \d > \di $, the set $ \mathrm{Isom}_\d(\rset^\di) $ is empty and 
so $ EW_2(\mu,\nu) = \inf_{\phi \in \mathrm{Isom}_\di(\rset^\d)} W_2(\mu,\phi_{\#}\nu) $. In contrast, 
if $ d < \di $, $ \mathrm{Isom}_\di(\rset^\d) $ is empty and so 
$ EW_2(\mu,\nu) = \inf_{\psi \in \mathrm{Isom}_\d(\rset^\di)} W_2(\psi_{\#}\mu,\nu) $. When $ d = d' $, 
the two infinimums are equivalent. \textbf{In all what follows, we will suppose without loss of generality that $ \d \geq \di $}.

\subsection{Properties of $ EW_2 $} 
We present here three properties of $ EW_2 $ which motivate its use 
between Gaussian mixture models. First, we start by showing 
that $ EW_2 $ defines a pseudometric that is invariant to isometries. 

\begin{proposition}\label{prop:ew2metric}
$ EW_2 $ defines a pseudometric on $ \bigsqcup_{k \geq 1} \spa{W}_2(\rset^k) $ such 
that for any  $ \mu \in \spa{W}_2(\rset^\d)$ and $ \nu \in \spa{W}_2(\rset^\di) $, $ EW_2(\mu,\nu) =  0 $ if and only if there exists an isometry $ \phi \colon \rset^\di \rightarrow \rset^\d $ such that $ \nu = \phi_{\#}\mu $. 
\end{proposition}

\begin{proof}[Sketch of proof.] 
Non-negativity and symmetry are straightforward. The triangular inequality 
can be proved observing first that the infinimum in $ \phi $ is always achieved, 
then remarking that $ EW_2 $ remains unchanged when one of the two measures 
is immersed in a third Euclidean space of greater dimension than $ \d $ and $ \di $. This makes $ EW_2 $ closely related to the distance between metric measure spaces introduced 
by \cite{sturm2006geometry} presented in \Cref{sec:oid}. See \Cref{sec:proof:ew2metric} 
for the full proof. 
\end{proof}

Now we show that $ EW_2 $ is equivalent to the OT 
distance introduced by \cite{alvarez2019towards} described in \Cref{sec:oid}
for a particular choice of transformation space $ \spa{H} $. In all what follows, 
we denote $ \mathbb{V}_\di(\rset^\d) $  the \emph{Stiefel manifold} \cite[]{james1976topology}, 
i.e. the set of rectangular othogonal matrices of size $ \d \times \di $ such that $ P^TP = \Id_\di $. More precisely, we show the following result.

\begin{proposition}\label{prop:ew2eq}
Let $ \mu \in \spa{W}_2(\rset^\d) $ 
and $ \nu \in \spa{W}_2(\rset^\di) $ and let suppose without loss of generality $ \d \geq \di $. 
Then,
\begin{equation}\label{eq:ew2def2}
EW^2_2(\mu,\nu) = \inf_{\op \in \Pi(\mu,\nu)} \inf_{P \in \mathbb{V}_\di(\rset^\d), \ b \in \rset^\d} \int_{\rset^\d \times \rset^\di} \|x-Py-b\|^2 \rmd \op(x,y) \eqsp.
\end{equation}
Moreover this latter problem is equivalent to 
\begin{equation}\label{eq:nucnorm}\tag{$*$-COV}
\sup_{\op \in \Pi(\bar{\mu},\bar{\nu}) } \left\|\int_{\rset^\d \times \rset^\di} xy^T \rmd \op(x,y)\right\|_* \eqsp.
\end{equation}
\end{proposition}

\begin{proof}[Sketch of proof.] Equation \eqref{eq:ew2def2} is a consequence of
\cite[Lemma 3.3]{delon2021generalized} and of 
the Mazur-Ulam theorem \cite[]{mazur1932transformations}, which implies that any
isometry from $ \rset^\di $ to $ \rset^\d $ (endowed with the Euclidean norm) is necessarily affine. The equivalence with Problem \eqref{eq:nucnorm} is then 
roughly a consequence of \cite[Lemma 4.2]{alvarez2019towards} which implies 
that Problem \eqref{eq:ew2def2} is achieved in $ P $ at 
$ P^* =  U_{\op}\Id^{[\d,\di]}_\di V_{\op}^T $  where $ U_{\op} \in \mathbb{O}(\rset^\d)$ and $ V_{\op} \in \mathbb{O}(\rset^\di) $ are the left and right orthogonal matrices
 associated with the Singular Value Decomposition (SVD) of $ \int xy^T \rmd \op(x,y) $.
See \Cref{sec:proofew2eq} for the full proof. 
\end{proof}

Since Problem \eqref{eq:nucnorm} is in general not equivalent to Problem \eqref{eq:norm2}, the $ EW_2 $ problem is in general not equivalent to 
the $ GW_2 $ problem with inner-product costs. However, the following 
result shows that between Gaussian distributions, the two problems 
share some common solutions. 

\begin{proposition}\label{thm:iw21}
Suppose without loss of generality that $ \d \geq \di $. Let $ \mu = \mathrm{N}(0,\Sigma_0) $ 
and $ \nu = \mathrm{N}(0,\Sigma_1) $ be two centered Gaussian measures on $ \rset^\d $ and $ \rset^\di $. 
Let $ P_0,D_0 $ and $ P_1,D_1 $ be the respective diagonalizations of  
$\Sigma_0 \ (= P_0D_0P_0^T)$ and $ \Sigma_1 \ (=P_1D_1P_1^T) $ that sort the eigenvalues in non-increasing order.
We suppose that $ \mu $ is not degenerate, i.e. $ \Sigma_0 $ is non-singular. Then the problem 
\begin{equation}
EW_2(\mu,\nu) = \inf_{P \in \mathbb{V}_\di(\rset^\d)} W_2(\mu,P_{\#}\nu) \eqsp,
\end{equation}
admits solutions of the form $ (\op^*,P^*) $ with $ P^* $ of the form $ P^* = P_0\widetilde{I}_\di^{[\d,\di]}P_1^T $ 
and $ \op^* = (\Id_\d,T)_{\#}\mu $ with $ T $ being any affine map
such that for all $ x \in \rset^d $,
\begin{equation}\label{eq:Tsol}
T(x) =  P_1 \left( \widetilde{I}_\di D_1^{\frac{1}{2}}{D_0^{(\di)}}^{-\frac{1}{2}} \right)^{[\di,d]}P^T_0x \eqsp.
\end{equation}
In other terms, the solutions of Problem \eqref{eq:gromov} with inner-product costs exhibited in \cite{salmona2021gromov}
are also solutions of Problem \eqref{eq:ew2def}. 
Furthermore, 
\begin{equation}
EW^2_2(\mu,\nu) = \mathrm{tr}(D_0) + \mathrm{tr}(D_1) - 2\mathrm{tr}({D_0^{(\di)}}^{\frac{1}{2}}D_1^{\frac{1}{2}}) \eqsp. 
\end{equation}
\end{proposition}

\begin{proof}[Sketch of proof.]
The proof of this result is inspired from the proof of Equation \eqref{eq:w2gaussian}
given by \cite{OTGaussian} that is based on Lagrangian analysis. The main 
difference with the proof of Equation \eqref{eq:w2gaussian} lies in 
the introduction of an additional variable $ P $ with constraint $ P \in \mathbb{V}_\di(\rset^\d) $. See \Cref{sec:proofiw21} for the full proof. 
\end{proof}

Note that it is not clear if the two problems are strictly equivalent 
or only share some common solutions because
it is not clear, to the best of our knowledge, 
that the solutions exhibited above are the only 
solutions of the $ GW_2 $ problem with inner-products costs, 
see \cite{salmona2021gromov} for more details. To complete
this section, we emphasize that $ EW_2 $ is different from the distance proposed in 
\cite{cai2022distances}, that we call here $ PW_2 $ for Projection Wasserstein discrepancy. Details about this difference can be found in \Cref{sec:addresultspw2}.

\subsection{Embedded Wasserstein distance between GMMs}

Similarly to \cite{delon2020wasserstein}, one can define 
an OT distance derived from $ EW_2 $ when $ \mu $ and $ \nu $ are GMMs by 
restricting the set of admissible couplings to be themselves GMMs. 

\begin{definition}
Let $ \mu \in GMM_K(\rset^\d) $ and $ \nu \in GMM_L(\rset^\di) $ and suppose that 
$ \d \geq \di $. We define
\begin{equation}\label{eq:mew2}
MEW_2(\mu,\nu) = \inf\left\{\inf_{\phi \in \mathrm{Isom}_\di(\rset^\d)} MW_2(\mu,\phi_{\#}\nu),\inf_{\psi \in \mathrm{Isom}_\d(\rset^\di)} MW_2(\psi_{\#}\mu,\nu)\right\} \eqsp.
\end{equation} 
\end{definition}
As before, one can reformulate this latter problem by observing 
that the isomorphic mappings for the Euclidean norm are necessarily 
of the form $ Px + b $ with $ P \in \mathbb{V}_\di(\rset^\d) $ and $ b \in \rset^\d $.
Similarly to $ EW_2 $, one can show that the infinimum in $ \phi $ is always achieved
and that $ MEW_2 $ satisfies all the properties of a pseudometric on $ \spa{GMM}_{\infty} $ by simply replacing $ W_2$ by $ MW_2 $ in the proof of \Cref{prop:ew2metric}. 
Supposing without loss of generality that $ \d \geq \di $ and using 
the equivalent discrete formulation  \eqref{eq:mw2discret} of the $ MW_2 $ problem, we get 
that for $ \mu = \sum_{k}a_k\mu_k $ and $ \nu = \sum_{l}b_l\nu_l $, the problem is equivalent
to 
\begin{equation}\label{eq:mew2disc}\tag{$MEW_2$}
\inf_{P \in \mathbb{V}_\di(\rset^\d)} \inf_{\omega \in \Pi(a,b)} \sum_{k,l} \omega_{k,l}W^2_2(\mu'_k,P_{\#}\nu'_l) \eqsp,
\end{equation}
where for any $ k \leq K $ and $ l \leq L $, $ \mu'_k $ and $ \nu'_l $ 
are the Gaussian components respectively associated 
to the centered GMMs $ \bar{\mu} $ and $ \bar{\nu} $. Note that  $ \mu'_k $ and $ \nu'_l $ are 
not necessarily themselves centered. 

\subsubsection{Numerical solver}

This time, it is not possible to derive
analytically the closed form of the optimal $ P^* $ for Problem \eqref{eq:mew2disc}. However, one can still
solve the problem numerically using an alternate minimization scheme. Indeed, 
Problem \eqref{eq:mew2disc} is not convex in $ P $ and $ \omega $, but is convex in $ \omega $ 
if $ P $ is fixed and is furthermore a simple small-scale discrete OT problem in that case, which motivates the use of an alternating optimization 
scheme for solving this problem. However, Problem \eqref{eq:mew2disc} is not convex 
in $ P $ for a fixed $ \omega $ because the feasible set, i.e. the Stiefel manifold $ \mathbb{V}_{\di}(\rset^\d) $,
is not convex.  For a fixed $ \omega $, the 
minimization in $ P $ 
can be done by projected gradient descent \cite[]{calamai1987projected}, i.e. for a given iterate 
$ P^{\{i\}} $ and a given $ \omega $, the 
next iterate $ P^{\{i+1\}} $ is given by
\begin{equation}
P^{\{i+1\}} = \kappa_{\mathbb{V}_\di(\rset^\d)}\left(P^{\{i\}} - \eta\frac{\partial J_{\omega}(P^{\{i\}})}{\partial P}\right) \eqsp,
\end{equation}
where $ \kappa_{\mathbb{V}_\di(\rset^\d)} $ is the projection mapping on the Stiefel manifold,
where $ \eta > 0 $ and where for all matrices $ P $ of size $ \di \times \d $, 
$ J_{\omega}(P) =   \sum_{k,l} \omega_{k,l}W_2(\mu'_k,P_{\#}\nu'_l) $. 
Observe that we have, as a byproduct of  \Cref{prop:ew2eq}, that for all $ P $ of size $ \di \times \d $, the projection $ \kappa_{\mathbb{V}_\di(\rset^\d)} $ is written 
\begin{equation}
\kappa_{\mathbb{V}_\di(\rset^\d)}(P) = U_P\Id_\di^{[\d,\di]}V_P^T \eqsp,
\end{equation}
where $ U_P \in \mathbb{O}(\rset^\d) $ and $ V_P \in \mathbb{O}(\rset^\di) $ are respectively the left and right orthogonal matrices associated with the SVD of $ P $.
Indeed, this projection can be written 
\begin{align}
\kappa_{\mathbb{V}_\di(\rset^\d)}(P) = \argmin_{\tilde{P} \in \mathbb{V}_\di(\rset^\d)} \| P - \tilde{P} \|_{\spa{F}}^2 = \argmin_{\tilde{P} \in \mathbb{V}_\di(\rset^\d)} \left[\|P\|_{\spa{F}}^2 +  \|\tilde{P} \|^2_{\spa{F}} - 2\mathrm{tr}(\tilde{P}^T P)\right] \eqsp.
\end{align}
Since for all $\tilde{P} \in \mathbb{V}_\di(\rset^\d) $, $ \|\tilde{P} \|^2_{\spa{F}} = \di $, we get that the problem is equivalent to $ \sup_{\tilde{P} \in \mathbb{V}_\di(\rset^\d)} \mathrm{tr}(\tilde{P}^TP) $ which 
is maximized when $ \tilde{P} = U_P\Id_\di^{[\d,\di]}V_P^T $, see the sketch of proof of \Cref{prop:ew2eq}. 
Finally, this yields to \Cref{alg:mew2}. Note that more involved optimization procedures using the specific structure of the Stiefel manifold could probably be used here~\cite[]{boumal2023introduction}. 

\begin{algorithm}
\caption{Mixture Embedded Wasserstein solver}\label{alg:mew2}
\begin{algorithmic}[1]
\Require $ \mu = \sum_k^Ka_k\mu_k $, $ \nu = \sum_l^Lb_l\nu_l $, $ P^{\{0\}} \in \mathbb{V}_\di(\rset^\d) $, $ \eta > 0 $.
\While{not converged}
 \State $ [C]_{k,l} \gets W^2_2(\mu_k,\nu_l) $ for $ k = 1,\dots,K $; $l =1,\dots,L $
  \State $ \omega^{\{i\}} \gets \Call{Solve-OT}{a,b,C} $\Comment{Solve a classic OT problem.}
  \While{not converged}\Comment{Do projected gradient descent on $ P $.}
    \State $ A \gets P^{\{i-1\}} - \eta\partial J_{\omega^{\{i\}}}(P^{\{i-1\}})/\partial P $
    \State $ U, \Sigma, V^T \gets \Call{SVD}{A} $
    \State $ P^{\{i\}} \gets U\Id_\di^{[\d,\di]}V^T $
  \EndWhile
\EndWhile \\
\Return $ \omega $, $ P $
\end{algorithmic}
\end{algorithm}

When $ \mu $ and $ \nu $ are only composed of non-degenerate Gaussian components, one can 
compute $ \partial J_{\omega}(P)/\partial P $ either by using automatic differentiation \cite[]{baydin2018automatic}
or by using the following technical result, whose proof is postponed to \Cref{sec:proofs}.

\begin{lemma}\label{prop:deriv}
Let for any $ 1\leq k \leq K $, $ \mu_k = \mathrm{N}(m_{0k},\Sigma_{0k}) $ with $ m_{0k} \in \rset^\d $ 
and $ \Sigma_{0k} \in \mathbb{S}^\d_{++} $ and for any $ 1\leq l \leq L $, $ \nu_l = \mathrm{N}(m_{1l},\Sigma_{1l}) $
with $ m_{1l} \in \rset^\di $ and $ \Sigma_{1l} \in \mathbb{S}^\di_{++} $. For any $ \omega $ in the 
$ K \times L $ simplex, let $ J_{\omega} \colon \rset^{\d \times \di} \rightarrow \rset $ be the functional 
defined, for all matrix $ P $ of size $ \d \times \di $, by
\begin{equation}
J_{\omega}(P) = \sum_{k,l} \omega_{k,l}W^2_2(\mu_k,P_{\#}\nu_l) \eqsp.
\end{equation}
Then for any full-rank matrix $ P $ of size $ \d \times \di $, we have
\begin{equation}
\frac{\partial J_{\omega}(P)}{\partial P} = 2\sum_{k,l} \omega_{k,l}\left[Pm_{1l}m^T_{1l} - m_{0k}m^T_{1l}  - \Sigma_{0k}P\Sigma_{1l}^{\frac{1}{2}}(\Sigma_{1l}^{\frac{1}{2}}P^T\Sigma_{0k}P\Sigma_{1l}^{\frac{1}{2}})^{-\frac{1}{2}}\Sigma_{1l}^{\frac{1}{2}}\right] \eqsp. 
\end{equation}
\end{lemma}

\paragraph{Initialization procedure.} 
Since the problem is non-convex, the solution to which \Cref{alg:mew2} converges 
strongly depends on the initialization of $ P $. It is therefore
crucial to design a good initialization procedure. To do so, we 
propose to use the \emph{annealing} scheme introduced by \cite{alvarez2019towards}. 
More precisely, we propose to set the initial $ P $ as the solution 
of the following iterative procedure. First we solve 
an entropic-regularized $ W_2 $ problem between 
the two discrete measures $ \mu^{\circ} = \sum_k a_k \delta_{m_{0k}} $ and $ \nu^{\circ} = \sum_k b_l \delta_{m_{1l}} $
with a large value of regularization $ \varepsilon_0 $ in order 
to obtain a coupling $ \omega^{\{1\}} $. Then 
we set
\begin{equation}
P^{\{1\}} = \textstyle{\kappa_{\mathbb{V}_\di(\rset^d)}\left(\sum_{k,l} \omega^{\{1\}}_{k,l} m_{0k}m_{1l}^T\right)}\eqsp.
\end{equation}
We then solve another entropic-regularized $ W_2 $ problem, this time 
between  $ \mu^{\circ} $ and $ P^{\{1\}}_{\#}\nu^{\circ} $,  
using a smaller value of regularization $ \varepsilon_1 = \alpha \times \varepsilon_0 $ 
with $ \alpha \in (0,1) $. We obtain thus a new coupling $ \omega^{\{2\}} $ and 
we can then derive $ P^{\{2\}} $ as previously. We repeat this procedure $ N_{it} $ times until 
the regularization term $ \varepsilon_{N_{it}} $ becomes small enough. This boils down 
to \Cref{alg:initpro}.

\begin{algorithm}
\caption{Annealed initialization procedure for mixture embedded Wasserstein}\label{alg:initpro}
\begin{algorithmic}[1]
\Require $ a $, $ b $, $ \{m_{0k}\}_k^K $, $ \{m_{1l}\}_l^L $, $\varepsilon_0 > 0 $, $ \alpha \in (0,1) $, $ P^{\{0\}} = \Id_\di^{[\d,\di]} $
\For{$i=1,\dots,N_{it}$}
\State $ [C]_{k,l} \gets \|m_{0k} - P^{\{i-1\}}m_{1l}\|^2 $
\State $ \omega^{\{i\}} \gets \Call{$ \varepsilon$-OT}{a,b,C,\varepsilon_{i-1}} $\Comment{Solve a regularized OT problem.}
\State $ A \gets \sum_{k,l} \omega^{\{i\}}_{k,l} m_{0k}m_{1l}^T $
\State $ U, \Sigma, V^T \gets \Call{SVD}{A} $ 
\State $ P^{\{i\}} = U\Id_{\di}^{[\d,\di]}V^T $
\State $ \varepsilon_i \gets \alpha\varepsilon_{i-1} $\Comment{Annealing scheme.}
\EndFor \\
\Return $ P $
\end{algorithmic}
\end{algorithm}
\noindent In practice, we set in all our experiments
$ \alpha =  0.95 $ and $\varepsilon_0 = 1 $ as in \cite{alvarez2019towards}. 
Furthermore we observed that in most cases, setting
$ N_{it} = 10 $ was sufficient to obtain a good initialization of $ P $ for \Cref{alg:mew2}. 

\subsubsection{Transportation plans and transportation maps}
\label{sec:transpomap}
Since \eqref{eq:mew2disc} has a continous 
equivalent formulation \eqref{eq:mew2}, one can derive from any optimal solution 
$ (\omega^*,P^*) $ of the former, an optimal solution 
$ (\op^*,\phi^*) $ of the latter. More precisely, 
we have on the one hand for all $ y \in \rset^\di $, $ \phi^*(y) = P^*y + b^* $, 
where $ b^* = \mathbb{E}_{X \sim \mu}[X] - P^*\mathbb{E}_{Y \sim \nu}[Y] $, and 
on the other hand for all $ (x,y) \in \rset^\d \times \rset^\di$,
\begin{equation}\label{eq:optiplan2}
\op^*(x,y) = \sum_{k,l}\omega^*_{k,l}p_{\mu_k}(x)\delta_{y = \psi^*\circ T^{k,l}_{W_2}(x)} \eqsp,
\end{equation}
where $ T^{k,l}_{W_2} $ is the optimal $ W_2 $ transport map 
between $ \mu'_k $ and $ P^*_{\#}\nu'_l $ (where we recall that $\mu'_k$ and $\nu'_l$ are the Gaussian components of the centered GMMs) and 
$ \psi^* \colon \rset^\d \rightarrow \rset^\di $ is defined for all $ x \in \rset^\d $
as $ \psi^*(x) = P^{*T}(x - b^*) $. As in \cite{delon2020wasserstein}, it 
is possible to define a unique assignment of each $ x $ by setting for all $ x \in \rset^\d $,
\begin{equation}
T_{\mathrm{mean}}(x) = \mathbb{E}_{(X,Y) \sim \op^*}[Y|X=x] = \textstyle{\frac{\sum_{k,l} \omega^*_{k,l}p_{\mu_k}(x)\psi^* \circ T^{k,l}_{W_2}(x)}{\sum_k a_kp_{\mu_k}(x)p_{\mu_k}(x)}} \eqsp.
\end{equation}
Note that $ T_{\mathrm{mean}} $ is not a Monge map since $ \op^* $ is not of the form $ (\Id_\d,T)_{\#}\mu $. 
In particular, $ T_{\mathrm{mean}\#}\mu $ is not equal to $ \nu $ and 
$ T_{\mathrm{mean}}$ is not necessarily the gradient of a convex function. 
When using $ MEW_2 $ to obtain an assignment between 
two sets $ \{x_i\}^M_i $ and $ \{y_j\}^N_j $ of respectively $ M $ and $ N$ vectors of $ \rset^\d $ 
and $ \rset^\di $, one can compute $ T_{\mathrm{mean}}(x_i) $ 
for each $ x_i $, and then determine which $ y_j $ is the closest of $ T_{\mathrm{mean}}(x_i) $ 
using a nearest-neighbor algorithm \cite[]{fix1951discriminatory}.

\subsubsection{Improving the \( MGW_2 \) method}
\label{sec:optiplanMGW2}
Inspired by the $ MEW_2 $ method presented above, we propose 
in this section to improve the $ MGW_2 $ method by: (i) proposing 
an annealed scheme similarly to \Cref{alg:initpro} 
in order to reduce the chances of converging to sub-optimal  
local minima, (ii) designing a transportation plan  
for $ MGW_2 $ similarly to \eqref{eq:optiplan2}. 

\paragraph{Annealing scheme.} Since Problem \eqref{eq:mgw2} is 
non-convex, we are only guaranteed to converge towards a local minimum 
when solving it with a classic non-regularized GW solver \cite[]{peyre2016gromov}. Furthermore, the 
convergence towards a particular minimum depends strongly on the initialization 
of the coupling $ \omega $. Since the discrete GW problem in $ MGW_2 $ 
is of very small scale and so not costly in itself, we propose, by anology 
with $ MEW_2 $, to use a similar annealing scheme as in \Cref{alg:initpro} 
to reduce the chance of converging to a sub-optimal local minimum. 
More precisely, this gives the following algorithm. 

\begin{algorithm}
\caption{Annealed mixture Gromov-Wasserstein solver}\label{alg:amgw2}
\begin{algorithmic}[1]
\Require $ \mu = \sum_k^Ka_k\mu_k $, $ \nu = \sum_l^Lb_l\nu_l $, $ \alpha \in (0,1) $, $ \varepsilon_0 $, $ \omega^{\{0\}} = ab^T $ 
\State $ [C^x]_{k,i} \gets W_2^2(\mu_k,\mu_i) $ for $ k = 1,\dots,K $, $ i = 1,\dots,K $
\State $ [C^y]_{l,j} \gets W_2^2(\nu_l,\nu_j) $ for $ l = 1,\dots,L $, $ j = 1,\dots,L $
\For{$ n = 1,\dots,N_{it}$}
\State $ \omega^{\{n\}} \gets \Call{$ \varepsilon $-GW}{a,b,C^x,C^y,\varepsilon_{n-1},\omega^{\{n-1\}}}$\Comment{Solve a regularized GW problem.}
\State $ \varepsilon_n \gets \alpha\varepsilon_{n-1} $\Comment{Annealing scheme.}
\EndFor \\
\Return $ \Call{Solve-GW}{a,b,C^x,C^y,\omega^{\{N_{it}\}}} $\Comment{Solve the non-regularized GW problem.}
\end{algorithmic}
\end{algorithm}
\noindent As previously, we set in our experiments $ \alpha = 0.95 $ and $ \varepsilon_0 = 1 $ 
as in \cite{alvarez2019towards} and we observed that, in toy cases 
where we know what the global minimum is, that $ N_{it} = 10 $ seemed to be 
a sufficient number of iterations to prevent the algorithm from converging 
towards a sub-optimal minimum.

\paragraph{Designing a transportation plan.}
Still by analogy with $ MEW_2 $, one can design a transportation plan for $ MGW_2 $ by defining 
a matrix $ P_{MGW_2} \in \mathbb{V}_\di(\rset^\d) $ and a vector $ b_{MGW_2} \in \rset^\d $, and 
then replacing $ \psi^* \circ T_{W_2}  $ in \eqref{eq:optiplan2} by $ \psi_{MGW_2} \circ T_{W_2}   $, where
for all $ x \in \rset^\d$, $ \psi_{MGW_2}(x) = P_{MGW_2}^T(x - b_{MGW_2}) $. More precisely, this can be done the following way. Given two GMMs 
$ \mu = \sum_k a_k\mu_k $ and $ \nu = \sum_l b_l \nu_l $ 
respectively in $ GMM_K(\rset^\d) $ and $ GMM_L(\rset^\di) $ and given the optimal discrete plan 
$ \omega^* $ solution of Problem \eqref{eq:mgw2}, 
one can define the matrix $ P_{MGW_2} $ as the solution of the following problem
\begin{equation}\label{eq:Pmgw2}
\inf_{P \in \mathbb{V}_\di(\rset^\d)} \sum_{k,l}\omega^*_{k,l}W_2^2(\mu'_k,P_{\#}\nu'_l) \eqsp, 
\end{equation}
where $ \mu'_k $ and $ \nu'_l $  are the Gaussian 
component of the centered GMMs $ \bar{\mu} $ and $ \bar{\nu} $, then we can set $ b_{MGW_2} = \mathbb{E}_{X \sim \mu}[X] - P_{MGW_2}\mathbb{E}_{Y \sim \mu}[Y] $. 
As above, this problem can be solved numerically by performing a projected gradient descent 
on $ P $, using either automatic differentiation or \Cref{prop:deriv}. This is also 
a non-convex optimization problem since $ \mathbb{V}_\di(\rset^\d) $ is non-convex and 
so the solution given by the projected gradient descent depends on the initialization. We propose 
thus to initialize with the projection on the Stiefel manifold
of the discrete cross-covariance matrix between the means of the Gaussian components, i.e. 
\begin{equation}
P^{\{0\}}_{MGW_2} = \textstyle{\kappa_{\mathbb{V}_\di(\rset^\d)}\left(\sum_{k,l}\omega^*_{k,l}m_{0k}m^T_{1l} \right)} \eqsp.
\end{equation}
Finally, using $ P_{MGW_2} $ one can define a continous plan $ \op_{MGW_2} $ associated with the 
discrete optimal plan $ \omega^* $ solution of the $ MGW_2 $ problem similarly to \eqref{eq:optiplan2}. 
We can therefore use $ MGW_2 $ to transport distributions, using as previously $ T_{\mathrm{mean}} $. We can also, as for $ MEW_2 $, use $ MGW_2 $  
to obtain an assignment between two sets of points.

\section{Experiments}
\label{sec:expe} 
In what follows, we use $ MGW_2 $ and $ MEW_2 $ to solve Gromov-Wasserstein 
related tasks on various datasets. More precisely, we apply 
first the two methods on simple toy low-dimensional GMMs. 
Then, we show that both methods can be used to solve relatively efficiently 
GW related tasks on real datasets in moderate to large scale settings involving 
sometimes several tens of thousands of points, both for evaluating 
distances between clouds of points and drawing correspondences between 
points. In all our experiments, we use the numerical
solvers provided by the Python Optimal Transport (POT) 
package\footnote{The package is accessible here: \href{https://pythonot.github.io/}{https://pythonot.github.io/}.}
\cite[]{flamary2021pot} that implements solvers for 
the non-regularized and regularized classic OT and GW problems. Code is available \href{https://github.com/AntoineSalmona/MixtureGromovWasserstein}{here}\footnote{\href{https://github.com/AntoineSalmona/MixtureGromovWasserstein}{https://github.com/AntoineSalmona/MixtureGromovWasserstein}}.

\subsection{Low dimensional GMMs}
In \Cref{fig:gaussian_example_2}, 
we use again the  example of \Cref{fig:gaussian_example} and 
we derive an optimal transport plan for the $ MGW_2 $ problem as 
described in \Cref{sec:optiplanMGW2}. We also show the plan obtained by solving the $ EW_2 $ problem. One can see that with both solutions,  
the global structure of the distribution is preserved in the sense that points that are closed to each other but in two different Gaussian components have been sent
to points that are also close to each other but in different Gaussian components.  

\begin{figure}[!ht]
  \centering
  \begin{tabular}{p{0.30\textwidth}p{0.30\textwidth}p{0.30\textwidth}}
    \centering \textbf{data} & \centering $ \boldsymbol{MGW_2} $ & \centering  $\boldsymbol{MEW_2}$ 
  \end{tabular}
  \\
  \begin{tabular}{c|c|c}
 \includegraphics[width=0.30\textwidth]{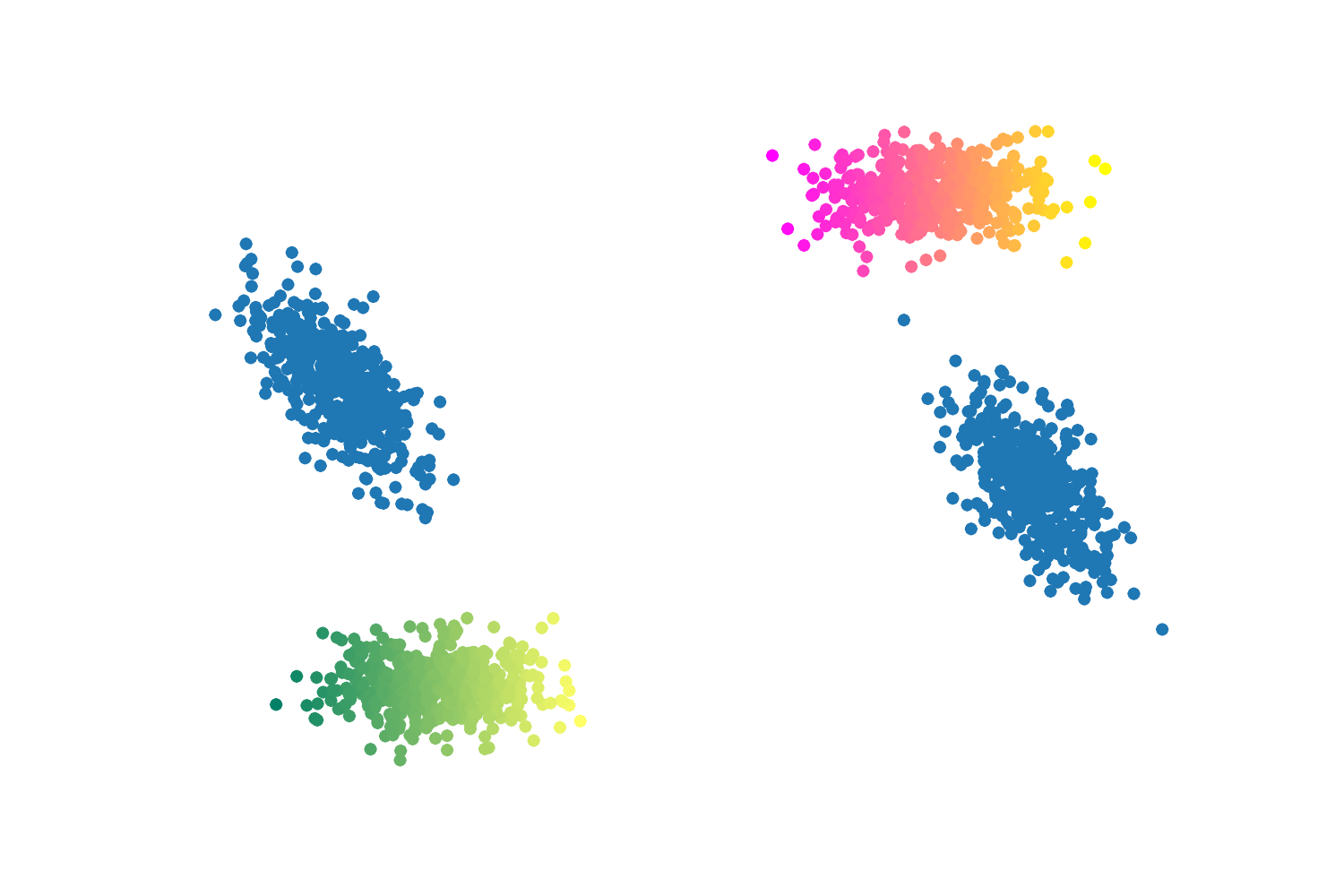} &
 \includegraphics[width=0.30\textwidth]{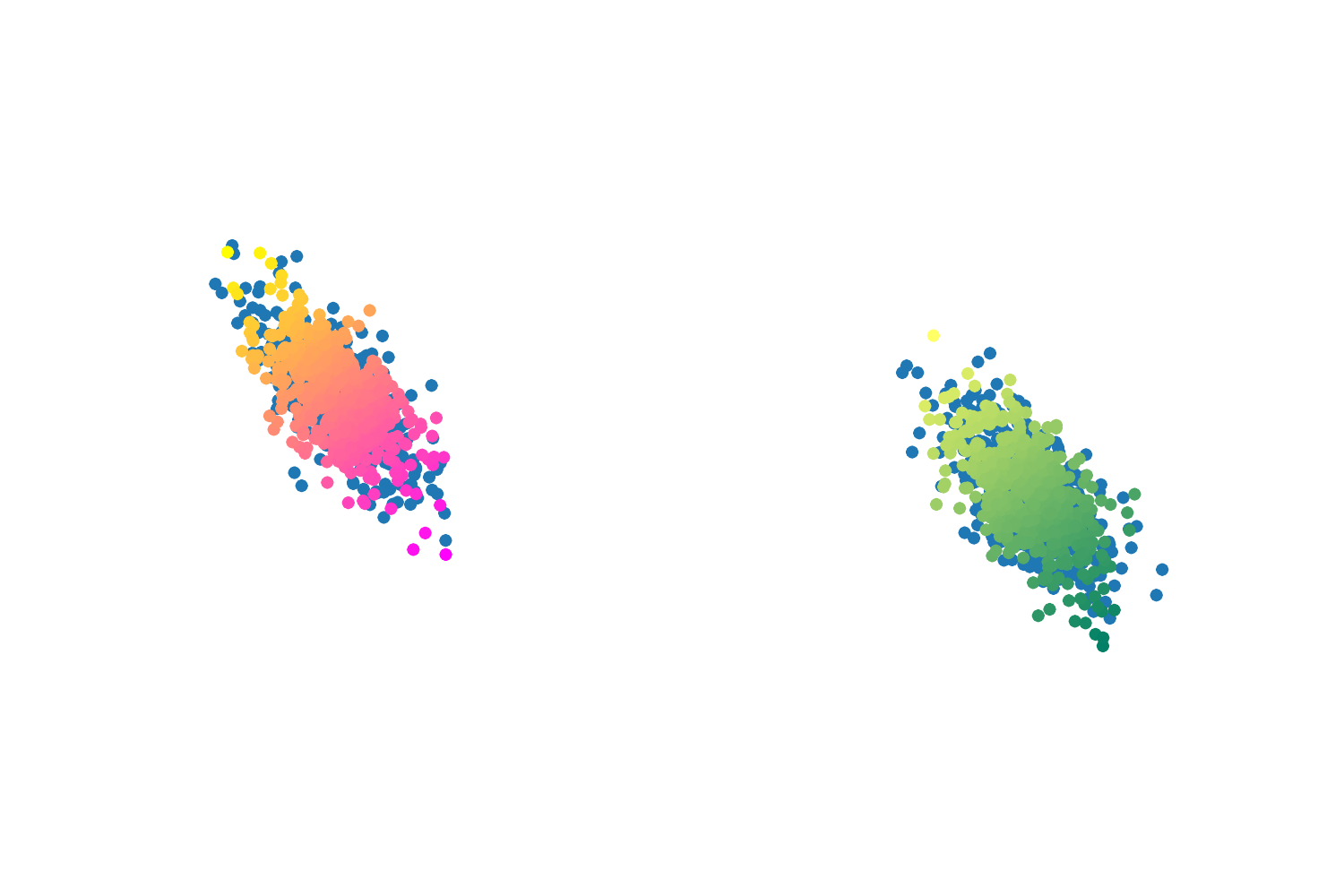} &
 \includegraphics[width=0.30\textwidth]{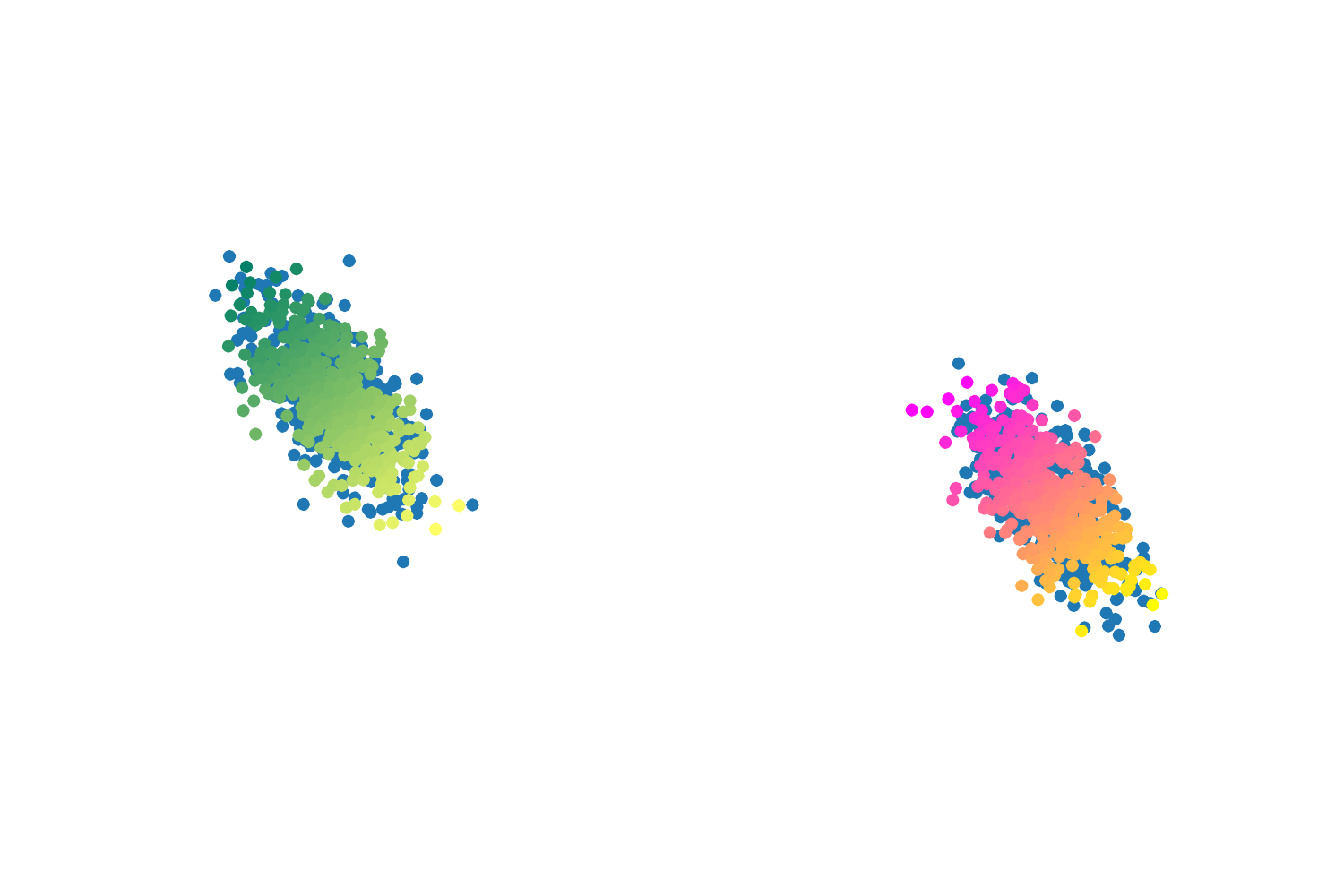} 
   \end{tabular}
 \caption{Left: two discrete distributions $ \hat{\mu} $ (in gradient of colors) and
 $ \hat{\nu} $ (in blue) that have been drawn from two GMMs. The colors have been added
 to $ \hat{\mu} $ in order to visualize the couplings between $ \hat{\mu} $ and $ \hat{\nu} $. 
 Middle: transport of $ \hat{\mu} $ obtained by solving the $ MGW_2 $ problem, then
 deriving $ P_{MGW_2} \in \mathbb{V}_2(\rset^2) $ by solving Problem \eqref{eq:Pmgw2}. 
 Right: transport of $ \hat{\mu} $ obtained by solving the $ MEW_2 $ problem.}\label{fig:gaussian_example_2}
\end{figure}

\subsection{Distances between clouds of points}
In this section, we illustrate the usability of our methods to 
assess distances between clouds of points. First, we reproduce 
an experiment originally conducted in \cite{rustamov2013map} and 
presented in \cite{solomon2016entropic} with the use of entropic-regularized GW, that 
aims to recover the cyclical nature of a horse's gallop. Then, 
we perform a comparison between runtimes of $ MGW_2 $ and other methods
existing in the literature that provide a GW-type distance between  point clouds.

\paragraph{Galloping horse sequence.} Here we repoduce 
the experiment of the galloping horse, that has been originally conducted in \cite{rustamov2013map} and 
presented in \cite{solomon2016entropic} with the use of entropic-regularized GW.
In this experiment, we compute a matrix of pairwise distances (either for $MGW_2$ or $MEW_2$) between 
 $45 $ meshes representing a galloping horse. Then, we conduct 
a Multi-Dimensional Scaling (MDS) \cite[]{borg2005modern} - which roughly can be thought 
as a generalization of PCA - of the $45 \times 45$ matrix of pairwise distances between meshes, in order to plot each 
mesh as a $ 2$-dimensional point.  \Cref{fig:gallopinghorse} shows these $ 2$-dimensional embeddings of the sequence. 
As observed in \cite{solomon2016entropic}, the interesting part here is that these points are positioned in a cyclical fashion, which means that the original set of pairwise distances seem to respect the periodic aspect of the sequence (both for $ MGW_2 $ and $ MEW_2 $). Each 
mesh is composed of approximately $ 9000 $ vertices and the average 
time to compute one distance when using the POT implementation 
of the entropic-regularized GW solver is around $ 30 $ minutes which 
makes the computation of the full pairwise
distance matrix impractical, as mentioned in \cite{solomon2016entropic}. In constrast, 
when using our methods with GMMs with $ K = 20 $ components, it took us only
approximately $ 10 $ minutes to compute the full distance matrix using $ MGW_2 $, 
and around one hour using $ MEW_2 $, these times including the fitting with EM of all the GMMs.

\begin{figure}[!ht]
  \centering
  \begin{tabular}{p{0.32\linewidth}p{0.32\linewidth}p{0.26\linewidth}}
      \centering \textbf{$\boldsymbol{MGW_2}$} & \centering \textbf{$\boldsymbol{MEW_2}$} & \centering \textbf{Data}
  \end{tabular}
   \\
  \begin{minipage}[c]{0.65\textwidth}
  $ \begin{array}{cc}
  \includegraphics[width=0.49\textwidth]{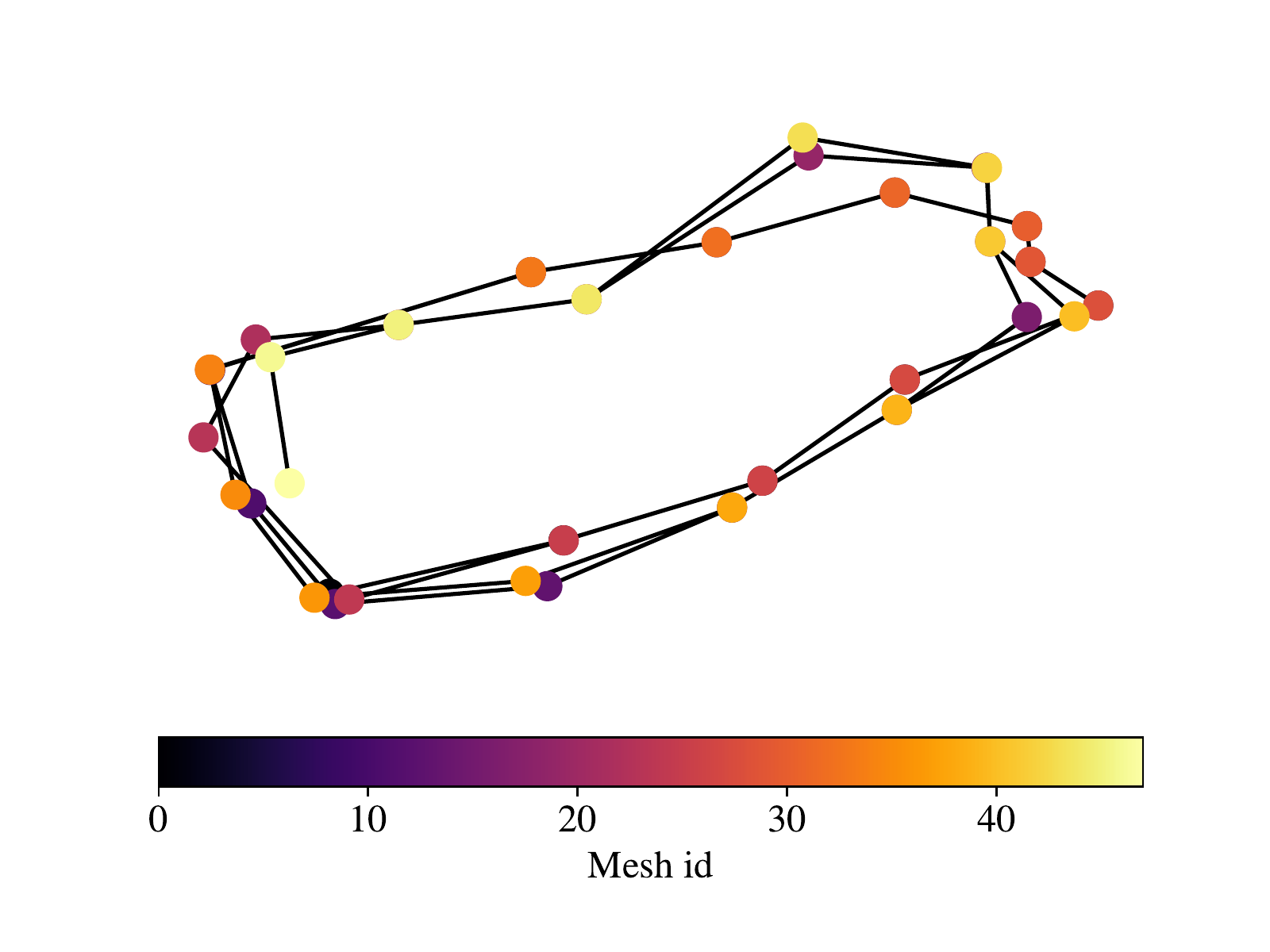} & \includegraphics[width=0.49\textwidth]{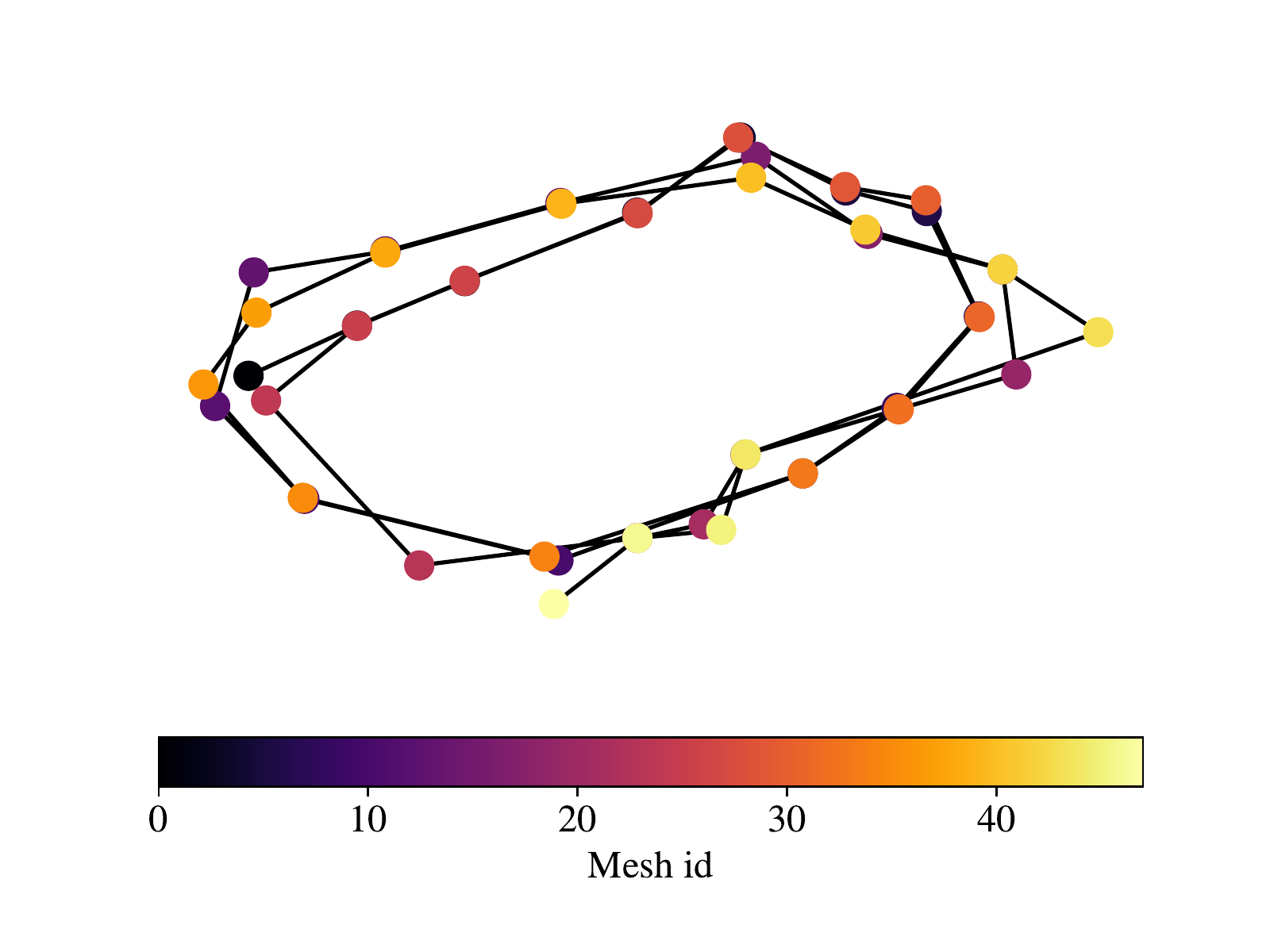} 
  \end{array} $
  \end{minipage}
  \begin{minipage}[c]{0.31\textwidth}
  \vspace{-1em}
  $ \begin{array}{cc}
  \includegraphics[width=0.48\textwidth]{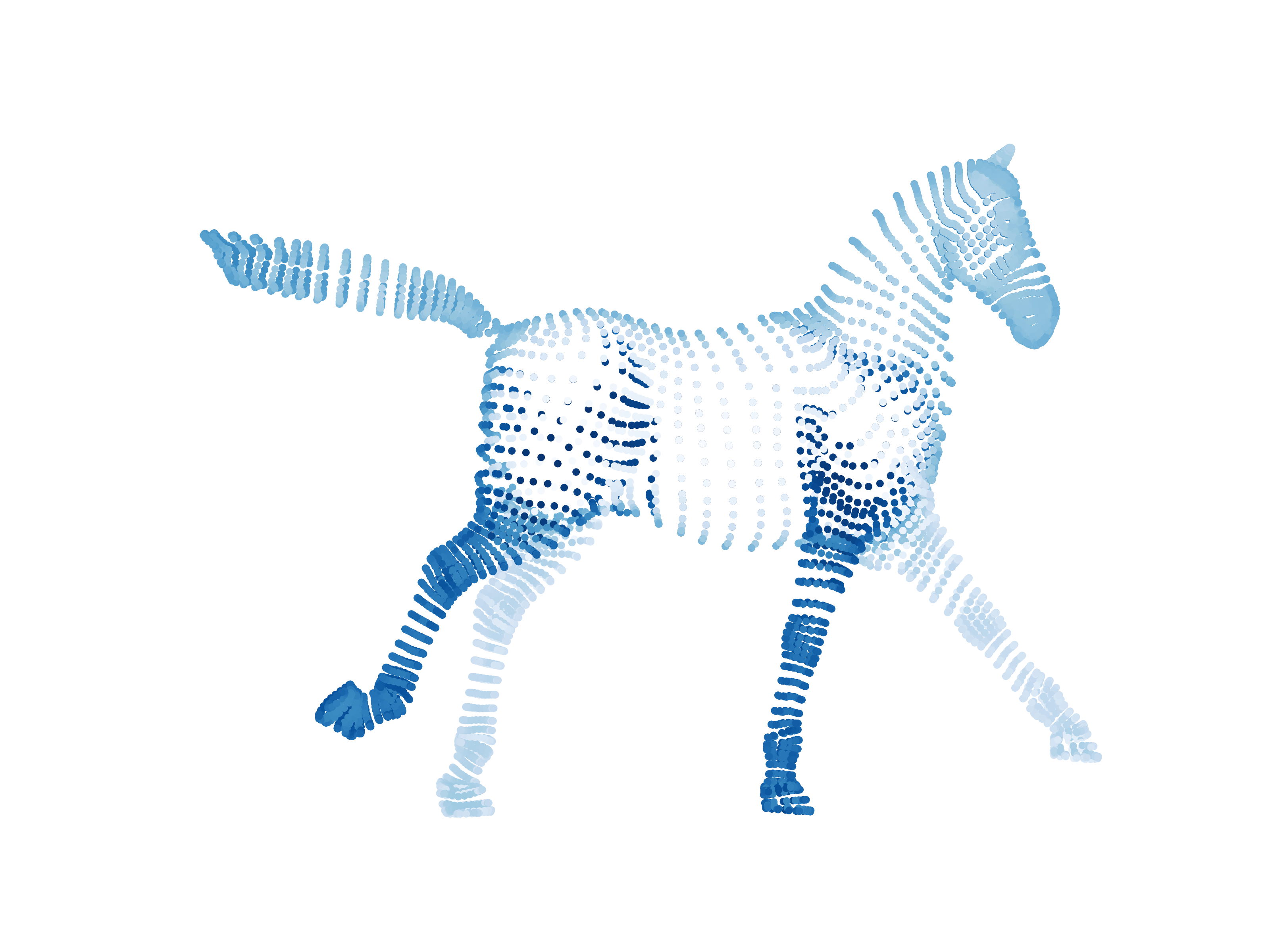} & \includegraphics[width=0.48\textwidth]{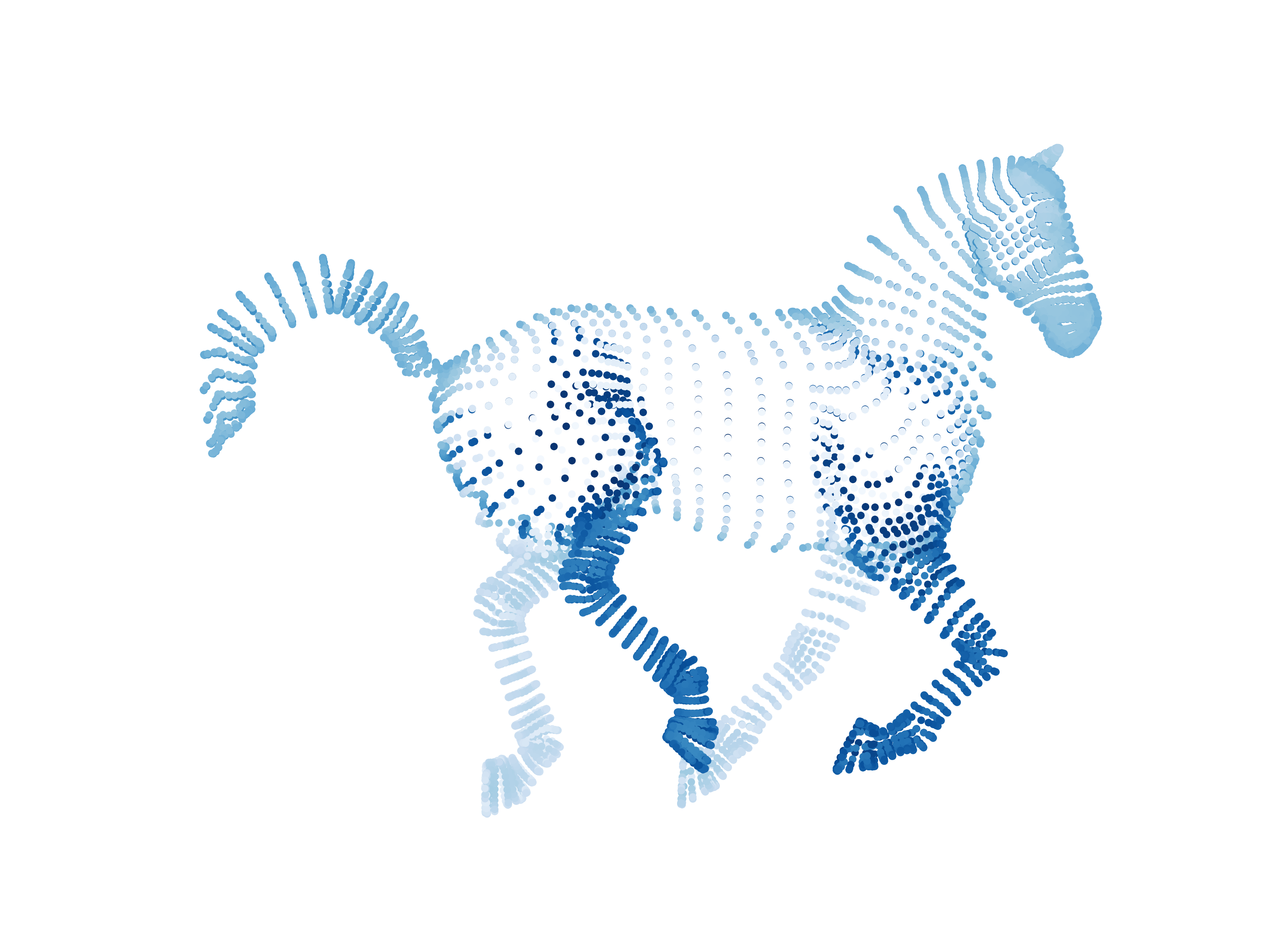} \\
  \includegraphics[width=0.48\textwidth]{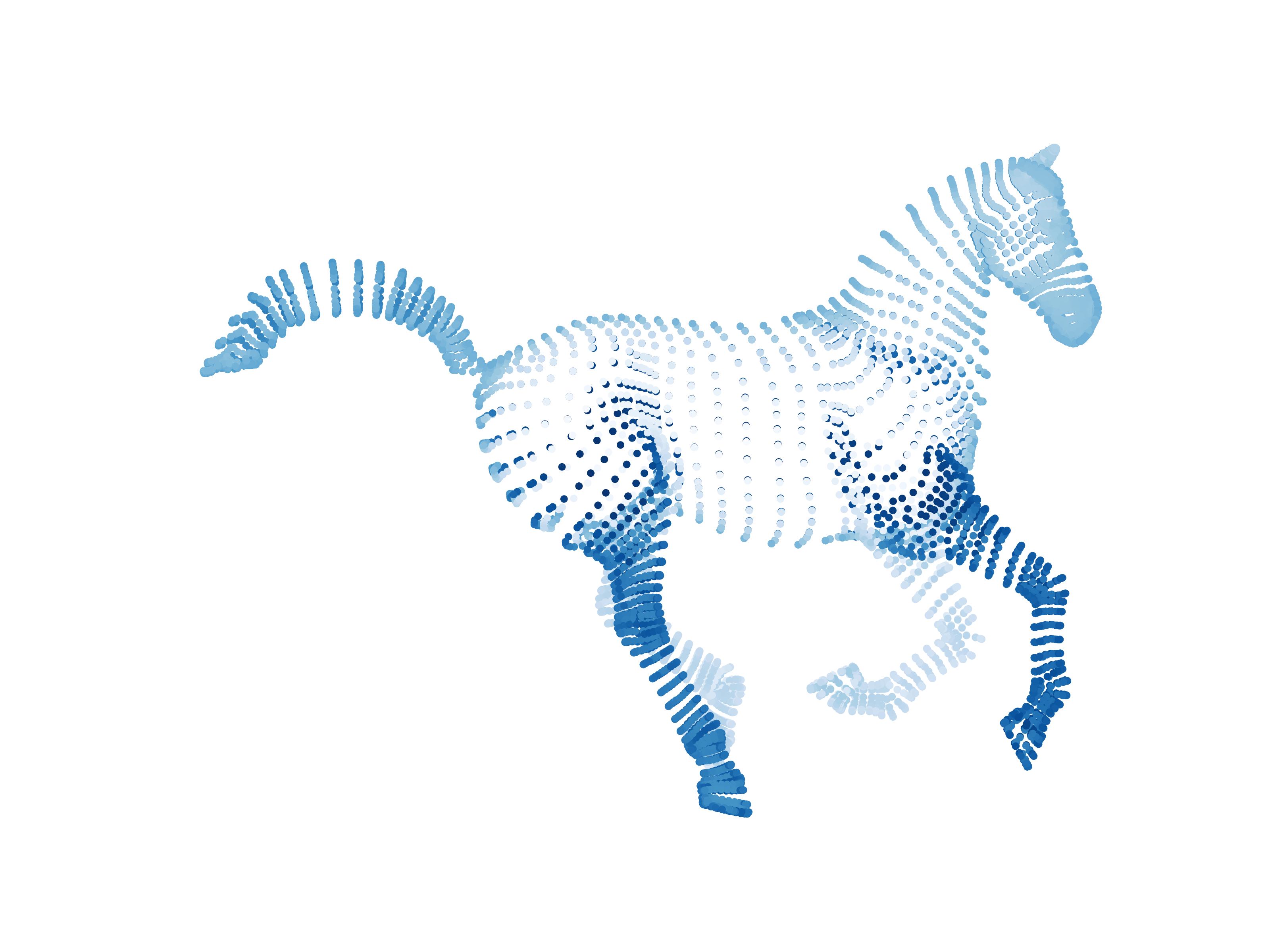} & \includegraphics[width=0.48\textwidth]{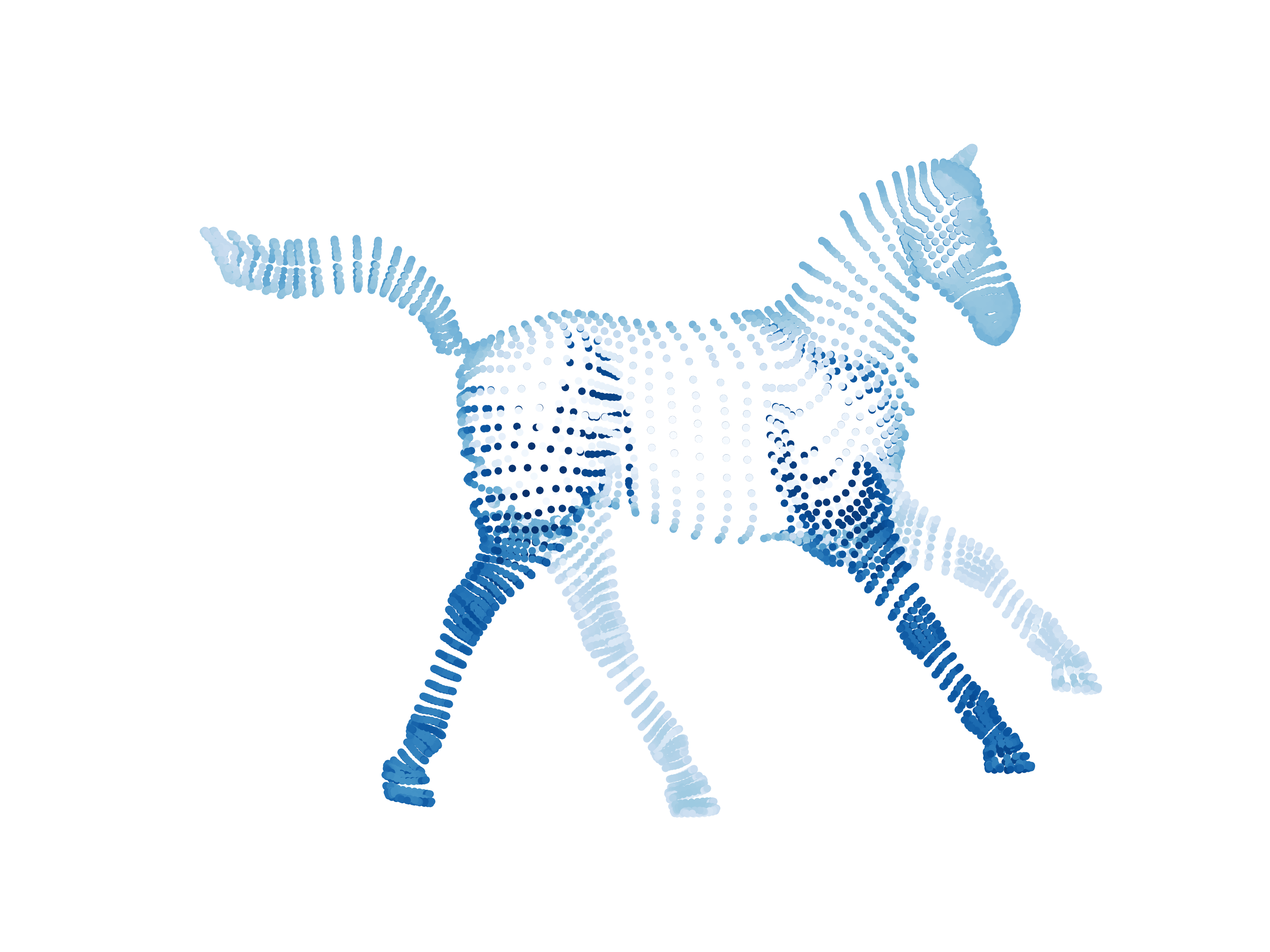}
  \end{array} $
  \end{minipage}
  \caption{MDS on the galloping horse animation using the $ MGW_2 $ distance (left), and the $ MEW_2 $ distance (middle). Each 
  point corresponds to a given mesh and the meshes are colored in function of their number in the sequence. Right: $ 4 $ examples among the $ 45 $ meshes 
  that composes the sequence. The computations of both distances have been done by first fitting GMMs with $ 20 $ components on each mesh independently.}\label{fig:gallopinghorse}
\end{figure}

\paragraph{Local minima.} To highlight the importance of using 
an annealing scheme when deriving $ MGW_2 $ or $ MEW_2 $, we have 
reconducted the previous experiment but this time without 
the annealing schemes described in \Cref{alg:amgw2} and \Cref{alg:initpro}. 
In \Cref{fig:localminima}, we plot 
the evolutions of the values of $ MGW_2^2 $ and $ MEW_2^2 $ 
between one given fixed mesh and all the others. In both 
cases, the annealing scheme seems to be useful to prevent 
the solver to converge towards sub-optimal mininima. 
However, if the $ MGW_2 $ solver seems 
to often converge to the same optimum regardless  
the use of the annealing scheme, this is not the case of 
$ MEW_2 $ which, without the annealing initialization procedure (\Cref{alg:initpro}),
converges most of the time to a sub-optimal minimum, so much 
that the periodical aspect doesn't even appear in that case. This experiment
also emphasizes the fact that when solving a GW problem with
classic non-regularized or entropic solvers from \cite{peyre2016gromov}, we are not at all guaranteed 
to converge towards a global minimum and, more critically, 
we have in general no ways to know if the solution we converged to 
is actually optimal or sub-optimal.

\begin{figure}[ht!]
  \centering
   \begin{tabular}{p{0.35\textwidth}p{0.35\textwidth}}
      \centering $ \boldsymbol{MGW_2} $ & \centering  $\boldsymbol{MEW_2}$ 
    \end{tabular}
   \\ \vspace{-3pt}
  \scalebox{0.9}{
  \begin{tabular}{cc}
  \includegraphics[width=0.40\textwidth]{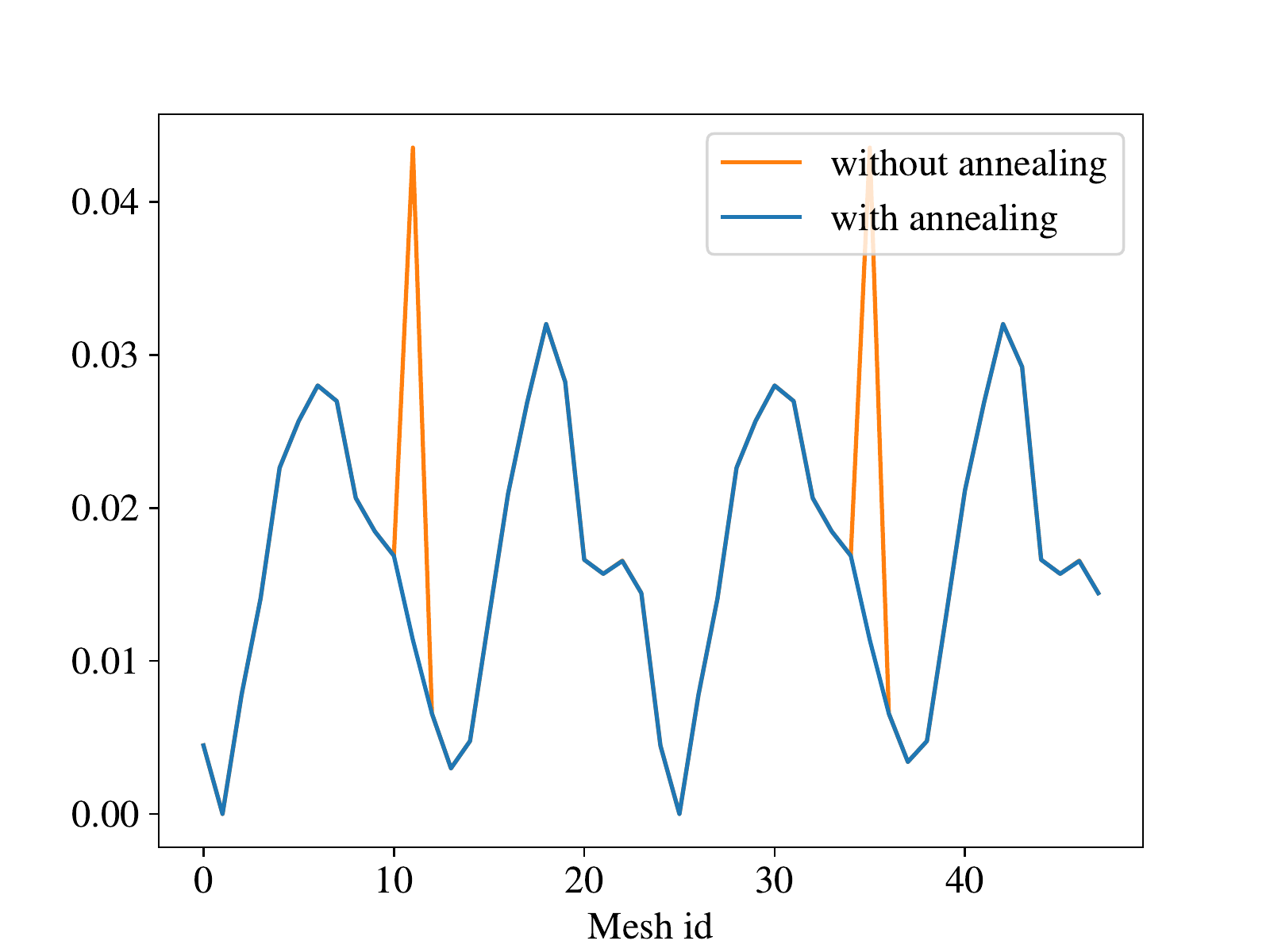} &\includegraphics[width=0.40\textwidth]{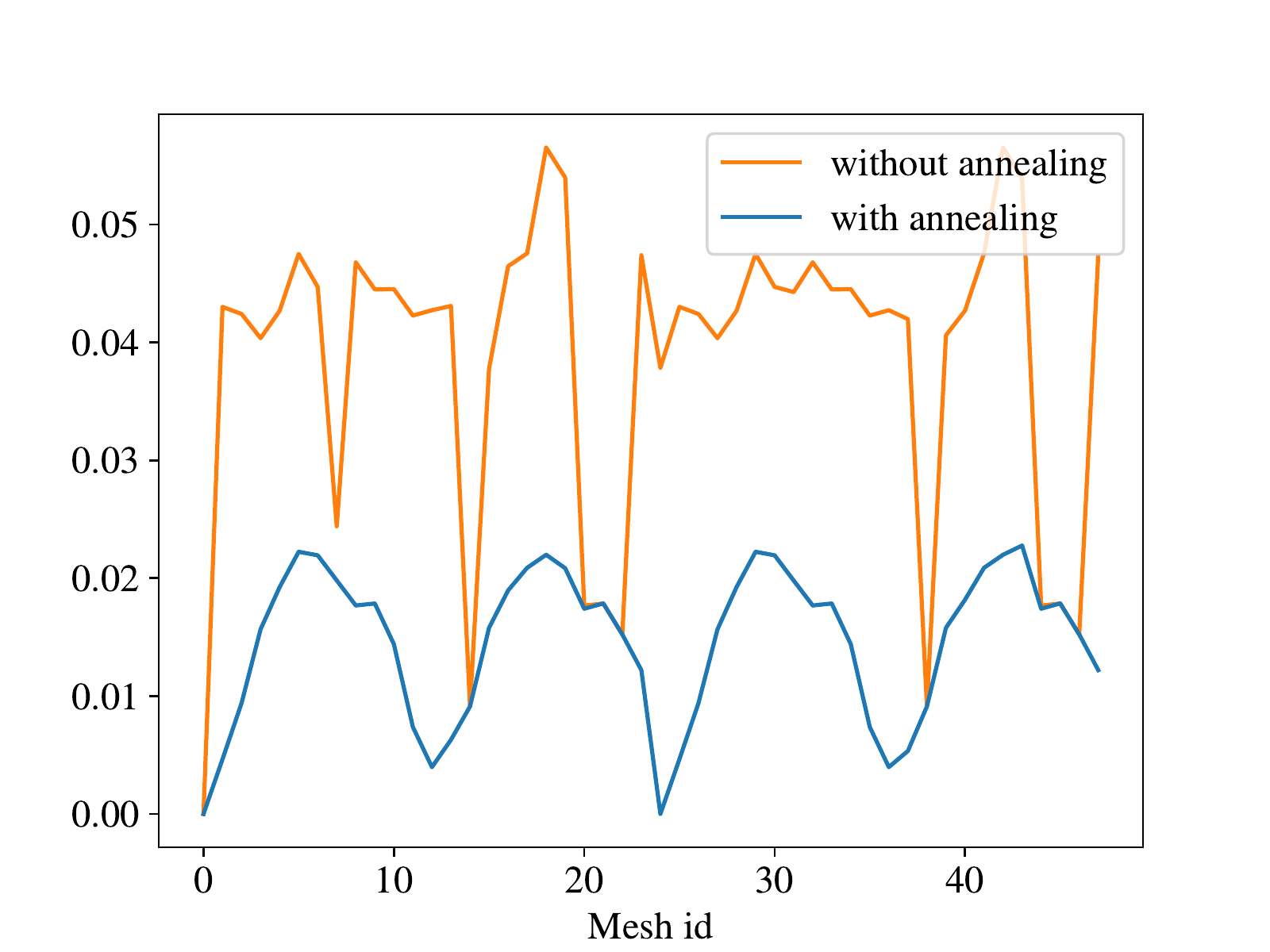}
  \end{tabular}}
  \caption{Left: Evolution of $ MGW^2_2 $ between the second mesh and all the others, 
  using an annealing scheme (\Cref{alg:amgw2}) in blue, and without the annealing scheme in orange. 
  Right: Evolution of $ MEW^2_2 $ between the first mesh and all the others, with 
  the annealing initialization procedure (\Cref{alg:initpro}) in blue, and without in orange. The
  computation of both distances has been done by first fitting GMMs with $ 20 $ components on each mesh independently.}\label{fig:localminima}
\end{figure}

\paragraph{Runtimes comparison.} We perform a comparison between runtimes of 
$ MGW_2 $, sliced GW (SGW) \cite[]{titouan2019sliced}, low-rank GW (lrGW)\cite[]{scetbon2022linear}, 
minibatch GW (mbGW) \cite[]{fatras2021minibatch}, entropic-regularized GW (erGW) \cite[]{peyre2016gromov} and 
quantized GW (qGW) \cite[]{chowdhury2021quantized} between two 2D random discrete 
distributions with varying number of points from $ n = 10^3 $ to $ n = 10^6 $. We use the codes  provided by the authors on their dedicated Github repositories. Note that $ MEW_2 $ is not included 
in this comparison as we observed in the previous 
experiment that this latter method was significantly slower than $ MGW_2 $. 
For $ MGW_2 $, we use GMMs with respectively $ K = \{10,20,50\} $ components. For SGW, 
we use the implementation on CPU with $ L = \{50,200\} $ projections. lrGW 
has a parameter $ r $ corresponding to the rank of the coupling matrix. We choose here respectively 
$ r = n/100 $  (this choice is advised by the authors of~\cite[]{scetbon2022linear} for lrGW to be a good approximation of erGW) and $ r = 100 $ (which yields a linear computational time). For mbGW, we use 
batches of size $ m = 50 $ with $ k = n/10 $ batches (these values are advised by the authors of~\cite[]{fatras2021minibatch}). For erGW, we use 
two different implementations of the method, the first one from POT and 
the second from \cite{scetbon2022linear}\footnote{\href{https://github.com/meyerscetbon/LinearGromov/blob/main/FastGromovWass.py}{https://github.com/meyerscetbon/LinearGromov/blob/main/FastGromovWass.py}},
both with regularization parameter $ \varepsilon = 0.1 $. Finally, for 
qGW, we use a proportion $ p = 0.1 $ of the points as partition 
block representatives and then we take a Voronoi partition with respect to these representatives. 
Note that this latter method only provides a coupling but we reinject it 
in the GW objective. Results 
can be found in \Cref{fig:comp}. We can observe that $ MGW_2 $ 
has similar runtimes as SGW (CPU version) and seems even a bit faster 
in large scale settings. Several algorithms fail to converge when the number of points is too large. The limits we observed are: $ 10^4 $ for both implementations 
of erGW, $ 2 \times 10^4 $ for lrGW with $ r = n/100 $, and $ 4 \times 10^4 $ 
for lrGW with $ r = 100 $. In the same way, considering our computational ressources, using qGW to compute 
a distance between the two point clouds with more than $ 3 \times 10^4 $ points was impossible  because the two full pair-to-pair distance matrices are becoming too heavy in terms of memory. Note that it is still possible to compute a coupling using qGW 
afterwards, see \cite[]{chowdhury2021quantized} for details, but it is no more possible 
to evaluate the GW objective, which necessarily 
requires to access the full pair-to-pair distance matrices.

\begin{figure}[ht!]
\centering
\includegraphics[width=\textwidth]{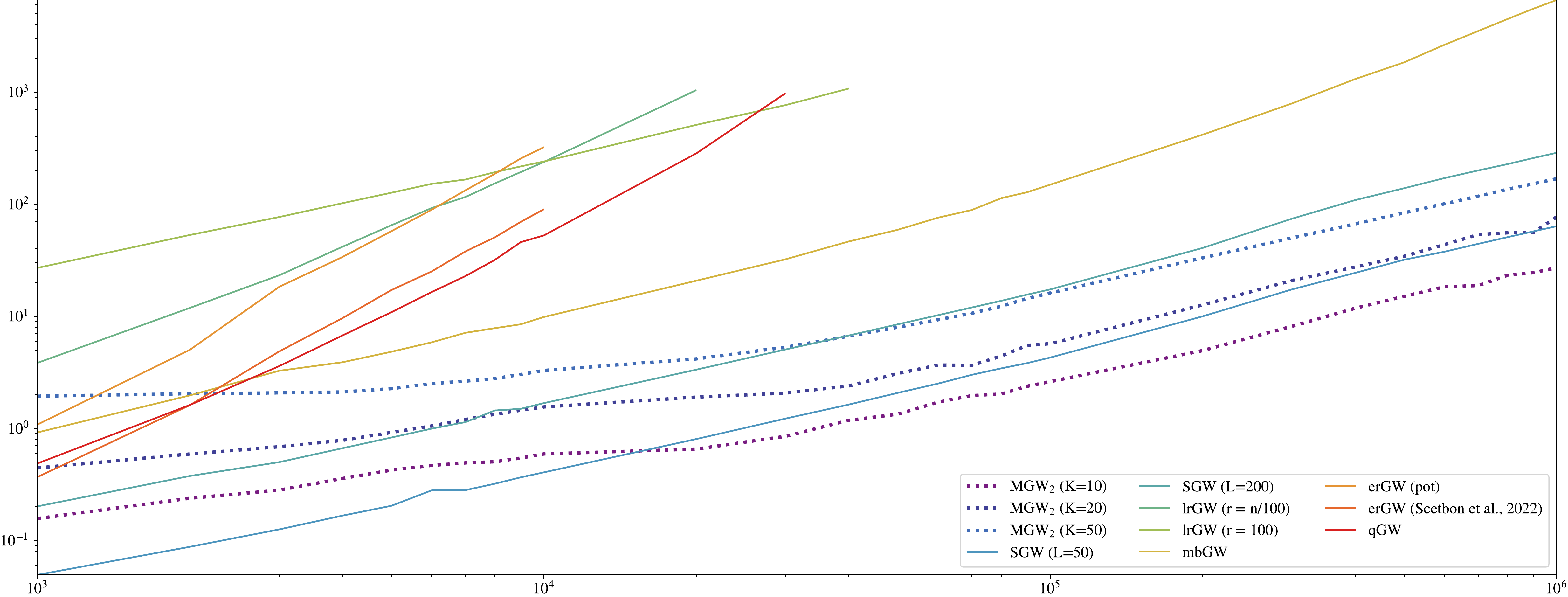}
\caption{Runtimes comparison between $ MGW_2 $, SGW \cite[]{titouan2019sliced} (CPU), 
lrGW \cite[]{scetbon2022linear}, mbGW \cite[]{fatras2021minibatch}, erGW \cite[]{peyre2016gromov} 
and qGW \cite[]{chowdhury2021quantized} between two 2D random discrete 
distributions with varying number of points from $ 10^3 $ to $ 10^6 $ in log-log scale. The time includes 
the computation of the pair-to-pair distance matrices and the fitting of the GMMs for $ MGW_2 $ (using scikit-learn).}\label{fig:comp}
\end{figure}

\subsection{Drawing correspondences between points}
In this section, we illustrate the usability of our methods to establish correspondences 
between clouds of points on two shape matching applications.

\begin{figure}[ht!]
  \centering
  \scalebox{0.9}{
  \begin{tabular}{ccc}
    \textbf{Source} & \textbf{Target ($ \boldsymbol{MGW_2} $)} &  \textbf{Target ($ \boldsymbol{MEW_2} $)} \\
  \includegraphics[width=0.32\textwidth]{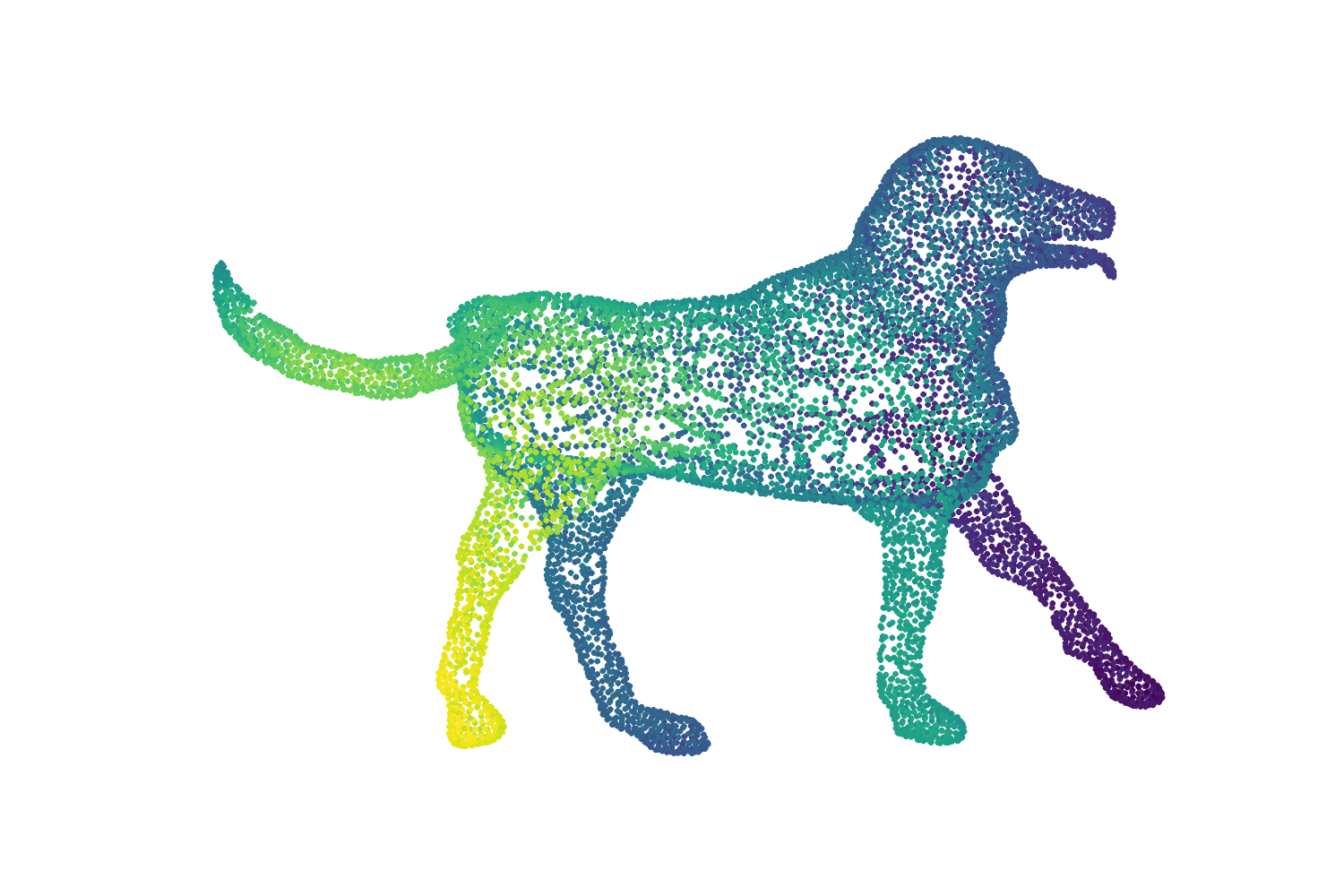} & \includegraphics[width=0.32\textwidth]{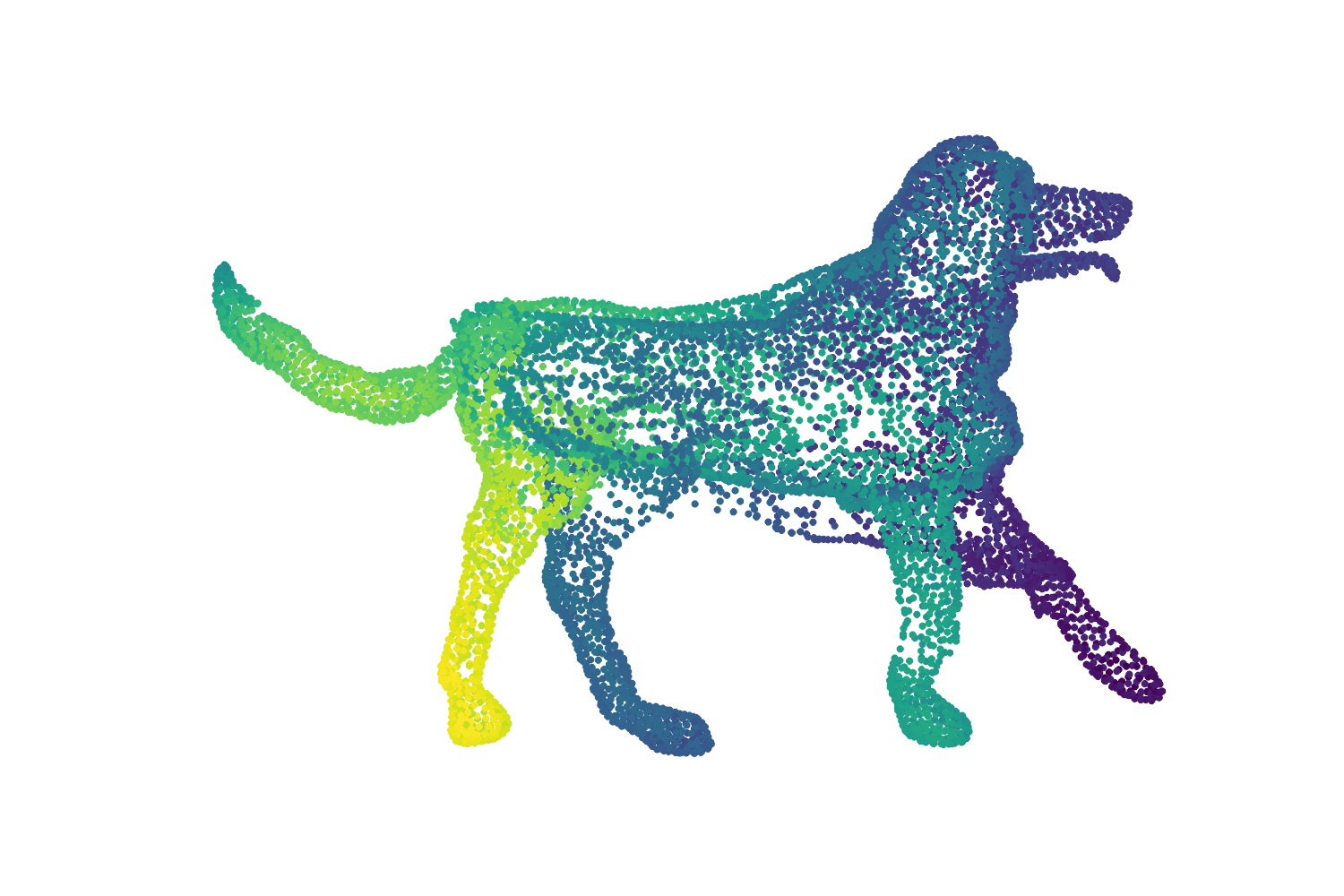} & \includegraphics[width=0.32\textwidth]{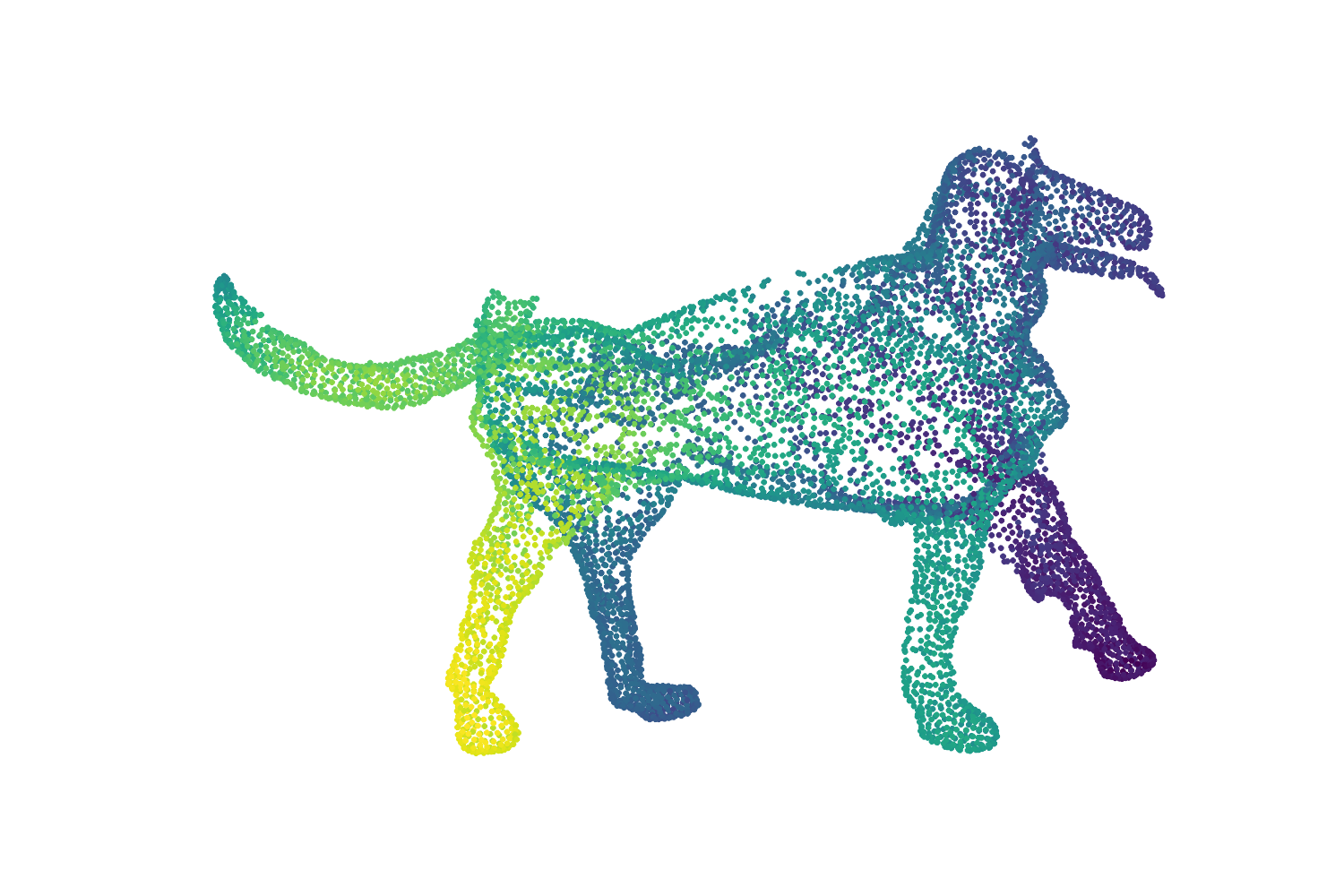}  \\
  \includegraphics[width=0.32\textwidth]{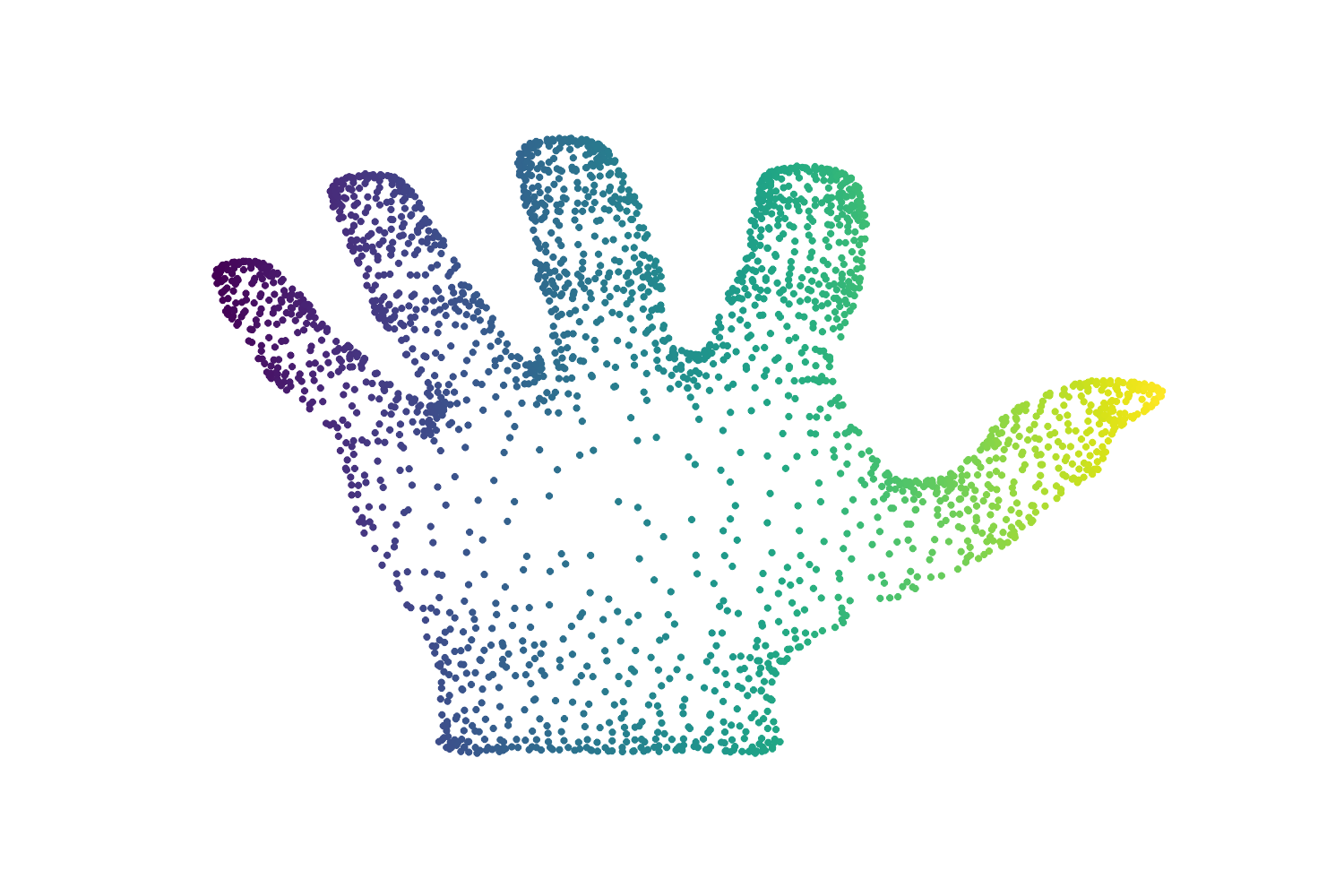} & \includegraphics[width=0.32\textwidth]{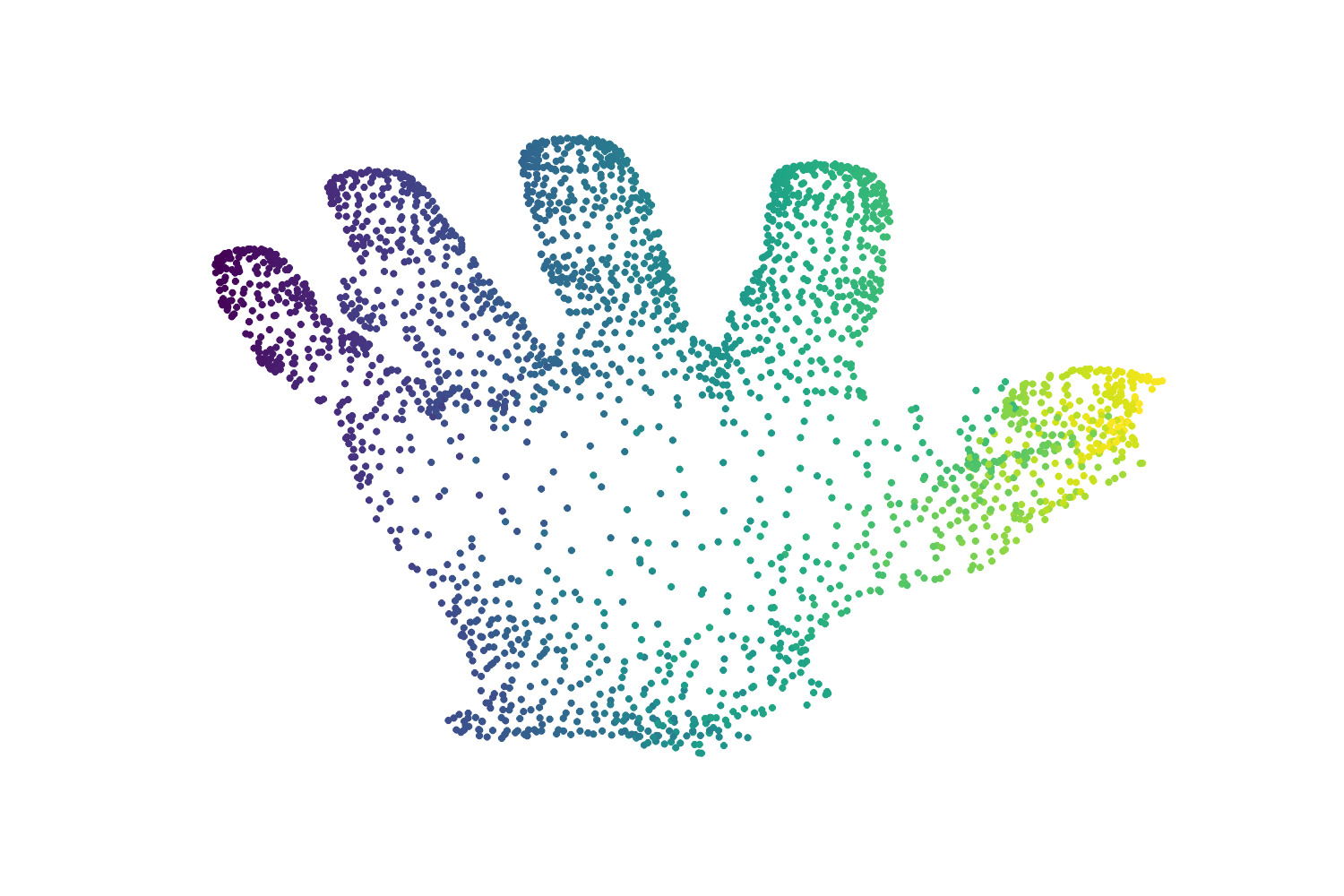} & \includegraphics[width=0.32\textwidth]{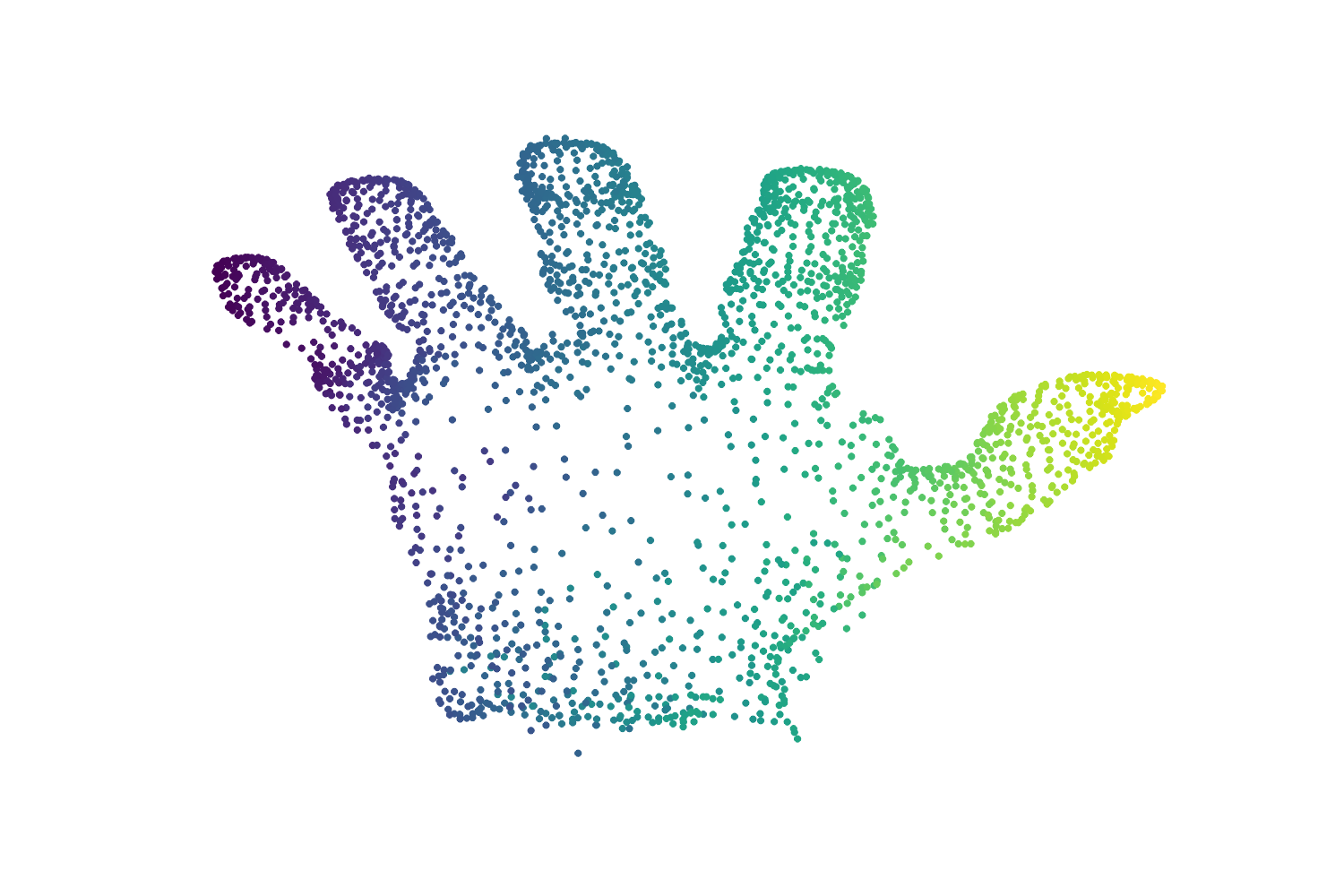} 
  \end{tabular}}
  \caption{Shape matching between shapes and their distorted versions. We plot 
  the output of the $ T_{\mathrm{mean}} $ transport map applied to the source shape on the left for respectively 
  $ MGW_2 $ (middle) and $ MEW_2 $ (right).  GMMs with $ 20 $ 
  components have been fitted independently on the source shape and on its noisy version.
  Colors have been added to visualize where the points have been transported. 
  }\label{fig:shapematch}
  \end{figure}

\paragraph{Quality of the transport map.}
We reproduce here an experiment from \cite{chowdhury2021quantized}. The goal is to match 
3D meshes from the CAPOD dataset \cite[]{papadakis2014canonically} with copies of themselves
whose vertices are permuted and perturbed randomly. 
To do so, we fit a GMM on each mesh as well as on its noisy version, 
then we derive a discrete coupling $ \omega $ between the Gaussian components using 
$ MGW_2 $ or $ MEW_2 $. Finally, we transport the points 
using $  T_{\mathrm{mean}} $ (see Section~\ref{sec:transpomap}). Two examples of matching using GMMs with $ 20 $ 
components can be found in \Cref{fig:shapematch}.
Observe that both methods $ MGW_2 $ and $ MEW_2 $ seem to be able to recover relatively well
the correct matching of the points on these examples. To complete this experiment with quantitative 
results, we compute, as in \cite{chowdhury2021quantized}, the distortion 
at each point $ x $ as the distance from its ground truth copy $ x $ and its matched point
$ y $ with a given coupling $ \op $. The distortion 
score of the coupling $ \op $ is then the mean squared distortion. We report class average 
distortion scores for $ MGW_2 $ and $ MEW_2 $ with GMMs with $ K = \{10,20,50\} $ components 
as well as computation times in \Cref{tab:disto}. Since the points at the output of 
$ T_{\mathrm{mean}} $ are not exactly corresponding to the points in the noisy version 
of the shapes, we introduce an additional nearest neighbors step in order to 
reproject the points onto the noisy shape. We also report in 
\Cref{tab:disto} class average distortion scores and computation times for qGW 
using a proportion $ p = \{0.01,0.1,0.2,0.5\} $ of the points as partition 
block representatives and using a Voronoi partition with respect to these representatives.
In terms of distortions scores, we observe 
that $ MGW_2 $ yields similar results that qGW for most classes.
We remark that the results we obtained on qGW are consistent with the ones reported in \cite[Table 1]{chowdhury2021quantized}. 
Note that \cite{chowdhury2021quantized} have shown that qGW significantly outperforms
MREC \cite[]{blumberg2020mrec} and mbGW \cite[]{fatras2021minibatch} on this task (which is why we don't include them in the comparison). In terms of 
running time, note that the situation is a little bit different 
from \Cref{fig:comp} since the $ MGW_2 $ method includes 
here an additional step of deriving a coupling between the points given the coupling between the Gaussian 
components whereas the qGW method doesn't require to assess the GW objective after 
computing the coupling anymore. We observe that both methods $ MGW_2 $ and $ MEW_2 $ are significantly 
slower than qGW in that setting. Yet the gap of running time between qGW and our methods
seems to reduce as the number of points increases.  Indeed, using $ MGW_2 $ or $ MEW_2 $ seems to largely decorrelate computation time from the number of points (since the number of Gaussian component stays fixed), without observable deterioration in terms of registration accuracy. These methods therefore make a lot of sense as the number of points increases.

\begin{table}[ht!]
\centering
\scalebox{0.8}{
\begin{tabular}{cccccccccc}
\hline
Method & Param & Humans & Planes & Spiders  & Cars & Dogs & Trees & Vases \\
& & 1926 & 2144 & 2664 & 5220 & 8937 & 10433 & 15828 \\
\hline
$ MGW_2 $ & 10 & \textbf{0.04} (17.5) & 0.40 (17.6) &	0.03 (17.7)  & \textbf{0.12} (17.8) &	0.13 (18.8) &	0.10 (18.8) &	0.34 (\underline{19.1})\\
& 20 & 0.15 (69.4) &	0.36 (69.5)  &	0.04 (69.9) &	0.17 (70.8) &	0.20 (71.9) & 0.10 (71.9) & 0.28 (73.1)\\
& 50 & 0.18 (431) &	0.10 (428) & \textbf{0.007} (431) &	0.12 (435) & 0.20 (437) &	\textbf{0.04} (438) & 0.20 (441) \\
\hline
$ MEW_2 $ & 10  & 0.09 (17.1) &	0.37 (22.9) & 0.02 (16.3) &	0.23 (17.6) & 0.20 (24.5)  &	0.11 (18.3) & 0.29 (23.2)\\
& 20  & 0.21 (77.2) &	0.39 (66.6) &	0.02 (64.1) &	0.25 (78.0) & 0.20 (85.6) &	0.13 (67.1) & 0.30 (76.8) \\
& 50 & 0.16 (555) &	0.17 (421) & 0.009 (397) &	0.20 (423) & 0.21 (462) &	0.08 (465) &	0.19 (486) \\
\hline 
$ qGW $ & 0.01  & 0.25 (\underline{0.59}) &	0.46 (\underline{0.78}) & 0.05 (\underline{1.08})  &	0.24 (\underline{3.88}) &	0.28 (\underline{11.3})  &	0.13 (\underline{17.4})&	0.28 (32.4) \\
& 0.1  & 0.16 (1.04) & 0.10 (1.33) & 0.02 (1.84) &	0.21 (5.80) & 0.02 (16.5) &	0.05 (26.8) & \textbf{0.18} (54.9)\\
& 0.2 & 0.11 (1.65) & 0.08 (2.12) & 0.01 (2.86) &	0.15 (9.25) & 0.008 (28.2) & \textbf{0.04} (53.6) &	0.21 (123)\\
& 0.5 & 0.10 (4.39) &	\textbf{0.07} (5.77) & \textbf{0.007} (7.73) &	0.16 (34.9) & \textbf{0.007} (104) & 0.15 (165) &	0.22 (418) \\
\hline
\end{tabular}}
\caption{Distortion scores (lower is better) and runtimes (in parentheses) for $ MGW_2 $, $ MEW_2 $, and qGW. The average
number of points in each shape class is provided under the shape class name. Results are listed for several parameter
choices of each method. Results have been averaged on $ 10 $ runs of the experiment.}\label{tab:disto}
\end{table}

\paragraph{Matching human shaped meshes.} To demonstrate the usability 
of our methods in larger scale settings, we use the SHREC'19 dataset \cite[]{melzi2019shrec}
that contains human shaped meshes that can sometimes be composed of
more than $ 300000 $ vertices. Our goal is to draw correspondences 
between the shapes using only the information of the vertices (the dataset
also includes edges). To do so, we first fit independently GMMs with $ 20 $ 
components on each mesh and we derive 
directly couplings at the scale of the Gaussian components
that represent the different parts of the bodies.   
In such large scale settings, the 
main bottleneck of the methods in terms of computational time 
is clearly the fitting of the GMMs that can take at worst $ 2 $ minutes 
for the meshes composed of the highest number of vertices. The results are displayed 
on \Cref{fig:shrec}. Observe that in most cases, both methods 
seem to be able to match correctly the colored parts. Yet in the last row,
$ MEW_2 $ matches a leg at the left in red to an arm at the right. This 
probably implies that the method has been trapped in a local minimum despite the 
annealing initialization procedure. Finally, note that 
we presented here cases where the methods performed relatively well, but 
there are cases where $ MGW_2 $ or $MEW_2 $ fail to 
find correct correspondences and exhibit behaviors similar to $ MEW_2 $
in the last row, which suggests that the methods converge sometimes to 
sub-optimal minima despite the annealing schemes. 

\begin{figure}[!ht]
  \centering
  \scalebox{0.9}{
  \begin{tabular}{ccc}
  \textbf{Source} & \textbf{Target ($ \boldsymbol{MGW_2} $)} &  \textbf{Target ($ \boldsymbol{MEW_2} $)} \\
  \includegraphics[width=0.32\textwidth]{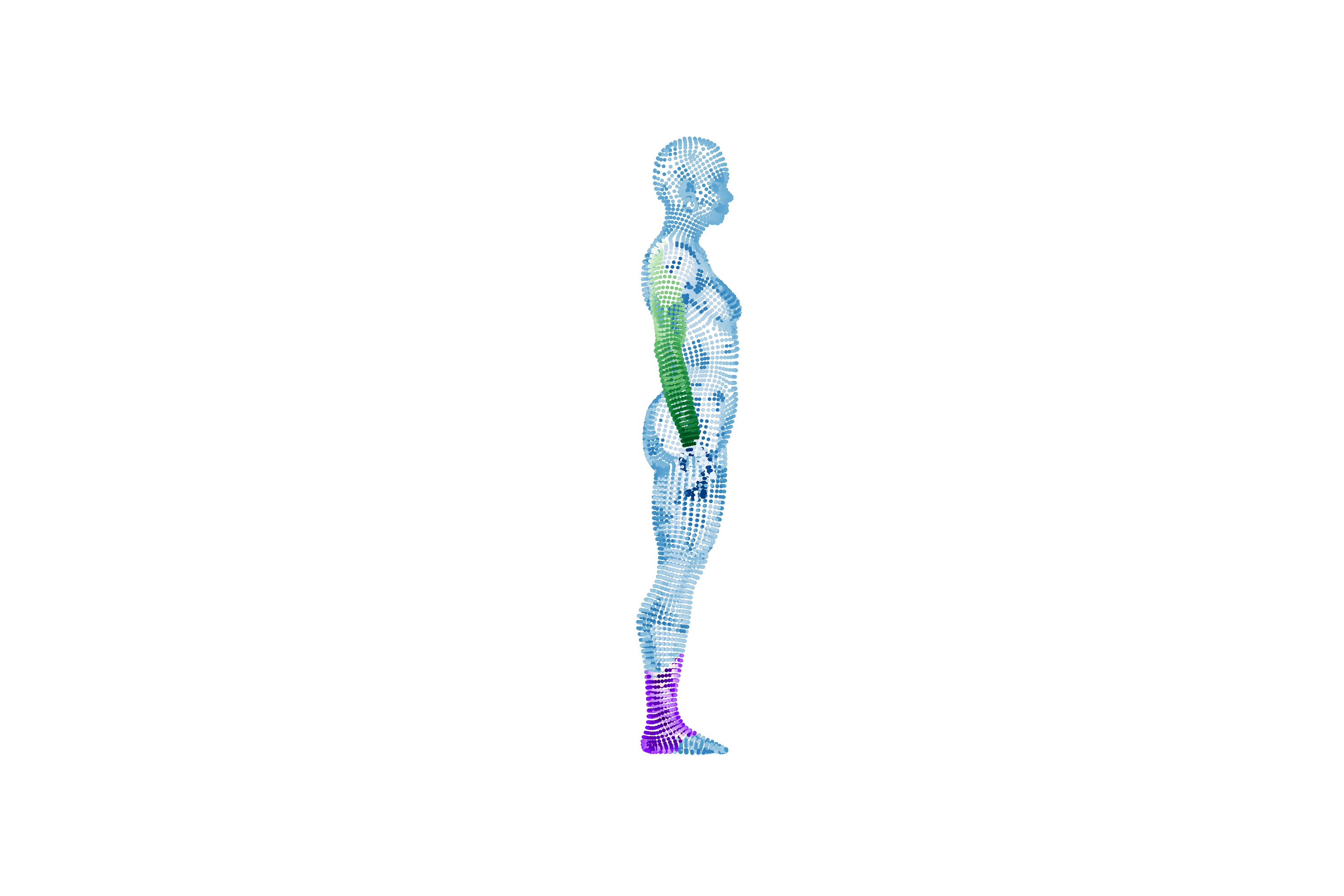} & \includegraphics[width=0.32\textwidth]{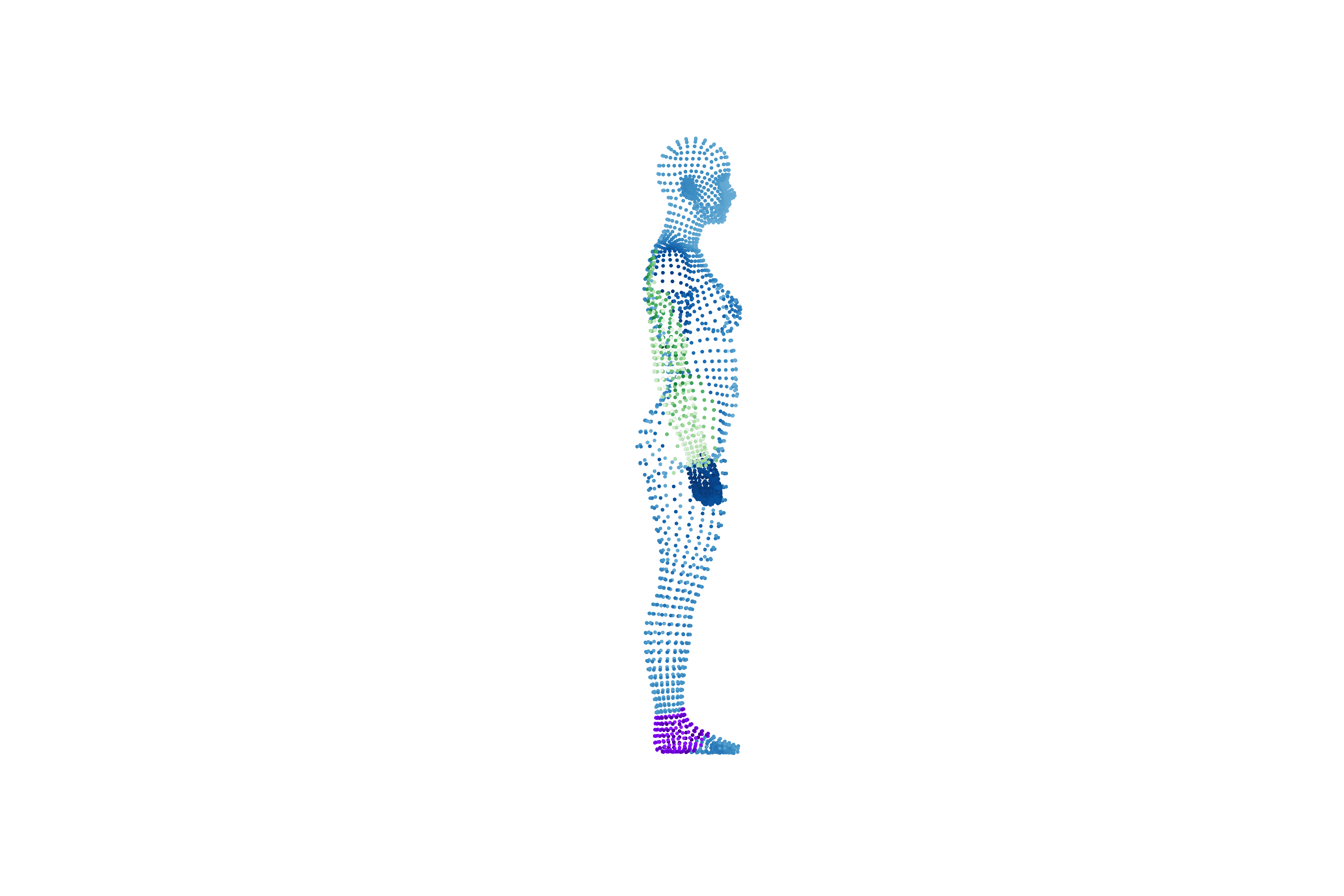} & \includegraphics[width=0.32\textwidth]{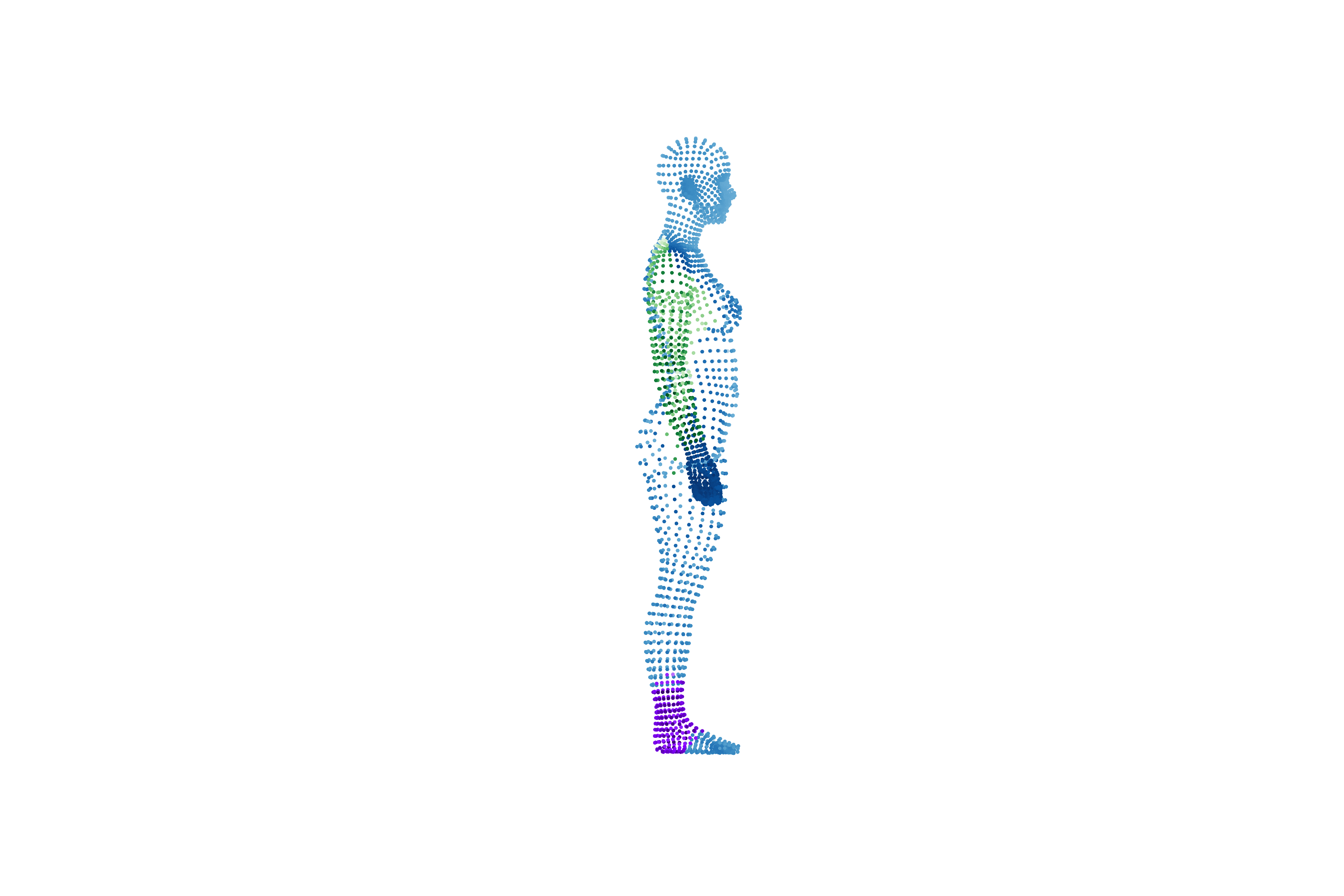}  \\
  \includegraphics[width=0.32\textwidth]{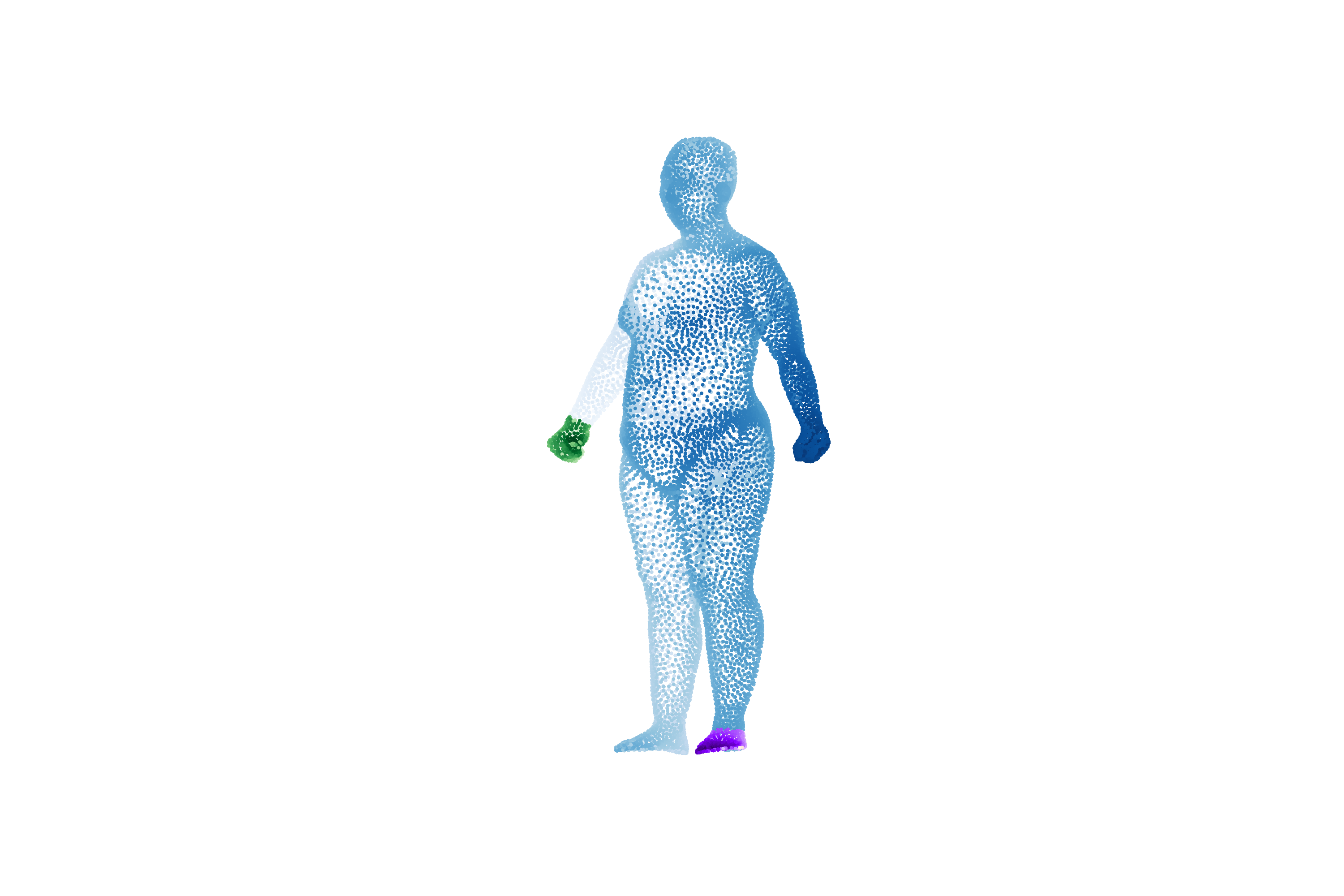} & \includegraphics[width=0.32\textwidth]{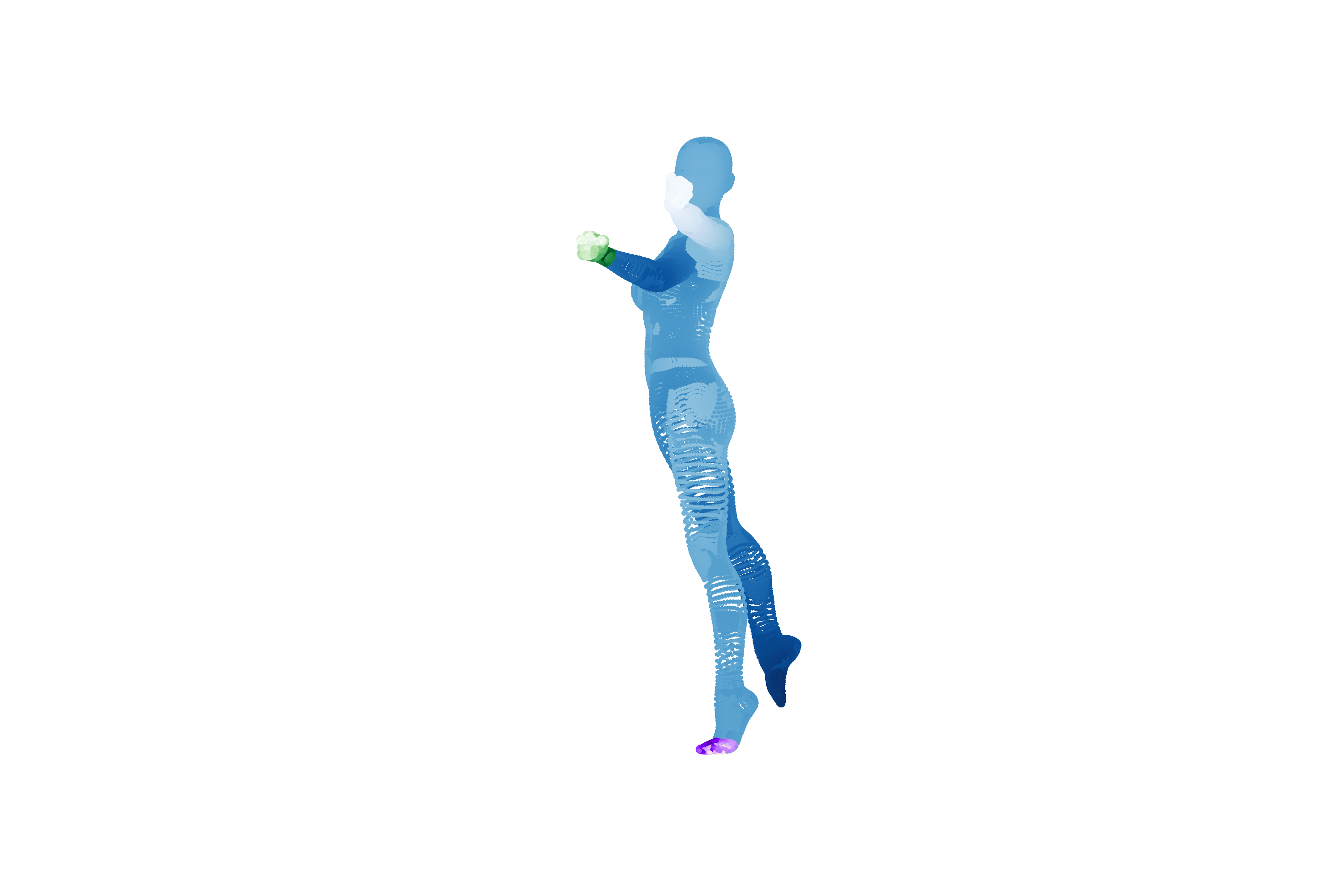} & \includegraphics[width=0.32\textwidth]{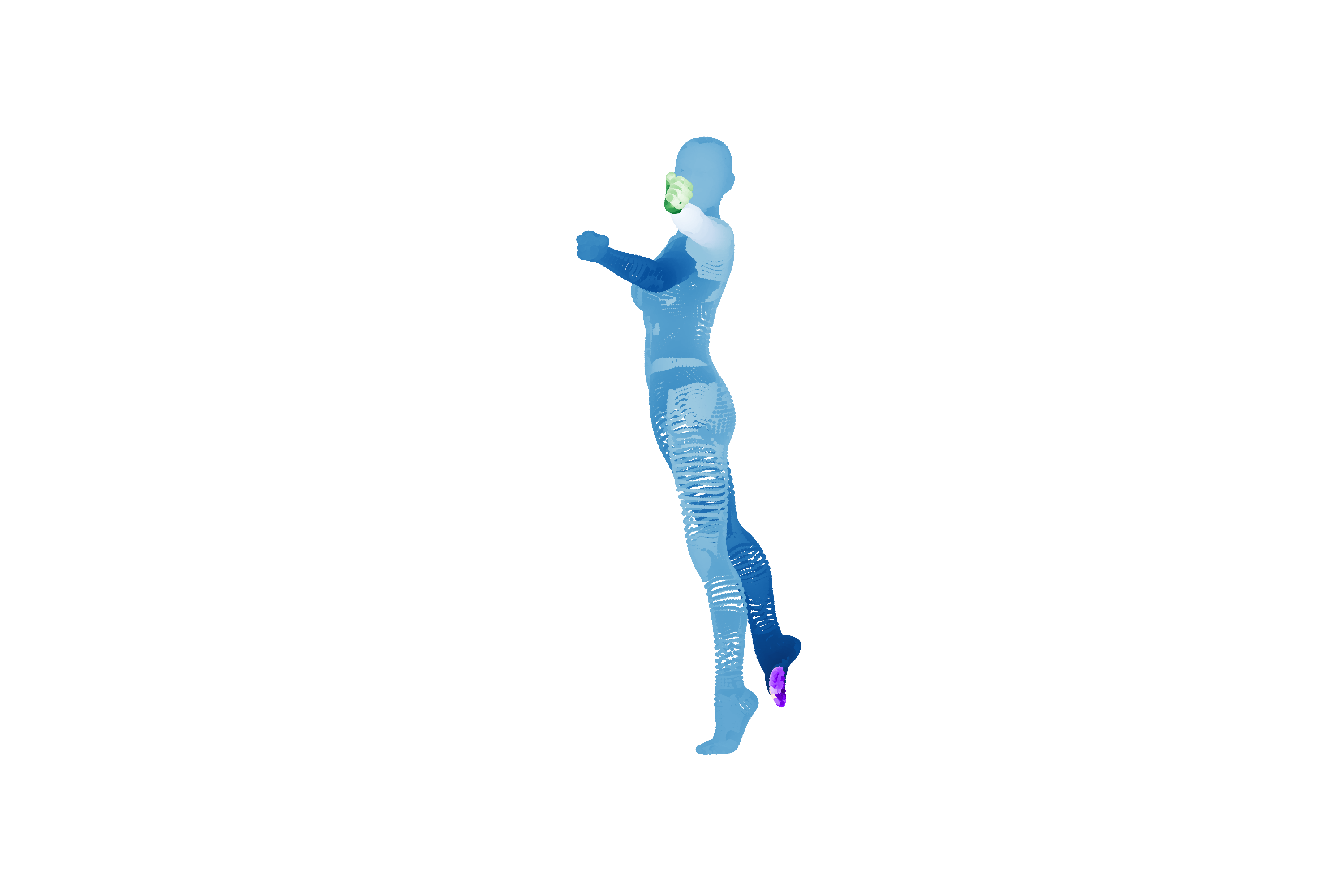} \\
  \includegraphics[width=0.32\textwidth]{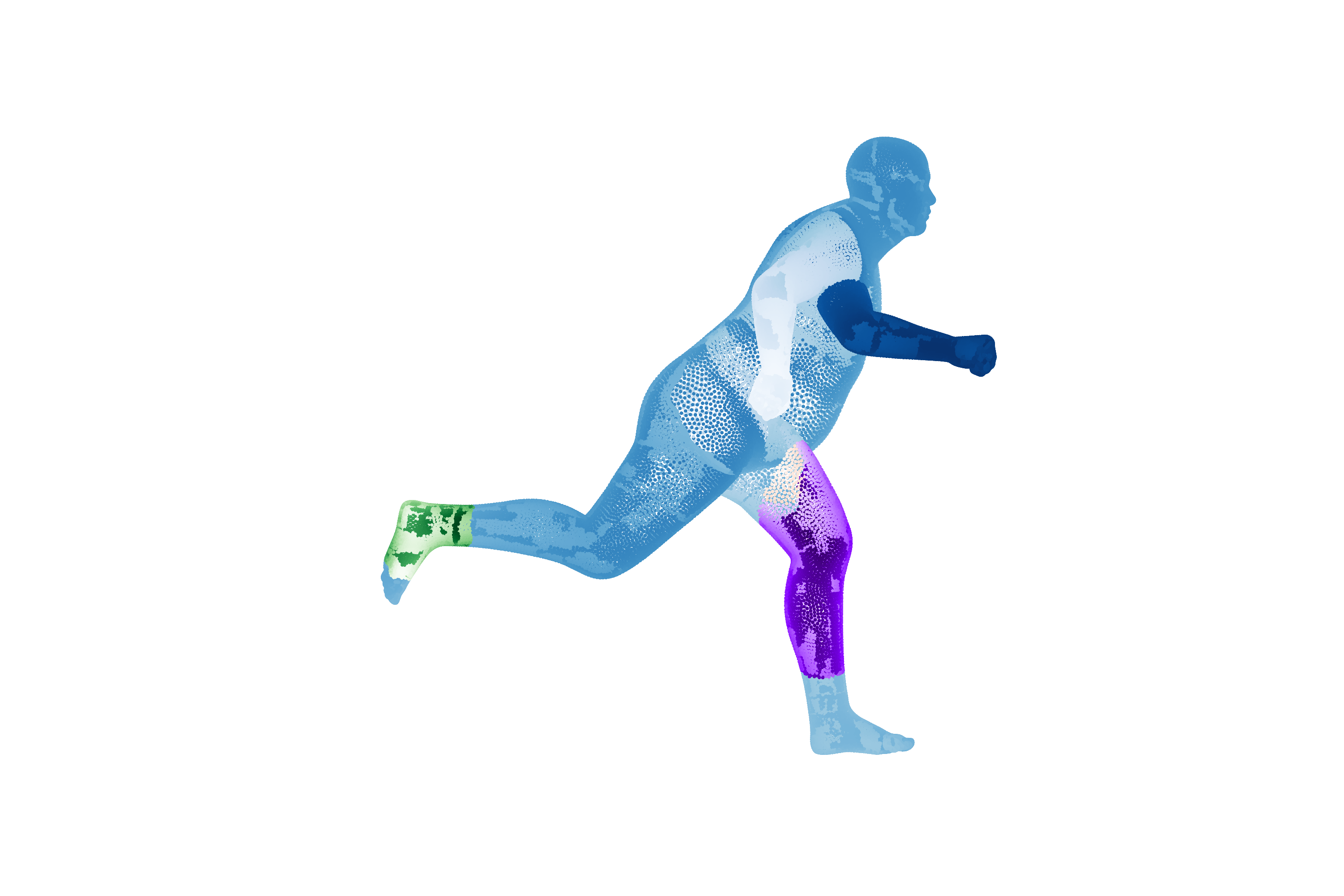} & \includegraphics[width=0.32\textwidth]{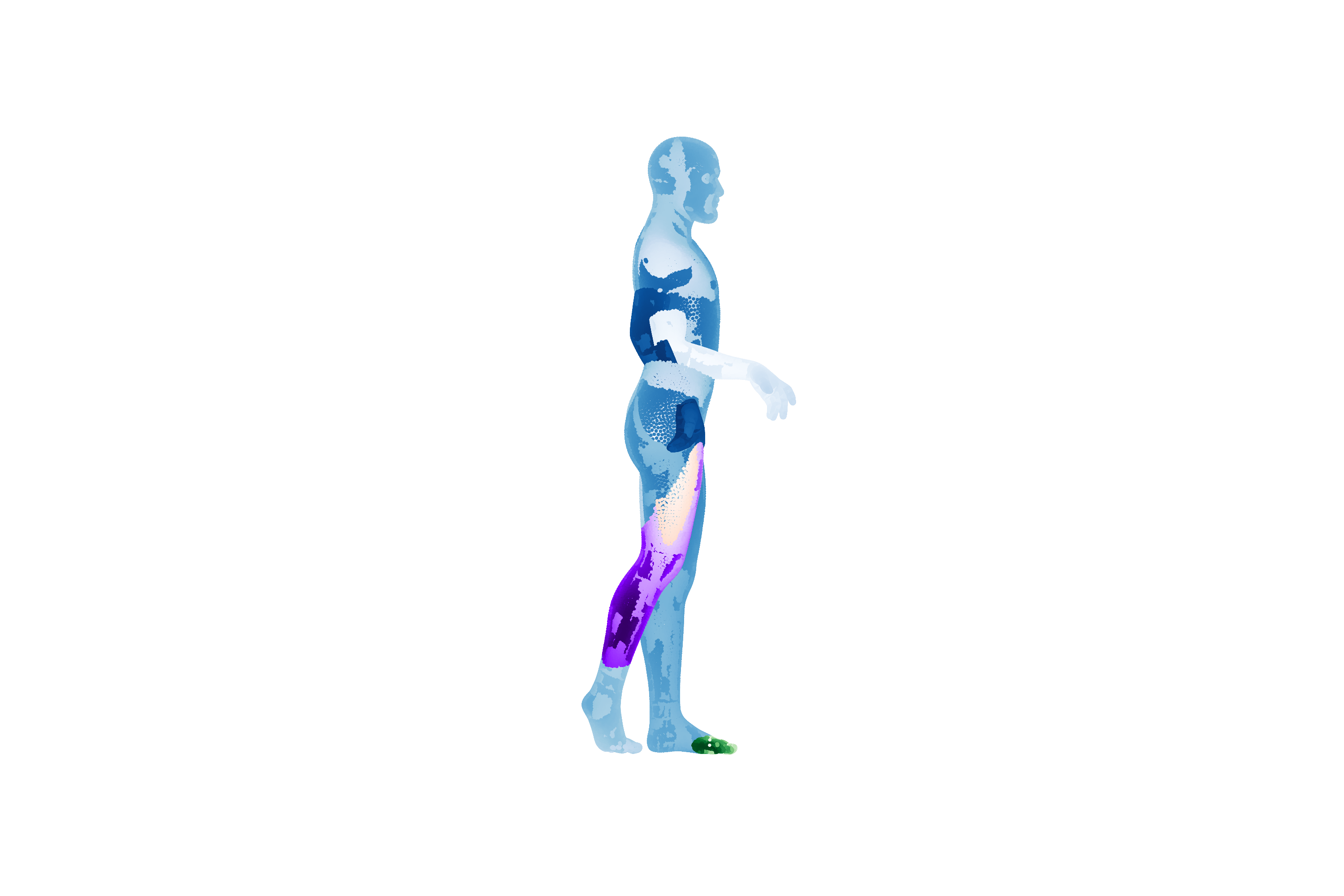} & \includegraphics[width=0.32\textwidth]{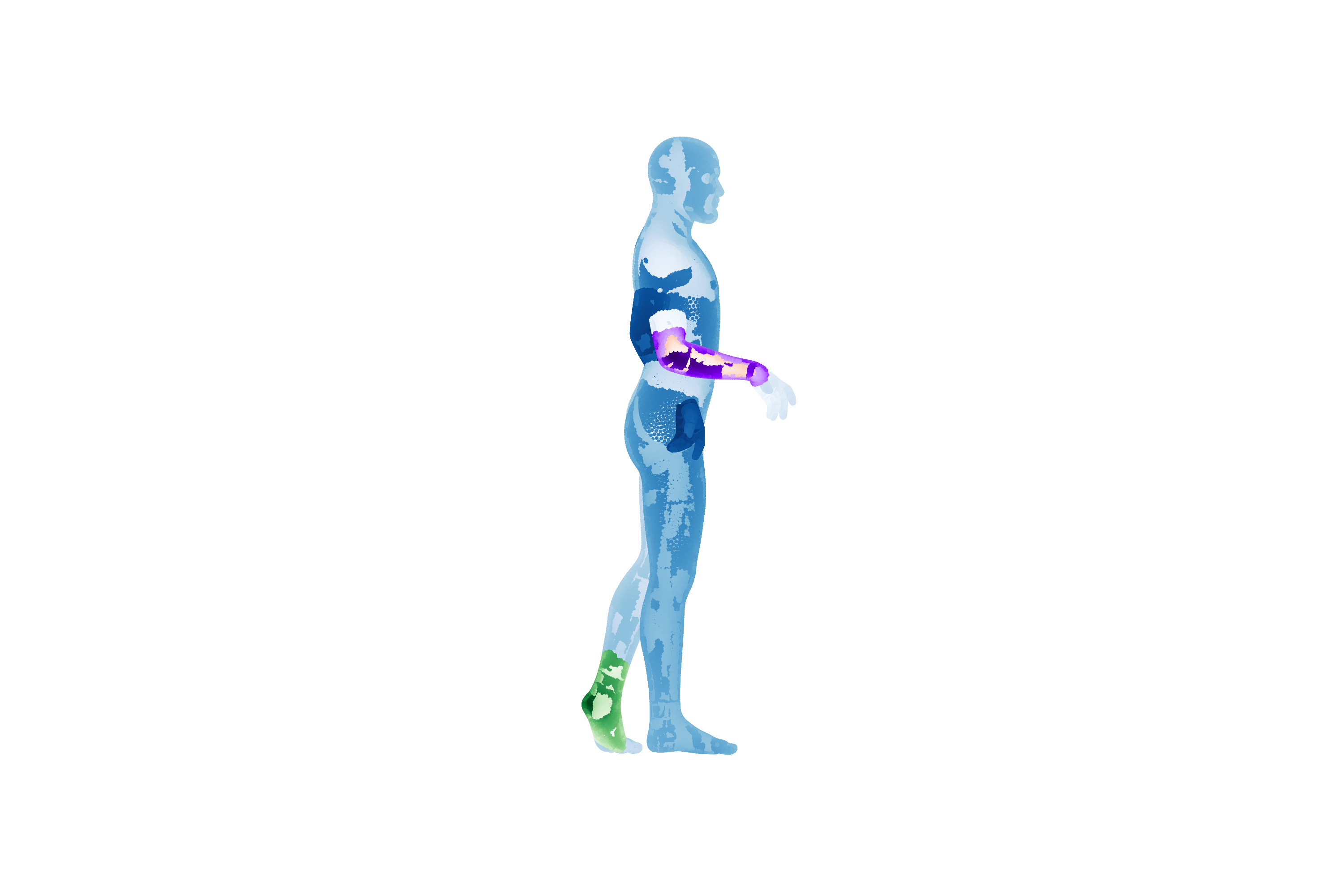}
  \end{tabular}}
  \caption{Shape matching between human-shaped meshes using $ MGW_2 $ (middle) and $ MEW_2 $ (right).
   Each shape on the left column is matched with the shapes on the same row. 
   GMMs with $ 20 $ components have been fitted independently on each shape and the points colored
  in green and purple correspond to Gaussian components that are matched together when solving $ MGW_2 $ or $ MEW_2 $. From left to right and top to bottom, the meshes are composed respectively of $ 84912 $, $ 30300 $, $ 75000 $, $ 273624 $, $ 360678 $, and $ 360357 $ vertices.}
  \label{fig:shrec}
\end{figure}

\subsection{Application to hyperspectral image color transfer}
The goal here is to reproduce the experiment of color transfer 
conducted in \cite{delon2020wasserstein}, but this 
time using a hyperspectral image, i.e an image with more than $ 3 $ color channels. 
More precisely, we aim to create an RGB image from an hyperspectral 
image $ u $ using the colors of another RGB image $ v $.  To do so, 
we consider images as empirical distributions in the color spaces and  
we solve a Gromov-Wasserstein problem between 
the distributions $ \hat{\mu} = \frac{1}{M}\sum_k^M \delta_{u_k} $ 
and $ \hat{\nu} = \frac{1}{N}\sum_l^N \delta_{v_l} $, where $ M $ 
and $ N $ are the number of pixels in respectively the hyperspectral 
image and the RGB image we use as color palette, and $ \{u_k\}_k^M  $
and $ \{v_l\}_l^N $ are the values at each pixel, i.e for here all $ l $, $ v_l \in \rset^3 $ 
and for all $ k $, $ u_k \in \rset^\d $ with $ d > 3 $. We thus 
fit two GMMs $ \mu $ and $ \nu $ on respectively $ \hat{\mu} $ 
and $ \hat{\nu} $ and we use $ MGW_2 $ or $ MEW_2 $ to derive 
a mapping $ T_{\mathrm{mean}} \colon \rset^\d \rightarrow \rset^3 $, as 
described in \Cref{sec:transpomap}. We apply this process 
to a hyperspectral image of $ 512 \times 512 $ pixels with $ 15 $ 
channels that are displayed in \Cref{fig:color_transfer} top left. We use 
as color palettes two paintings by Gauguin and 
Renoir, displayed in \Cref{fig:color_transfer} top right, 
that are respectively \emph{Manhana no atua} (top) and \emph{Le d\'ejeuner des canotiers} (bottom).
These two images are composed of $ 1024 \times 768 $ pixels. The 
resulting images $ T_{\mathrm{mean}}(u) $ are displayed in \Cref{fig:color_transfer} bottom 
(Gauguin at the left and Renoir at the right). For this experiment, 
we observed that setting the number of Gaussian components to $ K = 15 $ 
was a good compromise between capturing the complexity of the color distributions 
and obtaining a relatively regular mapping $ T_{\mathrm{mean}} $. This 
experiment shows that $ MGW_2 $ and $ MEW_2 $ can be used in large 
scale settings: observe indeed that the color distributions $ \hat{\mu}$ and $ \hat{\nu}$ 
are composed respectively of approximatively $ 300000 $ and $ 800000 $ points, which 
makes the problem intractable with entropic-GW solvers such as \cite{peyre2016gromov} or 
\cite{solomon2016entropic}. In term of computation time, the fitting 
of the two GMMs for the hyperspectral image takes aproximatively one minute against 
$ 20 $ seconds for the GMM for the RGB image. The projected gradient descent 
becomes rather slow in that setting, which makes it preferable to 
few updates of $ P $ at each step of \Cref{alg:mew2} for the computation of 
$ MEW_2 $. Finally, for both methods, it takes around $ 2 $ minutes 
to compute the whole RGB image $ T_{\mathrm{mean}}(u) $. 

\vspace{+0.4em}

\begin{figure}[!ht]
\centering 
\begin{minipage}[c]{0.70\textwidth} \vspace{-0.4em}
\includegraphics[width=\textwidth]{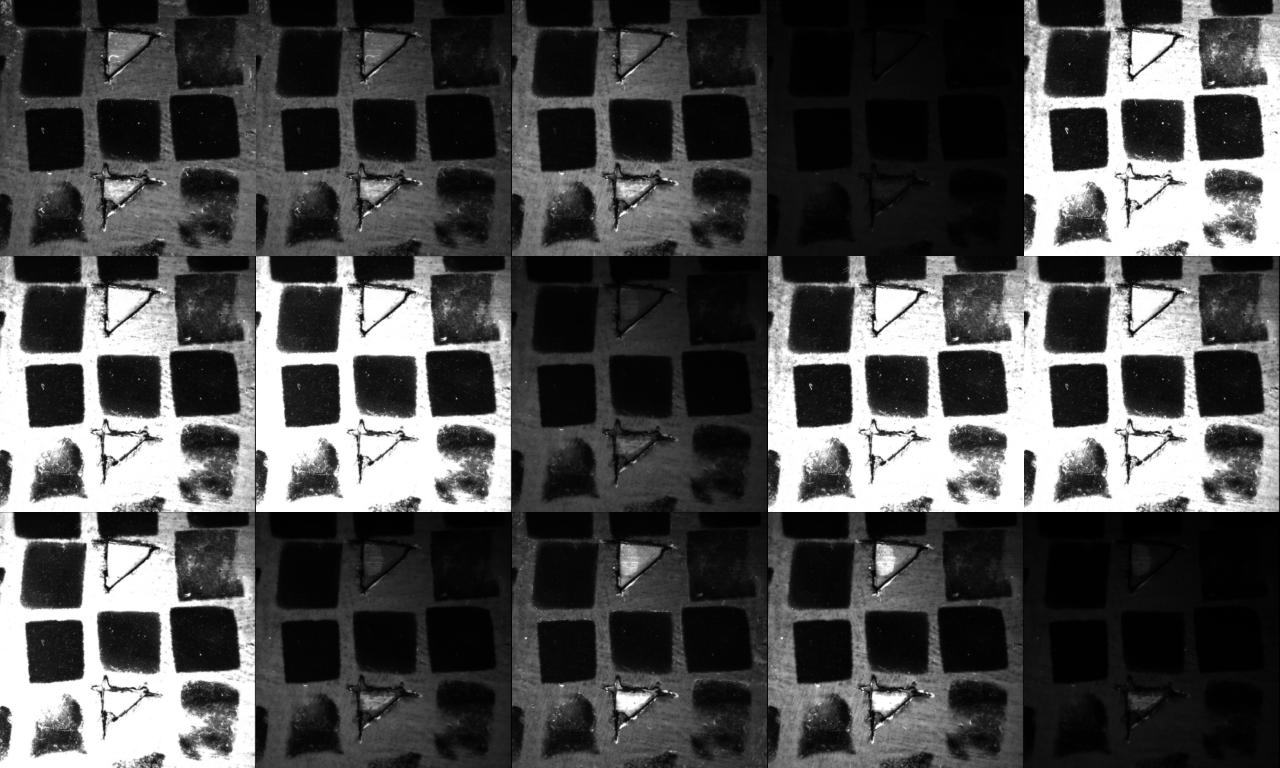}
\end{minipage}
\begin{minipage}[c]{0.28\textwidth}  
\addtolength{\tabbingsep}{-5pt}
\begin{tabular}{ll} \hspace{-0.8em} \vspace{-0.2em}
\includegraphics[width=\textwidth]{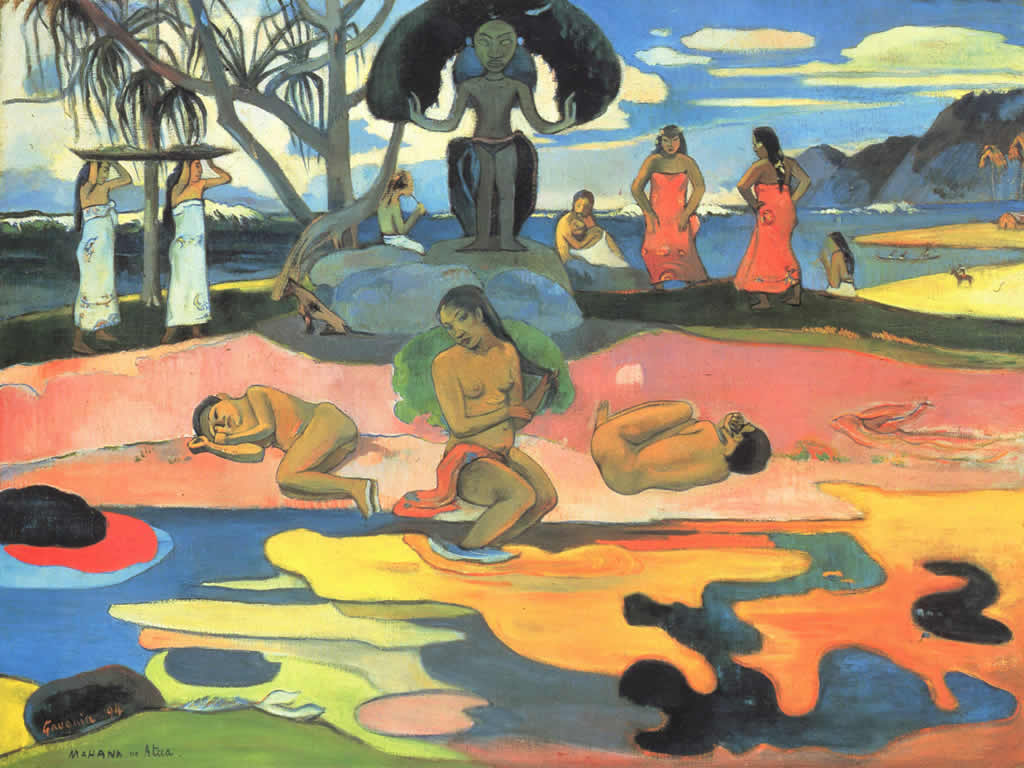}  \\
\hspace{-0.8em} \includegraphics[width=\textwidth]{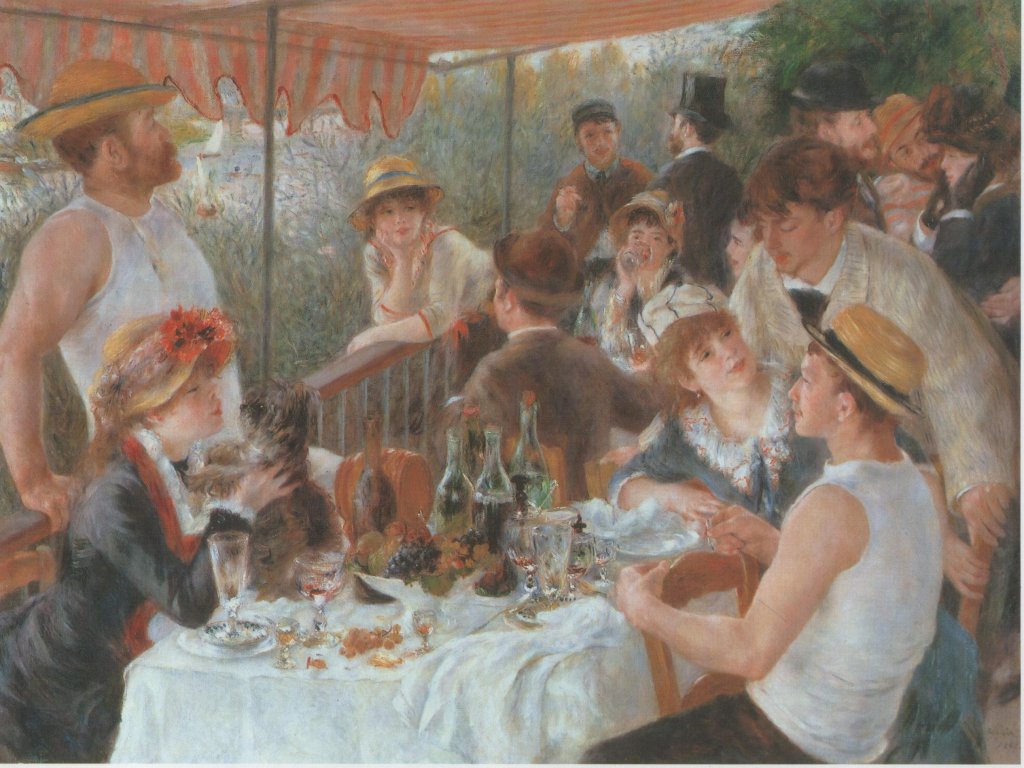} 
\end{tabular} 
\end{minipage}
\addtolength{\tabcolsep}{-3.9pt}    
\begin{tabular}{ccV{3.5}cc} 
  \includegraphics[width=0.24\textwidth]{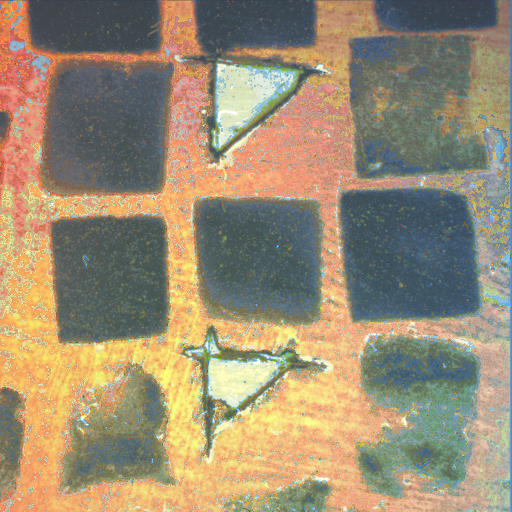} 
& \includegraphics[width=0.24\textwidth]{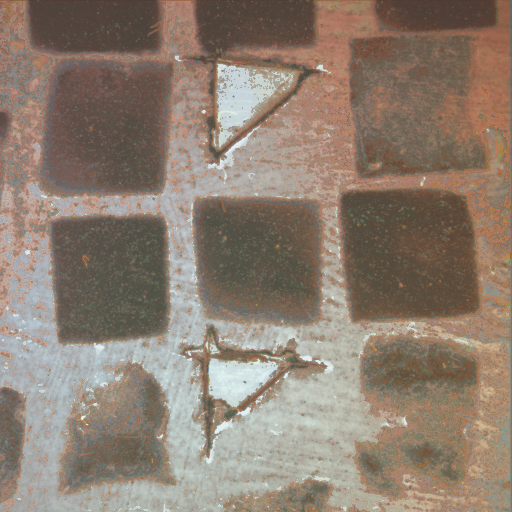} 
& \includegraphics[width=0.24\textwidth]{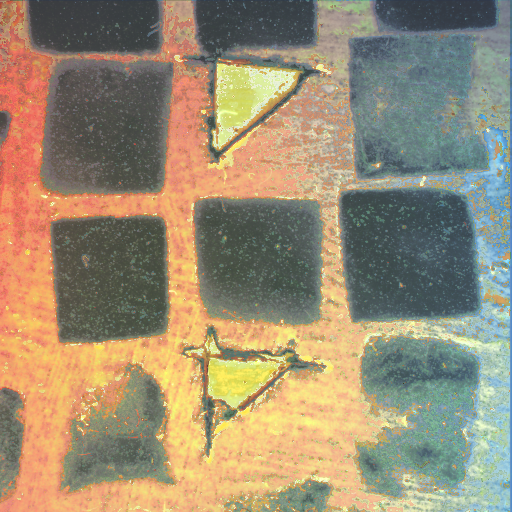} 
& \includegraphics[width=0.24\textwidth]{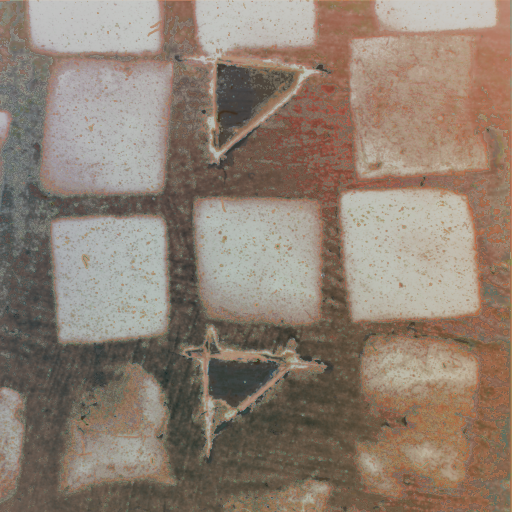} \vspace{-0.2em}
\end{tabular} 
\begin{tabular}{{p{0.4893\textwidth}V{3.5}p{0.5\textwidth}}} 
\centering $ \boldsymbol{MGW_2} $ & \centering $ \boldsymbol{MEW_2} $ 
\end{tabular} 
\addtolength{\tabcolsep}{+3.9pt}   
\addtolength{\tabbingsep}{+5pt}
\caption{Color transfers between a hyperspectral image with $ 15 $ channels 
(top left) and two paintings by Gauguin and Renoir (top right, middle right). 
Bottom line: the obtained RGB images using $ MGW_2$ and $ MEW_2$. For this 
experiment, we used GMMs with $ 15 $ components. Image taken by Francesca Ramacciotti (Alma Mater Studiorum - University
of Bologna) and Laure Cazals (supported by the European Commission in
the framework of the GoGreen project (GA no. 101060768)).}\label{fig:color_transfer}
\end{figure}

\section{Conclusion and perspectives}
In this paper, we have introduced two new OT distances on the set of 
Gaussian mixture models, $ MGW_2 $ and $ MEW_2 $,  
and we have shown that they both can be used 
to solve relatively efficiently Gromov-Wasserstein related problems on Euclidean spaces, especially in moderate-to-large scale settings involving several tens of thousands of points. These OT distances are also by design particularly suited to settings 
where there already exists a kind of clustering structure in the data. This being said, if $ MEW_2 $ remains an efficient alternative to the entropic GW 
solvers proposed by \cite{peyre2016gromov} and \cite{solomon2016entropic}, 
we observed that the method was actually slower and perhaps harder to tune 
than $ MGW_2 $ for a slighty lower quality of results,  
and so we believe that $ MGW_2 $ is a better choice in practice. This latter distance is part of the family of Gromov-Wasserstein type OT distances 
that reduce the size of the GW problem, which also includes notably 
qGW \cite[]{chowdhury2021quantized}, MREC \cite[]{blumberg2020mrec}
and scalable GW \cite[]{xu2019scalable}. To the best of our knowledge, no such method specifically based on Gaussian mixture clustering had already been proposed in the literature. Furthermore, our method differs from these three other approaches 
in the fact that we only need here to solve numerically 
one single GW problem at the scale of the clusters, using not only the centroid position information but also order $ 2 $ statistics.  Note in particular that our method has strong similarities with qGW, which uses Voronoi quantizations of the spaces instead of GMMs
However, an important difference lies in the fact that the $MGW_2 $ method primarily provides a distance between clouds of points before providing a heuristic coupling between them, while the qGW method directly derives a heuristic coupling which can then be reinjected into the objective function to derive a distance. For both methods, the optional additional step of computing a coupling for $MGW_2 $ and computing a distance for qGW increases the computation time. Consequently, $MGW_2 $ seems to be a more appropriate choice than qGW for tasks that require only a distance between clouds of points, see the computation times of \Cref{fig:comp}, while qGW seems to be a more appropriate choice for tasks that require only a coupling between points. Still, we have shown in our experiments that $MGW_2 $ can also provide couplings, with equivalent performance to qGW in terms of accuracy, although it is significantly slower in the small-to-medium scale setting of \Cref{tab:disto}. However, we believe that one advantage of $MGW_2 $ over qGW for deriving couplings is that $MGW_2 $ appears to better decorrelate the number of clusters needed to achieve good accuracy from the number of points, which may become an important feature in larger scale settings.

\paragraph{Perspectives for future work}
$ MGW_2 $ could be easily extended to other type of mixtures as soon as we have an identifiability property between the mixtures and the probability distributions on the space of the distributions 
that compose the mixtures. If in the Euclidean setting GMMs seem to be versatile
enough to represent large classes of concrete and applied problems, an interesting 
extension of our work could be to consider mixture of distributions on non-Euclidean 
spaces. 

Computationally speaking, the main bottleneck of the method probably comes from 
the fitting of the GMMs with the Expectation-Maximization (EM) algorithm \cite[]{dempster1977maximum}
which can become relatively costly in large scale settings or as soon as the dimension increases. If the EM 
algorithm remains invariably the classical algorithm for learning GMMs, some recent 
approaches \cite[]{hosseini2020alternative,sembach2022riemannian,pasande2022stochastic}
have proposed alternative algorithms that seem to outperform it. These approaches
are based on Riemannian stochastic optimization, leveraging the rich Riemannian structure
of the set of positive definite matrices. Another interesting alternative  
that has been shown to outperform the EM algorithm has been 
proposed by \cite[]{kolouri2018sliced2} and is based on the minimization  
of the sliced-Wasserstein distance. Integrating this in our method
could result thus in an approach fully-based on optimal transport. 

Another possible limitation of our work lies in the fact that the $ MGW_2 $ solver 
converges sometimes to sub-optimal local minima. If the annealed procedure 
introduced in \Cref{sec:optiplanMGW2} seems to reduce this issue, we generally have no 
guarantee that the solution we have converged to is optimal. This is not 
specific to our method and comes from the gradient descent structure of 
the classic GW solvers. Still, when 
solving the GW problem between GMMs rather than solving it directly between 
the points, it is likely that we increase the probability of converging towards a 
sub-optimal local minimum because we inevitably introduce symmetries by simplifying the problem and so we probably increase in the mean time the number of local minima in the GW objective. In the Euclidean setting, the recent work of \cite{ryner2023globally} proposes 
an algorithm for solving the GW problem that is guaranteed to converge 
toward a global minimum, leveraging the low-rank structure of the cost matrices
when the cost functions are the squared Euclidean distances. A future perspective 
of work could be therefore to study if a similar idea could be applied for solving 
the $ MGW_2 $ problem.

\subsection*{Acknowledgements}

This research was funded, in part, by the Agence nationale de la recherche (ANR), through the SOCOT project (ANR-23-CE40-0017), and the PEPR PDE-AI project (ANR-23-PEIA-0004).

\bibliographystyle{apalike} 
\bibliography{article}

\begin{thebibliography}{}

\bibitem[Altschuler et~al., 2019]{altschuler2019massively}
Altschuler, J., Bach, F., Rudi, A., and Niles-Weed, J. (2019).
\newblock Massively scalable {S}inkhorn distances via the {N}ystr{\"o}m method.
\newblock {\em Advances in neural information processing systems}, 32.

\bibitem[Altschuler et~al., 2018]{altschuler2018approximating}
Altschuler, J., Bach, F., Rudi, A., and Weed, J. (2018).
\newblock Approximating the quadratic transportation metric in near-linear
  time.
\newblock {\em arXiv preprint arXiv:1810.10046}.

\bibitem[Alvarez-Melis and Jaakkola, 2018]{alvarez2018gromov}
Alvarez-Melis, D. and Jaakkola, T. (2018).
\newblock Gromov--{W}asserstein alignment of word embedding spaces.
\newblock In {\em Proceedings of the 2018 Conference on Empirical Methods in
  Natural Language Processing}, pages 1881--1890.

\bibitem[Alvarez-Melis et~al., 2019]{alvarez2019towards}
Alvarez-Melis, D., Jegelka, S., and Jaakkola, T.~S. (2019).
\newblock Towards optimal transport with global invariances.
\newblock In {\em The 22nd International Conference on Artificial Intelligence
  and Statistics}, pages 1870--1879. PMLR.

\bibitem[Arjovsky et~al., 2017]{arjovsky2017wasserstein}
Arjovsky, M., Chintala, S., and Bottou, L. (2017).
\newblock Wasserstein generative adversarial networks.
\newblock In {\em International Conference on Machine Learning}, pages
  214--223. PMLR.

\bibitem[Baker, 1971]{baker1971isometries}
Baker, J. (1971).
\newblock Isometries in normed spaces.
\newblock {\em The American Mathematical Monthly}, 78(6):655--658.

\bibitem[Baydin et~al., 2018]{baydin2018automatic}
Baydin, A.~G., Pearlmutter, B.~A., Radul, A.~A., and Siskind, J.~M. (2018).
\newblock Automatic differentiation in machine learning: {a} survey.
\newblock {\em Journal of Marchine Learning Research}, 18:1--43.

\bibitem[Blumberg et~al., 2020]{blumberg2020mrec}
Blumberg, A.~J., Carriere, M., Mandell, M.~A., Rabadan, R., and Villar, S.
  (2020).
\newblock {MREC}: {a} fast and versatile framework for aligning and matching
  point clouds with applications to single cell molecular data.
\newblock {\em stat}, 1050:20.

\bibitem[Bolley, 2008]{bolley2008separability}
Bolley, F. (2008).
\newblock Separability and completeness for the {W}asserstein distance.
\newblock {\em Lecture Notes in Mathematics-Springer-Verlag-}, 1934:371.

\bibitem[Borg and Groenen, 2005]{borg2005modern}
Borg, I. and Groenen, P.~J. (2005).
\newblock {\em Modern multidimensional scaling: theory and applications}.
\newblock Springer Science \& Business Media.

\bibitem[Boumal, 2023]{boumal2023introduction}
Boumal, N. (2023).
\newblock {\em An introduction to optimization on smooth manifolds}.
\newblock Cambridge University Press.

\bibitem[Brogat-Motte et~al., 2022]{brogat2022learning}
Brogat-Motte, L., Flamary, R., Brouard, C., Rousu, J., and d’Alch{\'e} Buc,
  F. (2022).
\newblock Learning to predict graphs with fused {G}romov--{W}asserstein
  barycenters.
\newblock In {\em International Conference on Machine Learning}, pages
  2321--2335. PMLR.

\bibitem[Bunne et~al., 2019]{bunne2019learning}
Bunne, C., Alvarez-Melis, D., Krause, A., and Jegelka, S. (2019).
\newblock Learning generative models across incomparable spaces.
\newblock In {\em International conference on machine learning}, pages
  851--861. PMLR.

\bibitem[Cai and Lim, 2022]{cai2022distances}
Cai, Y. and Lim, L.-H. (2022).
\newblock Distances between probability distributions of different dimensions.
\newblock {\em IEEE Transactions on Information Theory}, 68(6):4020--4031.

\bibitem[Calamai and Mor{\'e}, 1987]{calamai1987projected}
Calamai, P.~H. and Mor{\'e}, J.~J. (1987).
\newblock Projected gradient methods for linearly constrained problems.
\newblock {\em Mathematical programming}, 39(1):93--116.

\bibitem[Chen et~al., 2018]{chen2018optimal}
Chen, Y., Georgiou, T.~T., and Tannenbaum, A. (2018).
\newblock Optimal transport for {G}aussian mixture models.
\newblock {\em IEEE Access}, 7:6269--6278.

\bibitem[Chowdhury and M{\'e}moli, 2019]{chowdhury2019gromov}
Chowdhury, S. and M{\'e}moli, F. (2019).
\newblock The {G}romov--{W}asserstein distance between networks and stable
  network invariants.
\newblock {\em Information and Inference: A Journal of the IMA}, 8(4):757--787.

\bibitem[Chowdhury et~al., 2021]{chowdhury2021quantized}
Chowdhury, S., Miller, D., and Needham, T. (2021).
\newblock Quantized {G}romov--{W}asserstein.
\newblock In {\em Machine Learning and Knowledge Discovery in Databases.
  Research Track: European Conference, ECML PKDD 2021, Bilbao, Spain, September
  13--17, 2021, Proceedings, Part III 21}, pages 811--827. Springer.

\bibitem[Cohen and Guibasm, 1999]{cohen1999earth}
Cohen, S. and Guibasm, L. (1999).
\newblock The {E}arth mover's distance under transformation sets.
\newblock In {\em Proceedings of the Seventh IEEE International Conference on
  Computer Vision}, volume~2, pages 1076--1083. IEEE.

\bibitem[Courant, 1920]{courant1920eigenwerte}
Courant, R. (1920).
\newblock {\"U}ber die eigenwerte bei den differentialgleichungen der
  mathematischen physik.
\newblock {\em Mathematische Zeitschrift}, 7(1-4):1--57.

\bibitem[Courty et~al., 2018]{courty2018learning}
Courty, N., Flamary, R., and Ducoffe, M. (2018).
\newblock Learning {W}asserstein embeddings.
\newblock In {\em ICLR 2018-6th International Conference on Learning
  Representations}, pages 1--13.

\bibitem[Courty et~al., 2016]{courty2016optimal}
Courty, N., Flamary, R., Tuia, D., and Rakotomamonjy, A. (2016).
\newblock Optimal transport for domain adaptation.
\newblock In {\em Transactions on Pattern Analysis and Machine Intelligence},
  volume~39, pages 1853--1865. IEEE.

\bibitem[Cuturi, 2013]{cuturi2013sinkhorn}
Cuturi, M. (2013).
\newblock Sinkhorn distances: lightspeed computation of optimal transport.
\newblock {\em Advances in neural information processing systems}, 26.

\bibitem[Delon and Desolneux, 2020]{delon2020wasserstein}
Delon, J. and Desolneux, A. (2020).
\newblock A {W}asserstein-type distance in the space of {G}aussian mixture
  models.
\newblock {\em SIAM Journal on Imaging Sciences}, 13(2):936--970.

\bibitem[Delon et~al., 2022a]{salmona2021gromov}
Delon, J., Desolneux, A., and Salmona, A. (2022a).
\newblock Gromov--wasserstein distances between gaussian distributions.
\newblock {\em Journal of Applied Probability}, 59(4):1178--1198.

\bibitem[Delon et~al., 2022b]{delon2021generalized}
Delon, J., Gozlan, N., and Saint-Dizier, A. (2022b).
\newblock Generalized {W}asserstein barycenters between probability measures
  living on different subspaces.
\newblock {\em Annals of Applied Probability}.

\bibitem[Dempster et~al., 1977]{dempster1977maximum}
Dempster, A.~P., Laird, N.~M., and Rubin, D.~B. (1977).
\newblock Maximum likelihood from incomplete data via the {EM} algorithm.
\newblock {\em Journal of the royal statistical society: series B
  (methodological)}, 39(1):1--22.

\bibitem[Fatras et~al., 2021]{fatras2021minibatch}
Fatras, K., Zine, Y., Majewski, S., Flamary, R., Gribonval, R., and Courty, N.
  (2021).
\newblock Minibatch optimal transport distances; analysis and applications.
\newblock {\em arXiv e-prints}, pages arXiv--2101.

\bibitem[Fischer, 1905]{fischer1905quadratische}
Fischer, E. (1905).
\newblock {\"U}ber quadratische formen mit reellen koeffizienten.
\newblock {\em Monatshefte f{\"u}r Mathematik und Physik}, 16:234--249.

\bibitem[Fix and Hodges, 1951]{fix1951discriminatory}
Fix, E. and Hodges, J. (1951).
\newblock Discriminatory analysis: nonparametric discrimination: consistency
  properties. report. 4.
\newblock {\em T. USAF School of Aviation Medicine}.

\bibitem[Flamary et~al., 2021]{flamary2021pot}
Flamary, R., Courty, N., Gramfort, A., Alaya, M.~Z., Boisbunon, A., Chambon,
  S., Chapel, L., Corenflos, A., Fatras, K., Fournier, N., et~al. (2021).
\newblock {POT}: Python {O}ptimal {T}ransport.
\newblock {\em The Journal of Machine Learning Research}, 22(1):3571--3578.

\bibitem[Forbes et~al., 2021]{forbes2021approximate}
Forbes, F., Nguyen, H.~D., Nguyen, T.~T., and Arbel, J. (2021).
\newblock Approximate bayesian computation with surrogate posteriors.
\newblock {\em Preprint hal-03139256.(Cited on pages 5, 65, 69, 71, 101, 103,
  251, 265, and 273.)}.

\bibitem[Forrow et~al., 2019]{forrow2019statistical}
Forrow, A., H{\"u}tter, J.-C., Nitzan, M., Rigollet, P., Schiebinger, G., and
  Weed, J. (2019).
\newblock Statistical optimal transport via factored couplings.
\newblock In {\em The 22nd International Conference on Artificial Intelligence
  and Statistics}, pages 2454--2465. PMLR.

\bibitem[Genevay et~al., 2018]{genevay2017learning}
Genevay, A., Peyre, G., and Cuturi, M. (2018).
\newblock Learning generative models with {S}inkhorn divergences.
\newblock In {\em International Conference on Artificial Intelligence and
  Statistics}, volume~84, pages 1608--1617. PMLR.

\bibitem[Givens et~al., 1984]{OTGaussian}
Givens, C.~R., Shortt, R.~M., et~al. (1984).
\newblock A class of {W}asserstein metrics for probability distributions.
\newblock In {\em Michigan Mathematical Journal}, volume~31, pages 231--240.
  the University of Michigan.

\bibitem[Hosseini and Sra, 2020]{hosseini2020alternative}
Hosseini, R. and Sra, S. (2020).
\newblock An alternative to {EM} for {G}aussian mixture models: batch and
  stochastic {R}iemannian optimization.
\newblock {\em Mathematical programming}, 181(1):187--223.

\bibitem[James, 1976]{james1976topology}
James, I.~M. (1976).
\newblock {\em The topology of {S}tiefel manifolds}, volume~24.
\newblock Cambridge University Press.

\bibitem[Kloeckner, 2010]{kloeckner2010geometric}
Kloeckner, B. (2010).
\newblock A geometric study of {W}asserstein spaces: Euclidean spaces.
\newblock {\em Annali della Scuola Normale Superiore di Pisa-Classe di
  Scienze}, 9(2):297--323.

\bibitem[Kolouri et~al., 2019]{kolouri2019generalized}
Kolouri, S., Nadjahi, K., Simsekli, U., Badeau, R., and Rohde, G. (2019).
\newblock Generalized sliced {W}asserstein distances.
\newblock {\em Advances in neural information processing systems}, 32.

\bibitem[Kolouri et~al., 2018]{kolouri2018sliced2}
Kolouri, S., Rohde, G.~K., and Hoffmann, H. (2018).
\newblock Sliced {W}asserstein distance for learning {G}aussian mixture models.
\newblock In {\em Proceedings of the IEEE Conference on Computer Vision and
  Pattern Recognition}, pages 3427--3436.

\bibitem[Lambert et~al., 2022]{lambert2022variational}
Lambert, M., Chewi, S., Bach, F., Bonnabel, S., and Rigollet, P. (2022).
\newblock Variational inference via {W}asserstein gradient flows.
\newblock In {\em Advances in Neural Information Processing Systems}.

\bibitem[Leclaire et~al., 2023]{leclaire2022optimal}
Leclaire, A., Delon, J., and Desolneux, A. (2023).
\newblock Optimal transport between {GMMs} for texture synthesis.
\newblock In {\em Scale Space and Variational Methods in Computer Vision: 9th
  International Conference, SSVM 2023}.

\bibitem[Luzi et~al., 2023]{luzi2023evaluating}
Luzi, L., Marrero, C.~O., Wynar, N., Baraniuk, R.~G., and Henry, M.~J. (2023).
\newblock Evaluating generative networks using {G}aussian mixtures of image
  features.
\newblock In {\em Proceedings of the IEEE/CVF Winter Conference on Applications
  of Computer Vision}, pages 279--288.

\bibitem[Magnus and Neudecker, 2019]{magnus2019matrix}
Magnus, J.~R. and Neudecker, H. (2019).
\newblock {\em Matrix differential calculus with applications in statistics and
  econometrics}.
\newblock John Wiley \& Sons.

\bibitem[Mazur and Ulam, 1932]{mazur1932transformations}
Mazur, S. and Ulam, S. (1932).
\newblock Sur les transformations isom{\'e}triques d’espaces vectoriels
  norm{\'e}s.
\newblock {\em CR Acad. Sci. Paris}, 194(946-948):116.

\bibitem[Melzi et~al., 2019]{melzi2019shrec}
Melzi, S., Marin, R., Rodol{\`a}, E., Castellani, U., Ren, J., Poulenard, A.,
  Wonka, P., and Ovsjanikov, M. (2019).
\newblock {SHREC} 2019: {m}atching humans with different connectivity.
\newblock In {\em Eurographics Workshop on 3D Object Retrieval}, volume~7,
  page~3. The Eurographics Association.

\bibitem[M{\'e}moli, 2009]{memoli2009spectral}
M{\'e}moli, F. (2009).
\newblock Spectral {G}romov--{W}asserstein distances for shape matching.
\newblock In {\em 2009 IEEE 12th International Conference on Computer Vision
  Workshops, ICCV Workshops}, pages 256--263. IEEE.

\bibitem[M{\'e}moli, 2011]{memoli}
M{\'e}moli, F. (2011).
\newblock Gromov--{W}asserstein distances and the metric approach to object
  matching.
\newblock In {\em Foundations of Computational Mathematics}, volume~11, pages
  417--487. Springer.

\bibitem[Papadakis, 2014]{papadakis2014canonically}
Papadakis, P. (2014).
\newblock The canonically posed 3d objects dataset.
\newblock In {\em Eurographics Workshop on 3D Object Retrieval}, pages 33--36.

\bibitem[Pasande et~al., 2022]{pasande2022stochastic}
Pasande, M., Hosseini, R., and Araabi, B.~N. (2022).
\newblock Stochastic first-order learning for large-scale flexibly tied
  {G}aussian mixture model.
\newblock {\em arXiv preprint arXiv:2212.05402}.

\bibitem[Pedregosa et~al., 2011]{pedregosa2011scikit}
Pedregosa, F., Varoquaux, G., Gramfort, A., Michel, V., Thirion, B., Grisel,
  O., Blondel, M., Prettenhofer, P., Weiss, R., Dubourg, V., et~al. (2011).
\newblock Scikit-learn: {m}achine learning in python.
\newblock {\em the Journal of machine Learning research}, 12:2825--2830.

\bibitem[Pele and Taskar, 2013]{pele2013tangent}
Pele, O. and Taskar, B. (2013).
\newblock The tangent {E}arth mover’s distance.
\newblock In {\em Geometric Science of Information: First International
  Conference, GSI 2013, Paris, France, August 28-30, 2013. Proceedings}, pages
  397--404. Springer.

\bibitem[Petersen et~al., 2008]{petersen2008matrix}
Petersen, K.~B., Pedersen, M.~S., et~al. (2008).
\newblock The matrix cookbook.
\newblock {\em Technical University of Denmark}, 7(15):510.

\bibitem[Peyr{\'e} et~al., 2016]{peyre2016gromov}
Peyr{\'e}, G., Cuturi, M., and Solomon, J. (2016).
\newblock Gromov--{W}asserstein averaging of kernel and distance matrices.
\newblock In {\em International conference on machine learning}, pages
  2664--2672. PMLR.

\bibitem[Rabin et~al., 2014]{rabin2014adaptive}
Rabin, J., Ferradans, S., and Papadakis, N. (2014).
\newblock Adaptive color transfer with relaxed optimal transport.
\newblock In {\em 2014 IEEE international conference on image processing
  (ICIP)}, pages 4852--4856. IEEE.

\bibitem[Rabin et~al., 2012]{rabin2012wasserstein}
Rabin, J., Peyr{\'e}, G., Delon, J., and Bernot, M. (2012).
\newblock Wasserstein barycenter and its application to texture mixing.
\newblock In {\em Scale Space and Variational Methods in Computer Vision: Third
  International Conference, SSVM 2011, Ein-Gedi, Israel, May 29--June 2, 2011,
  Revised Selected Papers 3}, pages 435--446. Springer.

\bibitem[Rustamov et~al., 2013]{rustamov2013map}
Rustamov, R.~M., Ovsjanikov, M., Azencot, O., Ben-Chen, M., Chazal, F., and
  Guibas, L. (2013).
\newblock Map-based exploration of intrinsic shape differences and variability.
\newblock {\em ACM Transactions on Graphics (TOG)}, 32(4):1--12.

\bibitem[Ryner et~al., 2023]{ryner2023globally}
Ryner, M., Kronqvist, J., and Karlsson, J. (2023).
\newblock Globally solving the {G}romov--{W}asserstein problem for point clouds
  in low dimensional euclidean spaces.
\newblock {\em arXiv preprint arXiv:2307.09057}.

\bibitem[Scetbon and Cuturi, 2020]{scetbon2020linear}
Scetbon, M. and Cuturi, M. (2020).
\newblock Linear time {S}inkhorn divergences using positive features.
\newblock {\em Advances in Neural Information Processing Systems},
  33:13468--13480.

\bibitem[Scetbon et~al., 2021]{scetbon2021low}
Scetbon, M., Cuturi, M., and Peyr{\'e}, G. (2021).
\newblock Low-rank {S}inkhorn factorization.
\newblock In {\em International Conference on Machine Learning}, pages
  9344--9354. PMLR.

\bibitem[Scetbon et~al., 2022]{scetbon2022linear}
Scetbon, M., Peyr{\'e}, G., and Cuturi, M. (2022).
\newblock Linear-time {G}romov--{W}asserstein distances using low rank
  couplings and costs.
\newblock In {\em International Conference on Machine Learning}, pages
  19347--19365. PMLR.

\bibitem[Seguy et~al., 2017]{seguy2017large}
Seguy, V., Damodaran, B.~B., Flamary, R., Courty, N., Rolet, A., and Blondel,
  M. (2017).
\newblock Large-scale optimal transport and mapping estimation.
\newblock {\em arXiv preprint arXiv:1711.02283}.

\bibitem[Sembach et~al., 2022]{sembach2022riemannian}
Sembach, L., Burgard, J.~P., and Schulz, V. (2022).
\newblock A {R}iemannian {N}ewton trust-region method for fitting {G}aussian
  mixture models.
\newblock {\em Statistics and Computing}, 32(1):8.

\bibitem[Sinkhorn and Knopp, 1967]{sinkhorn1967concerning}
Sinkhorn, R. and Knopp, P. (1967).
\newblock Concerning nonnegative matrices and doubly stochastic matrices.
\newblock {\em Pacific Journal of Mathematics}, 21(2):343--348.

\bibitem[Solomon et~al., 2015]{solomon2015convolutional}
Solomon, J., De~Goes, F., Peyr{\'e}, G., Cuturi, M., Butscher, A., Nguyen, A.,
  Du, T., and Guibas, L. (2015).
\newblock Convolutional {W}asserstein distances: {e}fficient optimal
  transportation on geometric domains.
\newblock {\em ACM Transactions on Graphics (ToG)}, 34(4):1--11.

\bibitem[Solomon et~al., 2016]{solomon2016entropic}
Solomon, J., Peyr{\'e}, G., Kim, V.~G., and Sra, S. (2016).
\newblock Entropic metric alignment for correspondence problems.
\newblock {\em ACM Transactions on Graphics (ToG)}, 35(4):1--13.

\bibitem[Sturm, 2006]{sturm2006geometry}
Sturm, K.-T. (2006).
\newblock On the geometry of metric measure spaces. i.
\newblock {\em Acta Math}, 196:65--131.

\bibitem[Sturm, 2012]{sturm2012space}
Sturm, K.-T. (2012).
\newblock The space of spaces: {c}urvature bounds and gradient flows on the
  space of metric measure spaces.
\newblock {\em arXiv preprint arXiv:1208.0434}.

\bibitem[Takatsu, 2010]{takatsu2010wasserstein}
Takatsu, A. (2010).
\newblock On {W}asserstein geometry of {G}aussian measures.
\newblock In {\em Probabilistic approach to geometry}, pages 463--472.
  Mathematical Society of Japan.

\bibitem[Tolstikhin et~al., 2018]{tolstikhin2018wasserstein}
Tolstikhin, I., Bousquet, O., Gelly, S., and Sch{\"o}lkopf, B. (2018).
\newblock Wasserstein auto-encoders.
\newblock In {\em 6th International Conference on Learning Representations
  (ICLR 2018)}. OpenReview. net.

\bibitem[Vayer et~al., 2019a]{titouan2019optimal}
Vayer, T., Courty, N., Tavenard, R., and Flamary, R. (2019a).
\newblock Optimal transport for structured data with application on graphs.
\newblock In {\em International Conference on Machine Learning}, pages
  6275--6284. PMLR.

\bibitem[Vayer et~al., 2019b]{titouan2019sliced}
Vayer, T., Flamary, R., Courty, N., Tavenard, R., and Chapel, L. (2019b).
\newblock Sliced {G}romov--{W}asserstein.
\newblock {\em Advances in Neural Information Processing Systems}, 32.

\bibitem[Villani, 2008]{villani2008optimal}
Villani, C. (2008).
\newblock {\em Optimal transport: {o}ld and new}, volume 338.
\newblock Springer Science \& Business Media.

\bibitem[Xu et~al., 2019]{xu2019scalable}
Xu, H., Luo, D., and Carin, L. (2019).
\newblock Scalable {G}romov--{W}asserstein learning for graph partitioning and
  matching.
\newblock {\em Advances in neural information processing systems}, 32.

\bibitem[Xu et~al., 2018]{xu2018distilled}
Xu, H., Wang, W., Liu, W., and Carin, L. (2018).
\newblock Distilled {W}asserstein learning for word embedding and topic
  modeling.
\newblock {\em Advances in Neural Information Processing Systems}, 31.

\bibitem[Yakowitz and Spragins, 1968]{yakowitz1968identifiability}
Yakowitz, S.~J. and Spragins, J.~D. (1968).
\newblock On the identifiability of finite mixtures.
\newblock {\em The Annals of Mathematical Statistics}, 39(1):209--214.

\bibitem[Zhang and Chen, 2020]{zhang2020unified}
Zhang, Q. and Chen, J. (2020).
\newblock A unified framework for {G}aussian mixture reduction with composite
  transportation distance.
\newblock {\em arXiv preprint arXiv:2002.08410}.

\end{thebibliography}

\appendix 

\theoremstyle{plain}

\setcounter{figure}{0}
\setcounter{table}{0}
\setcounter{lemma}{0}
\makeatletter

\renewcommand{\thefigure}{A\arabic{figure}}
\renewcommand{\thetheorem}{A\arabic{theorem}}
\renewcommand{\thedefinition}{A\arabic{definition}}
\renewcommand{\thelemma}{A\arabic{lemma}}
\renewcommand{\thesection}{A}
\renewcommand{\theremark}{A\arabic{remark}}
\renewcommand{\theproposition}{A\arabic{proposition}}
\renewcommand{\thecorollary}{A\arabic{corollary}}
\setcounter{tocdepth}{1}

\newpage
\section*{Organization of the supplementary}
The supplementary is organized as follows. First, in 
\Cref{sec:teclem}, we show
six technical results that will be used 
throughout the proofs of the paper. In \Cref{sec:proofs}, 
we give the full proofs of the technical results of the paper. Finally, 
in \Cref{sec:addresultspw2}, we give more details on the difference 
between $ EW_2 $ and the OT distance introduced by \cite{cai2022distances}.

\section{Technical lemmas}
\label{sec:teclem}
Before turning to the proofs of the theoretical results, we state here six technical lemmas that will be used 
throughout the proofs of the results of the paper. 

\subsection{A property of couplings between measures living in different dimensions}
First we start by recalling the following 
result \cite[Lemma 3.3]{delon2021generalized}.

\begin{citedlemma}[\textbf{\textup{\cite[]{delon2021generalized}}}]\label{lem:mappingop}
  Let $ \mu \in \mathcal{W}_2(\rset^\d) $ and $ \nu \in \mathcal{W}_2(\rset^\di) $ with $ \d$ 
  not necessarily greater than $ \di$, 
  and let $ T \colon \rset^\di \rightarrow \rset^\d $ be a measurable map.
  Then $ \op' \in \Pi(\mu,T_{\#}\nu) $ if and only if there is some $ \op \in \Pi(\mu,\nu)  $
  such that $ \op' = (\Id_\d,T)_{\#}\op $. In particular, 
  if there exist $ a, b \geq 0 $ such that $ \|T(y)\| \leq a + b\|y\| $ 
  for all $ y \in \rset^\di  $, then
  \begin{equation}
  \inf_{\op \in \Pi(\mu,\nu)} \int_{\rset^\d \times \rset^\di} \|x-T(y)\|^2 \rmd \op(x,y) = \inf_{\op \in \Pi(\mu,T_{\#}\nu)} \int_{\rset^\d \times \rset^\d} \|x-z\|^2 \rmd \op(x,z) \eqsp.
  \end{equation}
  \end{citedlemma}

\subsection{Isometries in Euclidean spaces}
We show the following result, that states that any isometry $ T \colon \rset^\di \rightarrow \rset^\d $
for the Euclidean norm is affine and of the form, for all $ y \in \rset^\di $, $ T(y) = Py + b $, 
where $ b \in \rset^\d $ and $ P $ is in the Stiefel manifold $ \mathbb{V}_\di(\rset^\d) $. 

\begin{lemma}\label{lem:isoeuclidean} Suppose $ \d \geq \di $. Then $ \phi \colon \rset^\di \rightarrow \rset^\d $ 
  is an isometry for the Euclidean norm if and only if there exist $ P \in \mathbb{V}_\di(\rset^\d) $ and $ b \in \rset^\d $ 
  such that for all $ y \in \rset^\di $,
  $ \phi $ is of the form 
  \begin{equation}
  \phi(y) = Py + b \eqsp.
  \end{equation}
\end{lemma}

\begin{proof}
  First observe that for $ P \in \mathbb{V}_\di(\rset^\d) $ and $ b \in \rset^\d $, $ y \mapsto Py + b $ 
  is an isometry since we have, for any $ y $  and $ y' $ in $ \rset^\di $
  \begin{equation}
  \|Py + b - Py' - b\|^2  = \|P(y-y')\|^2 = (y-y')^TP^TP(y-y') = (y-y')^T(y-y') = \|y-y'\|^2 \eqsp.
  \end{equation}
  The converse is a consequence of 
  the Mazur–Ulam theorem \cite[]{mazur1932transformations} that states 
  - in the version of \cite{baker1971isometries} - that an isometry from a real normed space 
  to a \emph{strictly convex} normed space, i.e. a normed space where the unit ball is 
  a stricly convex set, is necessarily affine. Since it is easy to show that 
  the unit ball $ \ensembleLigne{x \in \rset^\d}{\|x\| \leq 1}  $ is a strictly convex set, we
  get that for all $ x \in \rset^\di $, $ \phi $ is of the form $ y \mapsto Py + b $ with 
  $ P $ being a matrix of size $ \d \times \di $, and $ b \in \rset^\d $. Moreover
  we have for all $ y, y' \in \rset^\di $ 
  \begin{equation}
  \|\phi(y) - \phi(y') \|^2 = \|Py - Py'\|^2 = \|P(y -y')\|^2 = (y-y')^TP^TP(y - y') \eqsp.
  \end{equation}
  Since $ \phi $ is an isometry, it follows that $ \|y - y'\|^2 = (y-y')^TP^TP(y - y') $
  and so $ P^TP = \Id_\di $, which concludes the proof.  
\end{proof}

\subsection{Centering of measures}

\begin{lemma}\label{lem:centerediw2}
  Let $ \mu \in \mathcal{W}_2(\rset^\d) $  and $ \nu \in \mathcal{W}_2(\rset^\di) $ 
  with $ \d$  not necessarily greater than $ \di$. Let $ \bar{\mu} $ 
  and $ \bar{\nu} $ denote the centered measures associated
  to $ \mu $ and $ \nu $ and let $ \mathfrak{P} $ be any subset of matrices 
    of size $ \d \times \di $. Then,
    \begin{equation}
    \inf_{\op \in \Pi(\mu,\nu)} \inf_{P \in \mathfrak{P}, \ b \in \rset^\d} \int_{\rset^\d \times \rset^\di} \|x - Py - b \|^2 \rmd \op(x,y) = \inf_{\op \in \Pi(\bar{\mu},\bar{\nu})} \inf_{P \in \mathfrak{P}} \int_{\rset^\d \times \rset^\di} \|x - Py\|^2 \rmd \op(x,y) \eqsp.
    \end{equation}
\end{lemma}

\begin{proof}
  Denoting $ m_0 = \mathbb{E}_{X \sim \mu}[X] $, 
  $ m_1 = \mathbb{E}_{Y \sim \nu}[Y] $, 
  $ \tilde{x} = x - m_0 $, and $ \tilde{y} = y - m_1 $, we have for any $ \op \in \Pi(\mu,\nu) $,
  \begin{align}
  \int_{\rset^\d \times \rset^\di} \|x - Py - b \|^2 \rmd \op(x,y) 
  &= \int_{\rset^\d \times \rset^\di} \|\tilde{x} - P\tilde{y} - b + m_0 - Pm_1\|^2 \rmd \op(x,y) \\
  &= \|m_0 - b - Pm_1\|^2 + \int_{\rset^\d \times \rset^\di} \|\tilde{x} - P\tilde{y}\|^2 \rmd \op(x,y) \eqsp, \\
  \end{align}
  since $ \int \langle \tilde{x} - P\tilde{y}, m_0 - b - Pm_1 \rangle \rmd \op(x,y)  = 0 $. Thus
  it follows,
  \begin{align}
  \inf_{\op \in \Pi(\mu,\nu)} \inf_{P \in \mathfrak{P}, \ b \in \rset^\d} \int_{\rset^\d \times \rset^\di} &\|x - Py - b \|^2 \rmd \op(x,y) \\
  &= \inf_{P \in \mathfrak{P}} \left(\inf_{b \in \rset^\d} \|m_0 - Pm_1 - b \|^2 + \inf_{\op \in \Pi(\bar{\mu},\bar{\nu})}  \int_{\rset^\d \times \rset^\di}\|x-Py\|^2 \rmd \op(x,y) \right) \eqsp. 
  \end{align}
  Observe now that 
  for any $ P \in \mathfrak{P} $, $ \|m_0 - Pm_1 - b \|^2 = 0 $ if $ b = m_0 - Pm_1 $, which concludes the proof.
\end{proof}

\subsection{A matrix linear program}
\begin{lemma}\label{lem:maxsingvalue}
  Let $ K $ be a matrix of size $ \d \times \di $ with Singular
  Value Decomposition (SVD) $ K = U_K\Sigma_KV_K^T $ and 
  let $ \mathfrak{P} $ be any compact set of matrices of size $ \d \times \di $. Then, 
  \begin{equation}
  \sup_{P \in \mathfrak{P}} \mathrm{tr}(P^TK) = \max_{P \in \mathfrak{P}} \mathrm{tr}(\Sigma^T_P\Sigma_K) \eqsp,
  \end{equation}
  where $ \Sigma_P = \mathrm{diag}^{[\d,\di]}(\boldsymbol{\sigma}(P)) $ with 
  $ \boldsymbol{\sigma}(P) \in \rset_+^\di $ denoting the vector of singular values of P. Furthermore the supremum is
  achieved at $ P $ of the form, 
  \begin{equation}
  P = U_K\Sigma_PV_K^T \eqsp. 
  \end{equation}
  \end{lemma}

  \begin{proof}
    Note that this lemma can be proven 
    with a proof similar to the one of \citet[Lemma 4.2]{alvarez2019towards}, using 
    the min-max theorem for singular values. Here we offer 
    an alternative proof based on Lagragian analysis. 
    First observe that the supremum is achieved 
    as a direct consequence of the Weierstrass theorem because $ \mathfrak{P} $ 
    is compact and the mapping $ P \mapsto \mathrm{tr}(P^TK) $ is continuous.  
    For a given $ P \in \mathfrak{P} $, let $ U_P\Sigma_PV_P^T $ be the SVD of $ P $. 
    The problem can be rewritten as 
    \begin{equation}
    \max_{P \in \mathfrak{P}} \mathrm{tr}(V_P\Sigma^T_PU_P^TU_K\Sigma_KV_K^T) \eqsp.
    \end{equation}
    Now, let us denote $ U = U_P^TU_K $ and $ V = V_P^TV_K $. Observe 
    that $ U $ is in $ \mathbb{O}(\rset^\d) $ and $ V $ is in 
    $ \mathbb{O}(\rset^\di) $. Using the cyclical permutation of the trace operator, the problem becomes
    \begin{equation}
    \max_{P \in \mathfrak{P}} \mathrm{tr}(\Sigma^T_PU\Sigma_KV^T) \eqsp. 
    \end{equation}
    Now, for a given fixed $ \Sigma_P $, we determine which $ U $ and $ V $ 
    maximize $ \mathrm{tr}(\Sigma^T_PU\Sigma_KV^T) $. This problem reads as
    \begin{equation}
    \max_{U \in \mathbb{O}(\rset^\d), \ V \in \mathbb{O}(\rset^\di)} \mathrm{tr}(\Sigma^T_PU\Sigma_KV^T) \eqsp. 
    \end{equation}
    The Lagrangian of this problem reads as
    \begin{equation}
    \mathcal{L}(U,V,C_0,C_1) = - \mathrm{tr}(\Sigma^T_PU\Sigma_KV^T) + \mathrm{tr}(C_0(U^TU - \Id_\d)) + \mathrm{tr}(C_1(V^TV - \Id_\di)) \eqsp,
    \end{equation}
    where $ C_0 \in \mathbb{S}^\d $ and $ C_1 \in \mathbb{S}^\di $ are the Lagrange multipliers 
    respectively associated with the constraints $ U \in \mathbb{O}(\rset^\d) $ and $ V \in \mathbb{O}(\rset^\di) $.
    The first order condition gives 
    \begin{equation}
    \left\{ \begin{array}{ll} \Sigma_PV\Sigma_K^T = 2UC_0 \\
             \Sigma_P^TU\Sigma_K = 2VC_1 \eqsp, \end{array}\right. 
    \end{equation}
    or equivalently
    \begin{equation}
        \left\{ \begin{array}{ll} U^T\Sigma_PV\Sigma_K^T = 2C_0 \\
                 \Sigma_P^TU\Sigma_KV^T = 2VC_1V^T \eqsp. \end{array}\right. 
    \end{equation}
    Since $ C_0 $ and $ C_1 $ are symmetric matrices (because they are associated with symmetric constraints), 
    we get that both left-hand terms are symmetric. This gives the following conditions
    \begin{equation}
        \left\{ \begin{array}{ll} U^T\Sigma_PV\Sigma_K^T = \Sigma_KV^T\Sigma_P^TU \\
            \Sigma_P^TU\Sigma_KV^T  = V\Sigma_K^TU^T\Sigma_P \eqsp. \end{array}\right. 
    \end{equation}
    Now, observe that when multiplying the first condition at right  by $ U^T\Sigma_P $ 
    and multiplying the second condition at left by $ \Sigma_KV^T $, we 
    get by combining the two conditions 
    \begin{equation}
        \left\{ \begin{array}{ll} U\Sigma_KV^T\Sigma_P^T\Sigma_P = \Sigma_P\Sigma_P^TU\Sigma_KV^T  \\      
            U^T\Sigma_PV\Sigma_K^T\Sigma_K = \Sigma_K\Sigma_K^TU^T\Sigma_PV^T \eqsp, \end{array}\right. 
    \end{equation}
    or equivalently, 
    \begin{equation}
        \left\{ \begin{array}{ll} U\Sigma_KV^TD_P = D_P^{[\d]}U\Sigma_KV^T  \\
            U^T\Sigma_PVD_K = D_K^{[\d]}U^T\Sigma_PV^T \eqsp, \end{array}\right. 
    \end{equation}
    where $ D_P = \mathrm{diag}(\boldsymbol{\sigma}(P)) $ and $ D_K = \mathrm{diag}(\boldsymbol{\sigma}(K)) $. 
    Multiplying the first condition at left by $ V\Sigma_K^TU^T $ and the 
    second condition at right by $ V\Sigma_P^TU $, this yields to 
    \begin{equation}
        \left\{ \begin{array}{ll} VD_KV^TD_P = V\Sigma U^T D_P^{[\d]}U\Sigma_KV^T  \\
             D_K^{[\d]}U^TD_P^{[\d]}U = U^T\Sigma_PVD_KV\Sigma_P^TU \eqsp. \end{array}\right. 
    \end{equation}
    It follows that $ VD_KV^TD_P $ and $ D_K^{[\d]}U^TD_P^{[\d]}U $ are symmetric matrices 
    and so $ VD_KV^T $ commutes with $ D_P $ and $ U^TD_P^{[\d]}U $ commutes with $ D_K^{[\d]} $. 
    Thus we can deduce that $ U $ and $ V $ are permutation matrices. Since the 
    singular values are ordered in non-increasing order, we deduce that the problem 
    is maximized when $ U = \Id_{\d} $ and $ V = \Id_{\di} $. This implies
    that $ U_P = U_K $ and $ V_P = V_K $, which concludes the proof. 
    \end{proof}
    
    Note that \Cref{lem:maxsingvalue} is especially useful when the 
    constraint of belonging to the set $ \mathfrak{P} $ can 
    be expressed as a constraint on the singular values. Observe 
    that this is the case of $ \mathbb{V}_{\di}(\rset^\d) $ 
    since for all $ P \in \mathbb{V}_{\di}(\rset^\d) $, we have $ P^TP = \Id_{\di} $ 
    and so an equivalent condition of belonging in $ \mathbb{V}_{\di}(\rset^\d) $ 
    is that $ \boldsymbol{\sigma}(P) = \mathbbm{1}_\di $.

\subsection{Some properties of symmetric matrices}
Here we state two technical results on symmetric matrices that will be 
useful in the proofs of the results on Gaussian distributions.

\begin{lemma}\label{lem:eigenvectors}
  Let $ A \in \mathbb{S}^\d $. We denote $ \lambda_1 $ 
  and $ \lambda_\d $ its largest and smallest eigenvalues. For all $ x \in \rset^\d $ 
  such that $ \|x\|=1 $, we have
  \begin{itemize}
  \item[(i)] $ x $ is an eigenvector of $ A $ associated to $ \lambda_1 $ if and only if $ x^TAx = \lambda_1 $.
  \item[(ii)] $ x $ is an eigenvector of $ A $ associated to $ \lambda_\d $ if and only if $ x^TAx = \lambda_\d $.
  \end{itemize}
  \end{lemma}
  
  \begin{proof}
  Let $ x \in \rset^\d $ such $ \|x\|=1 $. Since $ A $ is symmetric, there exists 
  $ O \in \mathbb{O}(\rset^\d) $ and $ \Lambda = \mathrm{diag}((\lambda_k)_{1 \leq k \leq \d}) $ such 
  that $ x^TAx = x^TO \Lambda O^Tx $. Denoting $ z $ the vector $ O^Tx $, we get thus
  \begin{equation}
  x^TAx = z^T\Lambda z = \sum_{k=1}^\d \lambda_k z_k^2 \eqsp.
  \end{equation}
  Hence it follows that 
  \begin{equation}
  \lambda_\d\|z\|^2 \leq x^TAx \leq \lambda_1\|z\|^2 \eqsp,
  \end{equation}
  with equality if and only if $ z $ is an eigenvector associated with $ \lambda_1 $ or $ \lambda_\d $. 
  \end{proof}

  \begin{lemma}\label{lem:gaussianlem}
    Suppose that $ \d \geq \di $. 
    Let $ \Sigma $ be a positive semi-definite (PSD) matrix of size $ \d + \di $ of the form 
    \begin{equation}
    \Sigma = \begin{pmatrix} \Sigma_0 &  K \\ K^T & \Sigma_1 \end{pmatrix} \eqsp, 
    \end{equation}
    with $ \Sigma_0 \in \mathbb{S}_{++}^\d $, $ \Sigma_1 \in \mathbb{S}_{+}^\di $ 
    and $ K $ being a rectangular matrix of size $ \d \times \di $. Let $ S = \Sigma_1 - K^T\Sigma_0^{-1}K $ 
    be the Schur complement of $ \Sigma $. 
    Then there exists $ r \leq \di $ 
    and $ B_r \in \mathbb{V}_r(\rset^\d) $ such that 
    \begin{equation}
    K = \Sigma_0^{\frac{1}{2}}B_r\Lambda_rU_r^T \eqsp,
    \end{equation}
    where $ U_r \in \mathbb{V}_r(\rset^\di) $ and $ \Lambda_r $ is a diagonal positive matrix of size $ r $ such 
    that 
    \begin{equation}
    \Sigma_1 - S = U_r\Lambda^2_rU_r^T  \eqsp.
    \end{equation}
    \end{lemma}
    \begin{proof}
    For a given Schur complement $ S = \Sigma_1 - K^T\Sigma_0^{-1}K $, 
    we have $ K^T\Sigma_0^{-1}K = \Sigma_1 - S $. 
    Since $ \Sigma_0 \in \mathbb{S}_{++}^\d $,  we can deduce that $ K^T\Sigma_0^{-1}K \in \mathbb{S}_{+}^\di $ 
    and so that $ \Sigma_1 - S \in \mathbb{S}_{+}^\di $. 
    We note $ r $ the rank of $ K^T\Sigma_0^{-1}K $. 
    One can observe that 
    \begin{equation} 
    r \leq \di \leq \d \eqsp, 
    \end{equation}
    where the left-hand side inequality follows 
    from the fact that $ \text{rk}(AB) \leq \min\{\text{rk}(A),\text{rk}(B)\}$. 
    Then, $ \Sigma_1 - S $ can be diagonalized
    \begin{equation}\label{eq:diagonalisation}
    \Sigma_1 - S = K^T\Sigma_0^{-1}K = U\Lambda^2 U^T = U_r\Lambda^2_rU_r^T \eqsp,
    \end{equation}
    with  $ \Lambda^2 = \text{diag}(\lambda_1^2,...,\lambda_r^2)^{[\di]} $, 
    $ \Lambda_r^2 = \text{diag}(\lambda_1^2,...,\lambda_r^2) $, 
    and $ U_r \in \mathbb{V}_r(\rset^\di) $ such that  $ U = \begin{pmatrix} U_r  & U_{\di-r} \end{pmatrix}$. 
    From \eqref{eq:diagonalisation}, we can deduce that
    \begin{equation}
    (\Sigma_0^{-\frac{1}{2}}KU_r\Lambda_r^{-1})^T\Sigma_0^{-\frac{1}{2}}KU_r\Lambda_r^{-1} = \Id_r \eqsp,
    \end{equation}
    where $ \Lambda_r $ is the unique PSD square-root 
    of $ \Lambda^2_r $. Let us set $ B_r = \Sigma_0^{-\frac{1}{2}}KU_r\Lambda_r^{-1} $ such that $ B_r \in 
    \mathbb{V}_r(\rset^\d) $. It follows that
    \begin{equation}
    KU_r = \Sigma_0^{\frac{1}{2}}B_r\Lambda_r \eqsp.
    \end{equation}
    Moreover, since $ U_{\d-r}^TK^T\Sigma_0^{-1}KU_{\d-r} = 0 $ and $\Sigma_0 \in S_\d^{++}(\mathbb{R}) $, it follows that $ KU_{\di-r} = 0 $ and so 
    \begin{equation}\label{eq:Kexpression}
    K = KUU^T = KU_rU_r^T = \Sigma_0^{\frac{1}{2}}B_r\Lambda_rU_r^T \eqsp,
    \end{equation}
    which concludes the proof. 
    \end{proof}

\setcounter{figure}{0}
\setcounter{table}{0}
\setcounter{lemma}{0}
\makeatletter

\renewcommand{\thefigure}{B\arabic{figure}}
\renewcommand{\thetheorem}{B\arabic{theorem}}
\renewcommand{\thedefinition}{B\arabic{definition}}
\renewcommand{\thelemma}{B\arabic{lemma}}
\renewcommand{\thelemmaB}{B\arabic{lemmaB}}
\renewcommand{\thesection}{B}
\renewcommand{\theremark}{B\arabic{remark}}
\renewcommand{\theproposition}{B\arabic{proposition}}
\renewcommand{\thecorollary}{B\arabic{corollary}}

\section{Proofs of the theoretical results}
\label{sec:proofs}
\subsection{Proof of \Cref{prop:mgw2metric}}
\label{sec:proofmgw2metric}
\begin{proof}[Proof of \Cref{prop:mgw2metric}] 
\cite{takatsu2010wasserstein} has shown 
that the space of Gaussian distributions $ \mathcal{N}(\rset^\d) $ is a complete metric
space when endowed with $ W_2 $. Moreover,  $ \mathcal{N}(\rset^\d) $ 
is separable since it is a subspace of $ \mathcal{W}_2(\rset^\d) $
which is itself a separable metric space when endowed with $ W_2 $ \cite[]{bolley2008separability}. 
Thus,  $ \mathcal{N}(\rset^\d) $
is Polish and we can directly apply the Gromov-Wasserstein theory developped in \cite{sturm2012space}. Let 
$ (\mathcal{N}(\rset^\d),W_2,\tilde{\mu}) $ 
and $ (\mathcal{N}(\rset^\di),W_2,\tilde{\nu}) $ be two metric measure spaces 
respectively in $ \mathbb{M}_4 $. 
Let us define
\begin{equation}
D(\tilde{\mu},\tilde{\nu}) = \inf_{\op \in \Pi(\tilde{\mu},\tilde{\nu})} \int_{\mathcal{N}(\rset^\d) \times \mathcal{N}(\rset^\di)} \int_{\mathcal{N}(\rset^\d) \times \mathcal{N}(\rset^\di)} |W_2^2(\gamma,\gamma') - W_2^2(\zeta,\zeta')|^2 \rmd \op(\gamma,\zeta) \rmd \op(\gamma',\zeta') \eqsp.
\end{equation}
Applying \citet[Corollary 9.3]{sturm2012space}, 
we get that $ D $ defines a metric over the space of 
metric measure spaces of the form $ (\mathcal{N}(\rset^\d),W_2,\tilde{\mu}) $ quotiented by the strong isomorphisms, and thus we get directly that $ D $ is symmetric, 
non-negative, satisfies the triangle inequality and 
$ D(\tilde{\mu},\tilde{\nu}) = 0 $ if and only if there exists a bijection 
$ \phi \colon \mathrm{supp}(\tilde{\mu}) \rightarrow \mathrm{supp}(\tilde{\nu}) $ 
such that $ \tilde{\nu} = \phi_{\#}\tilde{\mu} $, 
where for any $ \gamma $ and $ \gamma' $ in $ \mathrm{supp}(\tilde{\mu}) $, 
$ W_2(\phi(\gamma),\phi(\gamma')) = W_2(\gamma,\gamma')$. Now observe
that if $\mu  = \sum_{k} a_k\mu_k $ and $ \nu = \sum_{l} b_l \nu_l$ 
are respectively in $ GMM_K(\rset^\d) $ and $ GMM_L(\rset^\di) $ and 
$ \tilde{\mu} = \sum_{k} a_k\delta_{\mu_k} $ and $ \tilde{\nu} = \sum_{l} b_l\delta_{\nu_l} $
are respectively in $ \mathcal{P}(\mathcal{N}(\rset^\d)) $ 
and $ \mathcal{P}(\mathcal{N}(\rset^\di)) $, we have 
\begin{equation}
\int_{\mathcal{N}(\rset^\d) \times \mathcal{N}(\rset^\d)} W_2^4(\gamma,\gamma') \rmd\tilde{\mu}(\gamma) \rmd\tilde{\mu}(\gamma') = \sum_{k,i} a_k a_i W_2^4(\mu_k,\mu_i) < + \infty \eqsp,
\end{equation}
and
\begin{equation}
\int_{\mathcal{N}(\rset^\di) \times \mathcal{N}(\rset^\di)} W_2^4(\zeta,\zeta') \rmd\tilde{\nu}(\zeta) \rmd \tilde{\nu}(\zeta') = \sum_{l,j} b_l b_j W_2^4(\nu_l,\nu_j) < + \infty \eqsp,
\end{equation}
so $ (\mathcal{N}(\rset^\d),W_2,\tilde{\mu}) $  and $ (\mathcal{N}(\rset^\di),W_2,\tilde{\nu}) $ 
are both in $ \mathbb{M}_4 $. 
Furthermore, we have $ MGW_2(\mu,\nu) =  D(\tilde{\mu},\tilde{\nu}) $. 
Hence $ MGW_2 $ inherits the metric properties of $ D $, which concludes the proof. 
\end{proof}

\subsection{Proof of \Cref{prop:invcase}}
\label{sec:proof:invcase}
\begin{proof}[Proof of \Cref{prop:invcase}.]
First recall that the push-foward measure $ T_{\#}\mu $ with $ \mu $ on 
$ \rset^\di $ and $ T \colon \rset^\di \rightarrow \rset^\d $ is defined as the measure on $ \rset^\d $ such that for every Borel set 
$ \msa $ of $ \rset^\d $, $ T_{\#}\mu(\msa) = \mu(T^{-1}(\msa)) $. Equivalently, 
for any measurable map $ h \colon \rset^\d \rightarrow \rset $, we have
\begin{equation}
\int_{\rset^\d} h(x) \rmd(T_{\#}\mu)(x) = \int_{\rset^\di} (h \circ T)(y) \rmd \mu (y) \eqsp. 
\end{equation}
Now observe that for any finite GMM $ \mu $ on $ \rset^\di $ of the form $ \mu = \sum_k^K a_k\mu_k $, we
have 
\begin{align}
\textstyle{\int_{\rset^\di} (h \circ T)(y) \rmd \mu (y)} &=  \textstyle{\int_{\rset^\di} (h \circ T)(y) \rmd\left(\sum_k^K a_k\mu_k(y)\right)} \\
 &= \textstyle{\sum_k^Ka_k \int_{\rset^\di} (h \circ T)(y) \rmd \mu_k(y)}  \\
 &= \textstyle{\sum_k^Ka_k \int_{\rset^\d} h(x) \rmd (T_{\#}\mu_k)(x)} \\
 &= \textstyle{\int_{\rset^\d} h(x) \rmd \left(\sum_k^Ka_k(T_{\#}\mu_k)(x) \right)} \eqsp,
\end{align}
and so $ T_{\#}\mu $ is of the form $ \sum_k^Ka_k(T_{\#}\mu_k) $
with $ T_{\#}\mu_k $ Gaussian since $ T $ is necessarily affine as a consequence of \Cref{lem:isoeuclidean}. Thus, $ T_{\#}\mu $ is  
in $ GMM_{\infty}(\rset^\d) $. This proves that $ \phi_{T} $ takes its values only in $ GMM_\infty(\rset^\d) $
and that $ \phi_{T}(\sum_{k=1}^K a_k \mu_k) $ is of the form $ \sum_{k=1}a_k\nu_k $. 
Now observe that, for every $ k $ and $ i $ smaller than $ K $,
\begin{equation}
W^2_2(\phi_{T}(\mu_k),\phi_{T}(\mu_i)) = \inf_{\op \in \Pi(T_{\#}\mu_k,T_{\#}\mu_i)} \int_{\rset^\d \times \rset^\d} \|x-y\|^2 \rmd \op(x,y) \eqsp. 
\end{equation}
Using two times successively \Cref{lem:mappingop} using the fact that $ T $ is an isometry an so 
for any $ y \in \rset^\di $, $ \|T(y)\|=\|y\| $, it follows
\begin{equation}
\inf_{\op \in \Pi(T_{\#}\mu_k,T_{\#}\mu_i)} \int_{\rset^\d \times \rset^\d} \|x-x'\|^2 \rmd \op(x,x') = \inf_{\op \in \Pi(\mu_k,\mu_i)}  \int_{\rset^\di \times \rset^\di} \|y-y'\|^2 \rmd \op(y,y') = W_2(\mu_k,\mu_i) \eqsp.
\end{equation}
Thus, $ MGW_2(\mu,T_{\#}\mu) = 0 $ as a direct consequence of \Cref{prop:mgw2metric}, which concludes the proof.
\end{proof}

\subsection{Proof of \Cref{prop:ew2eq}}
\label{sec:proofew2eq}
We prove \Cref{prop:ew2eq} before proving \Cref{prop:ew2metric} because 
we will use the former in the proof of the latter. 

\begin{proof}[Proof of \Cref{prop:ew2eq}.]
Since we suppose $ \d \geq \di $, we have 
\begin{equation}
EW_2^2(\mu,\nu) = \inf_{\phi \in \mathrm{Isom}_\di(\rset^\d)} W_2^2(\mu,\phi_{\#}\nu) \eqsp. 
\end{equation}
Let $ \phi \in \mathrm{Isom}_\di(\rset^\d) $ for the Euclidean norm. Using 
\Cref{lem:isoeuclidean}, we get that there exists $ P \in \mathbb{V}_\di(\rset^\d) $ and $ b \in \rset^\d $ such 
that for all $ y \in \rset^\di $, $ \phi(y) = Py + b $.
Moreover, we have, using \Cref{lem:mappingop},
\begin{align}
EW_2^2(\mu,\nu) &= \inf_{\phi \in \mathrm{Isom}_\di(\rset^\d)} \inf_{\op \in \Pi(\mu,\phi_{\#}\nu)} \int_{\rset^\d \times \rset^\d} \|x - y\|^2 \rmd \op(x,y) \\
&= \inf_{\phi \in \mathrm{Isom}_\di(\rset^\d)} \inf_{\op \in \Pi(\mu,\nu)} \int_{\rset^\di \times \rset^\d} \|x - \phi(y)\|^2 \rmd \op(x,y) \\
& = \inf_{\op \in \Pi(\mu,\nu)} \inf_{P \in \mathbb{V}_\di(\rset^\d), \ b \in \rset^\d} \int_{\rset^\di \times \rset^\d} \|x - Py - b\|^2 \rmd \op(x,y) \eqsp,
\end{align}
which proves Equation \eqref{eq:ew2def2}. Now we show the equivalence with Problem \eqref{eq:nucnorm}. Using \Cref{lem:centerediw2}, Problem \eqref{eq:ew2def2} can be rewritten 
      \begin{align}
      EW_2^2(\mu,\nu) &= \inf_{P \in \mathbb{V}_\di(\rset^\d)} \inf_{\op \in \Pi(\bar{\mu},\bar{\nu})} \int_{\rset^\d \times \rset^\di} \|x-Py\|^2\rmd \op(x,y) \\
      &= \inf_{P \in \mathbb{V}_\di(\rset^\d)} \inf_{\op \in \Pi(\bar{\mu},\bar{\nu})} \int_{\rset^\d \times \rset^\di} \left(\|x\|^2 + \|Py\|^2 - 2\langle x, Py \rangle \right) \rmd \op(x,y) \eqsp.
      \end{align}
      Since for all $ P \in \mathbb{V}_\di(\rset^\d) $, $ \|Py\| $ doesn't depend on $ P $, we get that the 
      problem is equivalent to 
      \begin{equation}
      \sup_{P \in \mathbb{V}_\di(\rset^\d)} \sup_{\op \in \Pi(\bar{\mu},\bar{\nu})} \int_{\rset^\d \times \rset^\di} \langle x,Py \rangle \rmd \op(x,y) \eqsp. 
      \end{equation}
      Now observe that for all $ \op \in \Pi(\bar{\mu},\bar{\nu}) $,
      \begin{equation}
      \int_{\rset^\d \times \rset^\di} \langle x,Py \rangle \rmd \op(x,y) = \int_{\rset^\d \times \rset^\di} \mathrm{tr}(xy^TP^T) \rmd \op(x,y) = \int_{\rset^\d \times \rset^\di} \mathrm{tr}(P^Txy^T) \rmd \op(x,y)  \eqsp,
      \end{equation}
      where we used the cyclical permutation property of the trace operator. Finally using the linearity of the trace, we 
      get that the problem is equivalent to 
      \begin{equation}
      \sup_{P \in \mathbb{V}_\di(\rset^\d)} \sup_{\op \in \Pi(\bar{\mu},\bar{\nu})} \mathrm{tr}\left(P^T\int_{\rset^\d \times \rset^\di} xy^T \rmd \op(x,y)\right) \eqsp,
      \end{equation}
      or equivalently, 
      \begin{equation}
      \sup_{P \in \mathbb{V}_\di(\rset^\d)} \sup_{\op \in \Pi(\bar{\mu},\bar{\nu})} \left\langle P, \int_{\rset^\d \times \rset^\di} xy^T \rmd \op(x,y)\right\rangle \eqsp.
      \end{equation}
      Now, using \Cref{lem:maxsingvalue} and using the fact that if $ P \in \mathbb{V}_\di(\rset^\d) $, $ \boldsymbol{\sigma}(P) = \mathbbm{1}_\di $,
      we get that the problem reduces to 
      \begin{equation}
      \sup_{\op \in \Pi(\bar{\mu},\bar{\nu})} \left\|\int_{\rset^\d \times \rset^\di} xy^T \rmd \op(x,y)\right\|_{*} \eqsp,
      \end{equation}
      and this is achieved for $ P^* = U_{\op}\Id_{\di}^{[\d,\di]}V_{\op}^T $, where
      $ U_{\op} \in \mathbb{O}(\rset^\d) $ and $ V_{\op} \in \mathbb{O}(\rset^\di) $ are 
      respectively the left and right orthogonal matrices of the SVD of $ \int_{\rset^\d \times \rset^\di} xy^T \rmd \op(x,y) $, 
      which concludes the proof. 
      \end{proof}

\subsection{Proof of \Cref{prop:ew2metric}}
\label{sec:proof:ew2metric}
Before turning to the proof of \Cref{prop:ew2metric}, we will prove
two useful results. First, we show that
the $ EW_2 $ problem is always achieved at an optimal couple $ (\op^*,\phi^*) $. 

\begin{lemmaB}\label{coro:ew2achieved}
Let $ \mu \in \spa{W}_2(\rset^\d) $ and $ \nu \in \spa{W}_2(\rset^\di) $ 
and let suppose $ \d \geq \di $. Then there exists 
an optimal isometry $ \phi^* \colon \rset^\di \rightarrow \rset^\d $ 
such that $ EW_2(\mu,\nu) = W_2(\mu,\phi^*_{\#}\nu) $.
\end{lemmaB}

\begin{proof}
  Using \Cref{lem:centerediw2} and \Cref{lem:mappingop}, we have that
  \begin{align}
  EW_2^2(\mu,\nu) &= \inf_{P \in \mathbb{V}_\di(\rset^\d)} \inf_{\op \in \Pi(\bar{\mu},\bar{\nu})}  \int_{\rset^\d \times \rset^\di} \|x - Py\|^2 \rmd \op(x,y) \\
    &= \inf_{P \in \mathbb{V}_\di(\rset^\d)} W^2_2(\bar{\mu},P_{\#}\bar{\nu}) \eqsp,
  \end{align}
  where  $ \bar{\mu} $ and $ \bar{\nu} $ are the centered measures associated with $ \mu $ and $ \nu $. 
  Let us denote $ J \colon P \mapsto W_2(\bar{\mu},P_{\#}\bar{\nu}) $ and 
  let us show that $ J $ is continuous. For any $ P_0 $ and $ P_1 $ in $ \mathbb{V}_\di(\rset^\d) $, we have, 
  \begin{equation}
  |J(P_0) - J(P_1)| = |W_2(\bar{\mu},P_{0\#}\bar{\nu}) - W_2(\bar{\mu},P_{1\#}\bar{\nu})| \leq W_2(P_{0\#}\bar{\nu},P_{1\#}\bar{\nu}) \eqsp,
  \end{equation} 
  where we used the triangular inequality property of $ W_2 $. Furthermore, 
  \begin{align}
  W^2_2(P_{0\#}\bar{\nu},P_{1\#}\bar{\nu}) &= \inf_{\op \in \Pi(P_{0\#}\bar{\nu},P_{1\#}\bar{\nu})} \int_{\rset^\d \times \rset^\d} \|x-y\|^2 \rmd \op(x,y) \\
  &= \inf_{\op \in \Pi(\bar{\nu},\bar{\nu})} \int_{\rset^\di \times \rset^\di} \|P_0x - P_1y\|^2 \rmd \op(x,y) \eqsp,
  \end{align}
  where we used \Cref{lem:mappingop} twice. Now observe that the coupling $ (\Id_\di,\Id_\di)_{\#}\bar{\nu} $ is in $ \Pi(\bar{\nu},\bar{\nu}) $, so it follows 
  \begin{equation}
  \inf_{\op \in \Pi(\bar{\nu},\bar{\nu})} \int_{\rset^\di \times \rset^\di} \|P_0x - P_1y\|^2 \rmd \op(x,y) \leq  \int_{\rset^\di} \|P_0x - P_1x\|^2 \rmd \bar{\nu}(x) \eqsp.
  \end{equation}
  Finally, for any $ x \in \rset^\di $, we have
  \begin{equation}
  \|P_0x - P_1x\|^2 \leq \|x\|^2\sup_{\|z\|=1} \|(P_0 - P_1)z\|^2 \leq \|P_0 - P_1\|^2_{\spa{F}}\|x\|^2 \eqsp,
  \end{equation}
  and so it follows that
  \begin{equation}
  |J(P_0) - J(P_1)|^2 \leq \|P_0 - P_1\|^2_{\spa{F}}\int_{\rset^n}\|x\|^2 \rmd \bar{\nu}\eqsp. 
  \end{equation}
  Since $ \nu $ is in $ \mathcal{W}_2(\rset^\di) $, $ \bar{\nu} $ is in $ \mathcal{W}_2(\rset^\di) $ 
  and so $ \int_{\rset^\di}\|x\|^2 \rmd \bar{\nu} < + \infty $. It follows 
  that $ |J(P_0) - J(P_1)| \longrightarrow 0 $ when $ \|P_0 - P_1\|^2_{\spa{F}} \longrightarrow 0 $ 
  and so $ J $ is continuous. Moreover, since $  \mathbb{V}_\di(\rset^\d) $ 
  is compact \cite[]{james1976topology}, $ J $ has a minimum on $  \mathbb{V}_\di(\rset^\d) $ 
  as a result of the classic Weierstrass theorem that states
  that any real-valued continous function defined on a compact set achieves its infinimum. 
  Thus, there exists $ P^* $ such that
  $ EW_2(\mu,\nu) = W_2(\bar{\mu},P^*_{\#}\bar{\nu}) $ and setting $ b^* = \mathbb{E}_{X \sim \mu}[X] - P^*\mathbb{E}_{Y \sim \nu}[Y] $ and $ \phi^*(x) = P^*x + b^* $ 
  for all $ x \in \rset^\d $, we get that there exists $ \phi^* \in \mathrm{Isom}_\di(\rset^\d) $ such that $ EW_2(\mu,\nu) = W_2(\mu,\phi^*_{\#}\nu) $, which concludes the proof. 
\end{proof}

Now we show the following results, which imply that $ EW_2 $ remains unchanged when one of the two measures is immersed in a third Euclidean space of greater dimension than $ \d $ and $ \di $. 

\begin{lemmaB}\label{lem:embed}
Let $ \mu \in \mathcal{W}_2(\rset^\d) $ and $ \nu \in \mathcal{W}_2(\rset^\di) $ with 
$ \d $ not necessarily greater than $ \di $. Let $ r \geq \max\{\d,\di\} $ 
and let $ \psi \in \mathrm{Isom}_{\d}(\rset^r) $. Then, $ EW_2(\mu,\nu) = EW_2(\psi_{\#}\mu,\nu) $.
\end{lemmaB} 

\begin{proof}
  First, using \Cref{lem:isoeuclidean}, we get that there exists $ P_1 \in \mathbb{V}_\d(\rset^r) $
  and $ b_1 \in \rset^r $ such that for all $ x \in \rset^\d $, $ \psi(x) = P_1x + b_1 $. 
  Since $ r \geq \di $, we have, denoting $ \bar{\mu} $, $ \overline{\psi_{\#}\mu} $ and $ \bar{\nu} $ the centered measures 
  respectively associated with $ \mu $, $ \psi_{\#}\mu $, and $ \nu $, and using successively
  \Cref{lem:centerediw2} and \Cref{lem:mappingop},
  \begin{align}
  EW^2_2(\psi_{\#}\mu,\nu) &= \inf_{\op \in \Pi(\psi_{\#}\mu,\nu)} \inf_{P \in \mathbb{V}_\di(\rset^r), \ b \in \rset^r} \int_{\rset^r \times \rset^\di} \|z - Py - b\|^2 \rmd \op(z,y) \\
  &=  \inf_{\op \in \Pi(\overline{\psi_{\#}\mu},\bar{\nu})} \inf_{P \in \mathbb{V}_\di(\rset^r)} \int_{\rset^r \times \rset^\di} \|z - Py\|^2 \rmd \op(z,y) \\
  &=  \inf_{\op \in \Pi(\bar{\mu},\bar{\nu})} \inf_{P \in \mathbb{V}_\di(\rset^r)} \int_{\rset^\d \times \rset^\di} \|P_1x - Py\|^2 \rmd \op(x,y) \\
  &= \int_{\rset^\d} \|P_1x\|^2 \rmd \bar{\mu}(x) +  \int_{\rset^\di} \|Py\|^2 \rmd \bar{\nu}(y) - 2\sup_{\op \in \Pi(\bar{\mu},\bar{\nu})} \sup_{P \in \mathbb{V}_\di(\rset^r)} \mathrm{tr}(P^TP_1K_{\op}) \\
  &= \int_{\rset^\d} \|x\|^2 \rmd \bar{\mu}(x) + \int_{\rset^\di} \|y\|^2 \rmd \bar{\nu}(y) - 2\sup_{\op \in \Pi(\bar{\mu},\bar{\nu})} \sup_{P \in \mathbb{V}_\di(\rset^r)} \mathrm{tr}(P^TP_1K_{\op}) \eqsp,
  \end{align}
  where $ K_{\op} = \int_{\rset^\d \times \rset^\di} xy^T \rmd \op(x,y) $. 
  Using the equivalence with Problem \eqref{eq:nucnorm}, we get
  \begin{equation}
  \sup_{\op \in \Pi(\bar{\mu},\bar{\nu})} \sup_{P \in \mathbb{V}_\di(\rset^r)} \mathrm{tr}(P^TP_1K_{\op}) = \sup_{\op \in \Pi(\bar{\mu},\bar{\nu})} \|P_1K_{\op}\|_* \eqsp.
  \end{equation}
  Now observe that $ P_1K_{\op} $ has the same singular values as $ K_{\op} $ 
  since $ K^T_{\op}P_1^TP_1K_{\op} = K^T_{\op}K_{\op} $. Thus
  $ \|P_1K_{\op}\|_* = \|K_{\op}\|_* $ 
  and so $ EW_2(\psi_{\#}\mu,\nu) = EW_2(\mu,\nu) $, which concludes the proof.
 \end{proof}

Observe that \Cref{lem:embed} highlights that $ EW_2 $ shares
close connections with the distance between metric measure spaces introduced 
in \cite{sturm2006geometry} and defined in Equation \eqref{eq:sturmdist}. However 
it is not clear whether the two distances are strictly equivalent or not because 
the infinimum in $ \spa{Z} $ in Equation \eqref{eq:sturmdist} also includes 
non-Euclidean spaces. However, if we restrict the problem only to Euclidean spaces $ \spa{Z} $, then \Cref{lem:embed} directly implies that the two distances are
equivalent. Now we are ready to prove \Cref{prop:ew2metric}.

\begin{proof}[Proof of \Cref{prop:ew2metric}]
  First observe that non-negativity is straightforward. Furthermore, 
  observe also that if $ \d \neq \di $, symmetry 
  is also straightfoward. Now suppose $ \d = \di $ and observe 
  that that the set $ \mathbb{V}_\di(\rset^\d) $ 
  coincides with the set of orthogonal matrices $ \mathbb{O}(\rset^\d) $. Thus we have 
  \begin{align}
  \inf_{\phi \in \mathrm{Isom}_\d(\rset^\d)} W_2(\mu,\phi_{\#}\nu) &= \inf_{\op \in \Pi(\mu,\nu)} \inf_{P \in \mathbb{O}(\rset^\d), \ b \in \rset^\d} \int_{\rset^\d \times \rset^\d} \|x - Py - b \|^2 \rmd \op(x,y) \\
  &= \inf_{\op \in \Pi(\mu,\nu)} \inf_{P \in \mathbb{O}(\rset^\d), \ b \in \rset^\d} \int_{\rset^\d \times \rset^\d} \|P^Tx - y - P^Tb \|^2 \rmd \op(x,y) \\
  &= \inf_{\psi \in \mathrm{Isom}_\d(\rset^\d)} W_2(\psi_{\#}\mu,\nu) \eqsp,
  \end{align}
  and so $ EW_2 $ is also symmetric in that case. 
  Before turning 
  to the proof of the two other points, we recall that the infinimum in $ \phi $ is always achieved, see \Cref{coro:ew2achieved}. 
  \begin{itemize}
  \item[(i)] Now we prove the triangle inequality. Let $ r \geq \max\{\d,\di,\dii\} $, $ \phi_0 \in \mathrm{Isom}_\d(\rset^r) $ and 
  for $ \xi \in \spa{W}_2(\rset^\dii)$, let $ \phi_{1} \in \argmin_{\phi \in \mathrm{Isom}_\dii(\rset^r)} 
  W_2(\phi_{0\#}\mu,\phi_{\#}\xi) $. 
  We have, using first \Cref{lem:embed}, then using the triangle inequality property of $ W_2 $, 
  \begin{align}
  EW_2(\mu,\nu) = EW_2(\phi_{0\#}\mu,\nu) &= \inf_{\phi \in \mathrm{Isom}_\di(\rset^r)} W_2(\phi_{0\#}\mu,\phi_{\#}\nu) \\ 
  & \quad \leq \inf_{\phi \in \mathrm{Isom}_\di(\rset^r)} \left[W_2(\phi_{0\#}\mu,\phi_{1\#}\xi) + W_2(\phi_{1\#}\xi,\phi_{\#}\nu)\right] \\
  & \quad \leq W_2(\phi_{0\#}\mu,\phi_{1\#}\xi) + \inf_{\phi \in \mathrm{Isom}_\di(\rset^r)} W_2(\phi_{1\#}\xi,\phi_{\#}\nu) \\
  & \quad \leq EW_2(\phi_{0\#}\mu,\xi) + EW_2(\phi_{1\#}\xi,\nu) \eqsp. 
  \end{align} 
  We conclude then by applying \Cref{lem:embed} on both terms. 
  \item[(ii)] Suppose without any loss of generality that 
  $ d \geq \di $ and suppose $ EW_2(\mu,\nu) = 0 $. Since the infinimum in $ \phi $ is achieved, 
  there exists $ \phi \in \mathrm{Isom}_\di(\rset^\d) $ such 
  that $ W_2(\mu,\phi_{\#}\nu) = 0 $ and so $ \mu = \phi_{\#}\nu $. The reverse implication is obvious. 
  \end{itemize}
  Finally, observe that if $ \mu $ and $ \nu $ have finite order $ 2 $ moments, then 
  $ EW_2 $ necessarily takes finite values, and so $ EW_2 $ defines a pseudometric on $ \bigsqcup_{k \geq 1} \spa{W}_2(\rset^k) $. 
\end{proof}

\subsection{Proof of \Cref{thm:iw21}}
\label{sec:proofiw21}
\begin{proof}[Proof of \Cref{thm:iw21}]
  As seen above, Problem \eqref{eq:ew2def} is equivalent to
  \begin{equation}\label{eq:ew2eqf}
  \sup_{\op \in \Pi(\mu,\nu)} \sup_{P \in \mathbb{V}_\di(\rset^\d)} \langle P, K_\op \rangle_{\spa{F}} \eqsp,
  \end{equation}
  where $ K_{\op} = \int xy^T \rmd\op(x,y) $. As in \cite{salmona2021gromov}, we use the necessary condition for $ \op $ 
  to be in $ \Pi(\mu,\nu) $
  that is that the covariance matrix $ \Sigma_\op $ of the law $ \op $ is a PSD matrix, or equivalently 
  that the Schur complement of $ \Sigma_\op $, i.e. $ \Sigma_1 - K_\op^T\Sigma_0^{-1}K_\op $ 
  is also a PSD matrix. This gives the following inequality:
  \begin{equation}
  \sup_{\op \in \Pi(\mu,\nu)} \sup_{P \in \mathbb{V}_\di(\rset^\d)} \langle P, K_\op \rangle_{\spa{F}} \leq \max_{K \ : \ \Sigma_1 - K^T\Sigma_0^{-1}K \in \mathbb{S}_{+}^\di} \max_{P \in \mathbb{V}_\di(\rset^\d)} \langle P, K \rangle_{\spa{F}}  \eqsp.
  \end{equation}
The rest of the proof is inspired 
from the proof of the closed-form of the $ W_2 $ between two Gaussians provided by \cite{OTGaussian}.
We want to solve the following constrained optimization problem
\begin{equation}\label{eq:pboptimPW2}
\min_{\substack{\Sigma_1 - K^T\Sigma_0^{-1}K \in \mathbb{S}_+^\di \\ P \in \mathbb{V}_\di(\rset^\d) }} 
 - 2\mathrm{tr}(P^TK) \eqsp.
\end{equation}
Using \Cref{lem:gaussianlem}, we can write $ \mathrm{tr}(P^TK) $ 
as a function of $ B_r $. This gives the following equivalent constrained optimization problem
\begin{equation}
\min_{B_r^TB_r = \Id_r,P^TP = \Id_\di} - 2\mathrm{tr}(P^T\Sigma_0^{\frac{1}{2}}B_r\Lambda_rU_r^T) \eqsp.
\end{equation}
The Lagrangian of this latter problem reads as
\begin{equation}
\mathcal{L}(B_r,P,C_0,C_1) = - 2\mathrm{tr}(P^T\Sigma_0^{\frac{1}{2}}B_r\Lambda_rU_r^T) + \mathrm{tr}(C_0(B_r^TB_r - \Id_r)) + \mathrm{tr}(C_1(P^TP - \Id_\di)) \eqsp,
\end{equation}
where $ C_0 \in \mathbb{S}^r $ and $ C_1 \in \mathbb{S}^\di $ 
are the Lagrange multipliers respectively 
associated with the constraints $ B_r^TB_r = \Id_r $ and $ P^TP = \Id_\di $. 
The first order condition gives 
\begin{equation}
\left\{ \begin{array}{ll} \Sigma_0^{\frac{1}{2}}PU_r\Lambda_r = B_rC_0 \\
        \Sigma_0^{\frac{1}{2}}B_r\Lambda_rU_r^T = PC_1 \eqsp. \end{array}\right. 
\end{equation}
Since $ \Sigma_0 $, $ P $, $ U_r $, and $ \Lambda_r $ are full rank, 
$ \Sigma_0^{\frac{1}{2}}PU_r\Lambda_r $ is of rank $ r $ and so
$ C_0 $ is also of rank $ r $. Thus we get that
\begin{equation}
B_r = \Sigma_0^{\frac{1}{2}}PU_r\Lambda_rC_0^{-1} \eqsp,
\end{equation}
and so 
\begin{equation}
B_r^TB_r = \Id_r = C_0^{-1}\Lambda_rU_r^TP^T\Sigma_0PU_r\Lambda_rC_0^{-1} \eqsp. 
\end{equation}
Thus,
\begin{equation}
C_0 = (\Lambda_rU_r^TP^T\Sigma_0PU_r\Lambda_r)^{\frac{1}{2}} \eqsp.
\end{equation}
On the other hand, by reinjecting the expression of 
$ B_r $ in the other first order condition we get
\begin{equation}
P^T\Sigma_0PU_r\Lambda_r(\Lambda_rU_r^TP^T\Sigma_0PU_r\Lambda_r)^{-\frac{1}{2}}\Lambda_rU_r^T = C_1  \eqsp.
\end{equation}
By multiplying this equation by itself we get
\begin{equation}
P^T\Sigma_0PU_r\Lambda^2_rU_r^T = C^2_1  \eqsp.
\end{equation}
Since $ C^2_1 $ is symmetric we get that $ P^T\Sigma_0P $ commutes with $ U_r\Lambda^2_rU_r^T $ and 
so $ \Sigma_1 - S $. Moreover, as before we have
\begin{align}
\mathrm{tr}(P^TK) &=  \mathrm{tr}(((\Sigma_1 - S)^{\frac{1}{2}}P^T\Sigma_0P(\Sigma_1 - S)^{\frac{1}{2}})^{\frac{1}{2}}) \\
&= \mathrm{tr}((\Sigma_1 - S)^{\frac{1}{2}}(P^T\Sigma_0P)^{\frac{1}{2}}) \eqsp.
\end{align}
Using the Courant-Fischer min-max theorem \cite[]{courant1920eigenwerte,fischer1905quadratische} to characterize the eigenvalues 
of $ \Sigma_1 - S $, see \cite[Proposition 7]{OTGaussian} for details, we get that $ \mathrm{tr}(P^TK) $ is maximized when $ S = 0 $ 
and so the problem is equivalent to the following problem 
\begin{equation}\label{eq:ew2}
\max_{\substack{P \in \mathbb{V}_\di(\rset^\d) \\ P^T\Sigma_0P\Sigma_1 = \Sigma_1P^T\Sigma_0P}} \mathrm{tr}(\hat{D}_1^{\frac{1}{2}}D_{0,P}^{\frac{1}{2}}) \eqsp,
\end{equation}
where $ (\hat{P}_1,\hat{D}_1) $ is any diagonalization 
of $ \Sigma_1 $ and $ D_{0,P} = \hat{P}^T_1P^T\Sigma_0P\hat{P}_1 $.
For all $ y \in \rset^\di $ we have
\begin{equation}
\alpha_\d\|y\|^2 \leq y^TP^T\Sigma_0Py \leq \alpha_1\|y\|^2 \eqsp,
\end{equation}
where $ \alpha_1,\dots,\alpha_\d $ are the eigenvalues of $ \Sigma_0 $ ordered in non-increasing order. Thus, 
denoting $ \lambda_1, \dots, \lambda_\di $ the eigenvalues of $ P^T\Sigma_0P $,
 we get that for all $ k \leq \di $,
\begin{equation}
\alpha_\d \leq \lambda_k \leq \alpha_1 \eqsp.
\end{equation}
Since we want to maximize $ \textstyle{\mathrm{tr}(\hat{D}_1^{\frac{1}{2}}D_{0,P}^{\frac{1}{2}})} $, 
we set the largest eigenvalue $ \lambda_1 $ of $ P^T\Sigma_0P $ to $ \alpha_1 $. 
We denote $ y_1 \in \rset^\di $ the eigenvector associated.  
We have $ y_1P^T\Sigma_0Py_1 = \alpha_1 $ and $ \|Py_1\| = \|y_1\| = 1 $ 
so using \Cref{lem:eigenvectors}, we get that $ \|Py_1\| $ is an eigenvector 
of $ \Sigma_0 $ associated with $ \alpha_1 $. Let $ \lambda_k $ and $ y_k $ be any other eigenvalue and 
its associated eigenvector in the orthonormal basis in which $ P^T\Sigma_0P $ is diagonal.
We have $ y_k^Ty_1 = 0 $ and so $ y_k^TP^TPy_1 = 0 $. Thus $ Py_k $
is orthogonal to $ Py_1 $. Since $ \|Py_k\| = 1 $, we get that $ Py_k $ is also 
an eigenvector of $ \Sigma_0 $ and so it exists $ i \leq \d-1 $
such that $ \lambda_k = y_k^TP^T\Sigma_0Py_k = \alpha_i $. Thus, we conclude that the
eigenvalues of the optimal 
$ P^T\Sigma_0P $ are the $ \di $ largest eigenvalues
of $ \Sigma_0 $. Moreover, $ \mathrm{tr}(\hat{D}_1^{\frac{1}{2}}D_{0,P}^{\frac{1}{2}})  $ is clearly maximized
when $ D_{0,P} $ and $ \hat{D}_1 $ have their eigenvalues sorted in the same order. We conclude then
that setting $ D_{0,P} = D_0^{(\di)} $ and $ \hat{D}_1 = D_1 $, where $ D_0 $ and $ D_1 $
are the diagonal matrices associated with the diagonalizations 
that sort the eigenvalues in non-increasing, maximizes the problem
and so it follows that
\begin{equation}
\max_{\substack{ \Sigma_1 - K^T\Sigma_0^{-1}K \in \mathbb{S}_+^\di \\ P \in \mathbb{V}_\di(\rset^\d)}} 2\mathrm{tr}(P^TK) = 2\mathrm{tr}({D_0^{(\di)}}^{\frac{1}{2}}D_1^{\frac{1}{2}}) \eqsp.
\end{equation}
Finally, observe that 
when setting $ K^* $ of the form 
\begin{equation}
K^* = P_0(\widetilde{I}_\di{D_0^{(\di)}}^\frac{1}{2}D_1^\frac{1}{2})^{[\d,\di]}P_1^T \eqsp,
\end{equation}
we have
\begin{equation}
\|K\|_{*} = \mathrm{tr}((K^{*T}K^*)^{\frac{1}{2}}) = \mathrm{tr}((D_0^{(\di)}D_1)^{\frac{1}{2}}) = \mathrm{tr}({D_0^{(\di)}}^{\frac{1}{2}}D_1^{\frac{1}{2}}) \eqsp.
\end{equation}
Moreover, observe that this is the solution of \Cref{eq:norm2} exhibited 
in \cite[Lemma 3.2]{salmona2021gromov}. Thus $ K^* $ 
is cleary in the feasible set and so is optimal.
By reinjecting the optimal value in the expression of $ EW_2(\mu,\nu) $, we get
\begin{equation}
EW_2^2(\mu,\nu) = \mathrm{tr}(D_0) + \mathrm{tr}(D_1) - 2\mathrm{tr}({D_0^{(\di)}}^{\frac{1}{2}}D_1^{\frac{1}{2}}) \eqsp.
\end{equation}
Furthermore, using the results of \cite{salmona2021gromov}, we get directly 
that the optimal plans $ \op^* $ are of the form $ (\Id_\d,T)_{\#}\mu $ with $ T $ linear of the form 
\begin{equation}
T = P_1\left(\widetilde{I}_\di D_1^{\frac{1}{2}} {D_0^{\di}}^{-\frac{1}{2}}\right)^{[\di,\d]}P_0^T
\end{equation}
Finally, observe that $ K^* $ admits as SVD 
$  P_0({D_0^{(\di)}}^\frac{1}{2}D_1^\frac{1}{2})^{[\d,\di]}\widetilde{I}_\di P_1^T $. 
For a given fixed $ \widetilde{I}_\di $, we get using \Cref{lem:maxsingvalue}, that 
the optimal $ P^* $ associated with $ K^* $ is  $ P^* = P_0\widetilde{I}_\di^{[\d,\di]}P_1^T $, 
which concludes the proof. 
\end{proof}

\subsection{Proof of \Cref{prop:deriv}}
\begin{proof}[Proof of \Cref{prop:deriv}]
  First note that in this proof, we denote $ \rset^{\d \times \di} $ 
  the set of matrices of size $ \d \times \di $ that we distinguish from the set $ \rset^{\d\di} $ of vector with $ \d \times \di $ coordinates.
  We set $ g \colon P \in \rset^{\d \times \di} \mapsto \Sigma_1^{\frac{1}{2}}P^T\Sigma_0P\Sigma_1^{\frac{1}{2}} $ 
  and $ h \colon Q \in \mathbb{S}^\di_+ \mapsto Q^{\frac{1}{2}} $  
  such that for all matrix $ P  $ of size $ \d \times \di $, we have
  \begin{equation}
  f(P) = \mathrm{tr}(h(g(P))) \eqsp. 
  \end{equation}
  For any matrix $ A \in \rset^{\d \times \di} $, we denote $ \mathrm{vec}(A) \in \rset^{\d\di} $ 
  the vector obtained by stacking the columns of $ A $. 
  Observe that, see \cite[]{magnus2019matrix} for details, 
  for any function $ \phi \colon \rset^{\d \times \di} \rightarrow \rset^{r \times s} $, the Jacobian matrix
  $ J[\phi] $ of $ \phi $ can be defined as, for all $ P \in \rset^{\d \times \di} $,
  \begin{equation}
  J[\phi](P) = \frac{\partial \mathrm{vec}(f(P))}{\partial \mathrm{vec}(P) } \eqsp.
  \end{equation}
  Moreover, observe that since $ f \colon \rset^{\d \times \di} \rightarrow \rset $, $ J[f][P] \in \rset^{\d\di}$ and 
  \begin{equation}
  \frac{\partial f(P)}{\partial P} = \mathrm{vec}^{-1}(J^T[f](P)) \eqsp,
  \end{equation}
  where $ \mathrm{vec}^{-1} $ is the inverse vector operator, i.e. such 
  that for any $ A \in \rset^{\d \times \di} $, $ \mathrm{vec}^{-1}(\mathrm{vec}(A)) = A$ . 
  Applying the chain rule to derive $ f $, we have 
  \begin{equation}
  J[f](P) = J[\mathrm{tr}]((h\circ g)(P)) J[h](g(P))J[g](P) \eqsp.
  \end{equation} 
  \begin{itemize}
  \item First, we compute $ J[g](P) $. It follows, using formula provided by \cite{petersen2008matrix} and \cite{magnus2019matrix}, 
  \begin{equation}
  \partial(\Sigma_1^{\frac{1}{2}}P^T\Sigma_0P\Sigma_1^{\frac{1}{2}}) = \Sigma_1^{\frac{1}{2}}\partial P^T \Sigma_0 P \Sigma_1^{\frac{1}{2}} + \Sigma_1^{\frac{1}{2}} P^T \Sigma_0 \partial P \Sigma_1^{\frac{1}{2}} ,
  \end{equation}
  and so 
  \begin{align}
  \partial \mathrm{vec}(\Sigma_1^{\frac{1}{2}}P^T\Sigma_0P\Sigma_1^{\frac{1}{2}}) &= (\Sigma_1^{\frac{1}{2}}P^T\Sigma_0 \otimes_K \Sigma_1^{\frac{1}{2}}) \partial \mathrm{vec}(P^T) 
      + ( \Sigma_1^{\frac{1}{2}} \otimes_K \Sigma_1^{\frac{1}{2}}P^T\Sigma_0  ) \partial \mathrm{vec}(P) \\
      &= (\Sigma_1^{\frac{1}{2}}P^T\Sigma_0 \otimes_K \Sigma_1^{\frac{1}{2}})K_{\d\d'}\partial \mathrm{vec}(P) + ( \Sigma_1^{\frac{1}{2}} \otimes_K \Sigma_1^{\frac{1}{2}}P^T\Sigma_0  ) \partial \mathrm{vec}(P) \\
      &= (I_{\di^2} + K_{\di^2})(\Sigma_1^{\frac{1}{2}} \otimes_K \Sigma_1^{\frac{1}{2}}P^T\Sigma_0) \partial \mathrm{vec}(P) \eqsp, 
  \end{align}
  where $ \otimes_K $ denotes the Kronecker product and for any $ r $, $ K_r $ 
  is the commutation matrix of size $ r \times r $, see \cite[]{magnus2019matrix} for details. Thus,
  \begin{equation}
  J[g](P) = (I_{\di^2} + K_{\di^2})(\Sigma_1^{\frac{1}{2}} \otimes_K \Sigma_1^{\frac{1}{2}}P^T\Sigma_0) \eqsp. 
  \end{equation}
  \item Now we compute $J[h](Q) $. Observe that we have for any $ Q \in \mathbb{S}^\di_+ $,
  \begin{equation}
  Q^{\frac{1}{2}}Q^{\frac{1}{2}} = Q \eqsp.
  \end{equation}
  Thus it follows, denoting $ s \colon Q \mapsto Q^{\frac{1}{2}} $, 
  \begin{equation}
  \partial s(Q) Q^{\frac{1}{2}} + Q^{\frac{1}{2}}\partial s(Q) =  \partial Q \eqsp. 
  \end{equation}
  This latter equation is a Sylvester equation with variable $ \partial s (Q) $, 
  which is equivalent to the following linear system:
  \begin{equation}
  (Q^{\frac{1}{2}} \oplus_K Q^{T\frac{1}{2}} )\partial \mathrm{vec}(s(Q)) = \partial \mathrm{vec}(Q) \eqsp, 
  \end{equation}
  where $ \oplus_K $ stands for the Kronecker sum. 
  If $ Q $ is non-degenerate, $ Q^{\frac{1}{2}} \oplus_K Q^{T\frac{1}{2}} $ is 
  also non-degenerate and so in that case
  \begin{equation}
  J[h](Q) = (Q^{\frac{1}{2}} \oplus_K Q^{T\frac{1}{2}} )^{-1} \eqsp.  
  \end{equation}
  \item Finally, it is easy to see that for $ R \in \rset^{\di \times \di} $ we have 
  \begin{equation}
  J[\mathrm{tr}](R) = \mathrm{vec}^T(\Id_\di).
  \end{equation}
  \end{itemize}
  Thus, denoting $ A = \Sigma_1^{\frac{1}{2}}P^T\Sigma_0P\Sigma_1^{\frac{1}{2}} $ 
  and observing that $ A $ is symmetric and
  full-rank when $ P $ is full-rank (since we supposed that $ \Sigma_0 $ and $ \Sigma_1 $ are full rank), it follows that for all full-rank matrix $ P $ of size $ \d \times \di $,
  \begin{equation}
  J^T[f](P) = (\Sigma_1^{\frac{1}{2}} \otimes_K \Sigma_0P\Sigma_1^{\frac{1}{2}})(I_{\di^2} + K_{\di^2})(A^{\frac{1}{2}} \oplus_K A^{\frac{1}{2}} )^{-1}\mathrm{vec}(\Id_\di) \eqsp, 
  \end{equation}
  where we used that $ K_{\di^2}  $ and $ (A \oplus_K A)^{-1} $ were symmetric. Observe now that 
  $ (A^{\frac{1}{2}} \oplus_K A^{\frac{1}{2}})^{-1}\mathrm{vec}(\Id_\di) = \mathrm{vec}(X) $, where $ X \in \rset^{\di \times \di} $
  is the unique solution of the following Sylvester equation
  \begin{equation}
  A^{\frac{1}{2}}X + XA^{\frac{1}{2}} = \Id_\di \eqsp. 
  \end{equation}
  Since $ A $ is symmetric, one can set $ A = QDQ^T $ where $ Q \in \mathbb{O}(\rset^\di) $ 
  and $ D $ is a diagonal matrix of size $ \di $.
  The Sylvester equation can be rewritten
  \begin{equation}
  D^{\frac{1}{2}}Y + YD^{\frac{1}{2}} = \Id_\di \eqsp,
  \end{equation}
  where $ Y = Q^TXQ $. Since $ A $ is full-rank, $ D $ is invertible 
  and it is easy to see that the unique solution of this latter 
  equation is $ Y = (1/2)D^{-\frac{1}{2}} $ and so $ X = (1/2)A^{-\frac{1}{2}} $ and thus
  \begin{equation}
  (A^{\frac{1}{2}} \oplus_K A^{\frac{1}{2}})^{-1}\mathrm{vec}(\Id_\di) = \frac{1}{2}\mathrm{vec}(A^{-\frac{1}{2}}) \eqsp. 
  \end{equation}
  Moreover, since $ A $ is symmetric, we have $ K_{\di^2}\mathrm{vec}(A^{-\frac{1}{2}}) = \mathrm{vec}(A^{-\frac{1}{2}}) $ and so it
  follows that 
  \begin{align}
  J^T[f](P) &= (\Sigma_1^{\frac{1}{2}} \otimes_K \Sigma_0P\Sigma_1^{\frac{1}{2}})\mathrm{vec}(A^{-\frac{1}{2}}) \\
      &= \mathrm{vec}(\Sigma_0P\Sigma_1^{\frac{1}{2}}A^{-\frac{1}{2}}\Sigma_1^{\frac{1}{2}}) \eqsp,
  \end{align} 
  which concludes the proof. 
  \end{proof}

\setcounter{figure}{0}
\setcounter{table}{0}
\setcounter{lemma}{0}
\makeatletter
\renewcommand{\thefigure}{C\arabic{figure}}
\renewcommand{\thetheorem}{C\arabic{theorem}}
\renewcommand{\thedefinition}{C\arabic{definition}}
\renewcommand{\thelemma}{C\arabic{lemma}}
\renewcommand{\thelemmaC}{C\arabic{lemmaC}}
\renewcommand{\thepropositionC}{C\arabic{propositionC}}
\renewcommand{\thesection}{C}
\renewcommand{\theremark}{C\arabic{remark}}
\renewcommand{\theproposition}{C\arabic{proposition}}
\renewcommand{\thecorollary}{C\arabic{corollary}}

\section{More details on Projection Wasserstein discrepancy}
\label{sec:addresultspw2}

In this section, we give more details on the difference between $ EW_2 $ and 
the OT distance introduced in \cite{cai2022distances} 
that we call here \emph{projection Wasserstein discrepancy}.
We recall that for $ \mu \in \mathcal{W}_2(\rset^\d) $ 
and $ \nu \in \mathcal{W}_2(\rset^\di)$ with $ \d \geq \di$, this OT distance is 
defined as 
\begin{equation}\label{eq:pw2}\tag{$ PW_2 $}
PW_2(\mu,\nu) =  \inf_{\phi \in \Gamma_\d(\rset^\di)} W_2(\phi_{\#}\mu,\nu)  \eqsp,
\end{equation}
where $ \Gamma_\d(\rset^\di) $ is the set of all affine mapping from $ \rset^\d $ to $ \rset^\di $ of the form $ \varphi(x) = P^T(x - b) $ with 
$ P \in \mathbb{V}_\di(\rset^\d) $ and $ b \in \rset^\d $. One key results of \cite{cai2022distances} is 
to show that $ PW_2 $ has the following equivalent formulation
\begin{equation}
PW_2(\mu,\nu) =  \inf_{\xi  \in \mathcal{W}^\nu_2(\rset^\d)} W_2(\mu,\xi) \eqsp,
\end{equation}
where $ \mathcal{W}^\nu_2(\rset^\d) $ is the subset of $ \mathcal{W}_2(\rset^\d) $  defined
as 
\begin{equation}
\resizebox{0.99\hsize}{!}{$\mathcal{W}^\nu_2(\rset^\d) = \ensembleLigne{ \xi \in \mathcal{W}_2(\rset^\d)}{ \text{ there exists }\phi(x) = P^T(x - b) \text{ with } P \in \mathbb{V}_\di(\rset^\d) \text{ and } b \in \rset^\di \text{ such that } \phi_{\#}\xi = \nu } \eqsp.$}
\end{equation}
Observe that this latter formulation is structurally different of $ EW_2 $ since 
for any isometry $ \phi \colon \rset^\di \rightarrow \rset^\d $, the distribution $ \phi_{\#}\nu $ 
is necessarily degenerate, whereas this is not the case for the distribution $ \xi $. The difference 
between $ EW_2 $ and $ PW_2 $ is illustrated in \Cref{fig:embed_projection}.

\begin{figure}[!ht]
 \centering
 \includegraphics[width=0.5\textwidth]{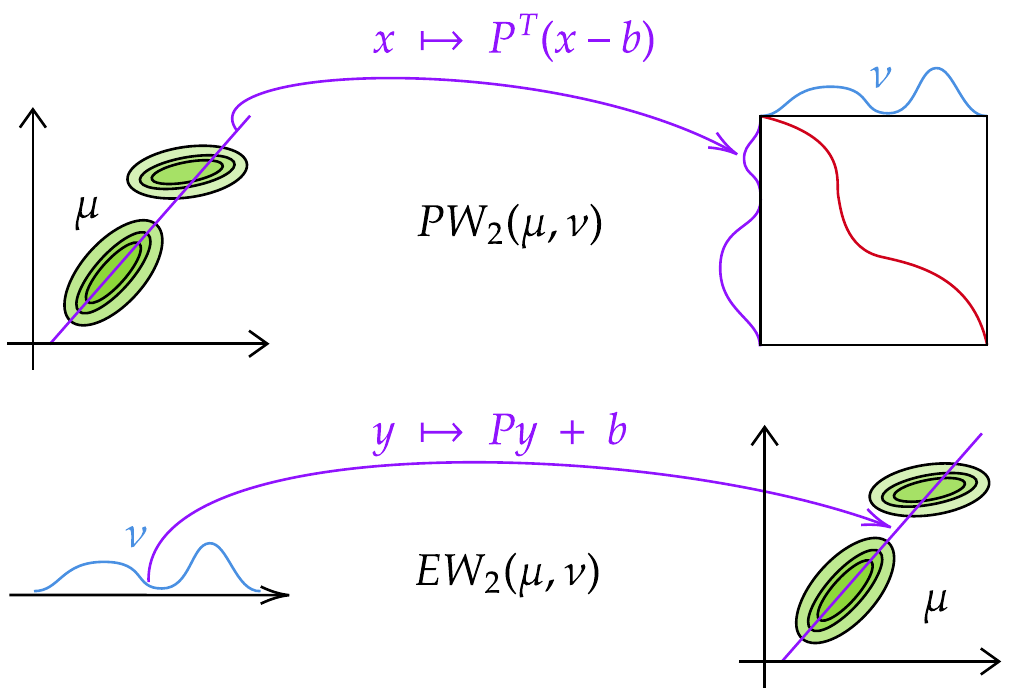}
  \caption{Link between $ PW_2 $ and $ EW_2 $ for two distributions $ \mu $ 
  and $ \nu $ respectively on $ \rset^2 $ and $ \rset $. In $ PW_2 $, $ \mu $
  is projected into $ \rset $ by a mapping of the form $ x \mapsto P^T(x-b) $. In 
  $ EW_2 $, $ \nu $ is transformed into a degenerate measure (lying on the 
  purple line) on $ \rset^2 $ with 
  an isometric mapping of the form $ y \mapsto Py + b $.}\label{fig:embed_projection}
 \end{figure}

To highlight even more the difference between $ EW_2 $ and $ PW_2 $, 
we derive an equivalent problem of Problem \eqref{eq:pw2}.  
Observe that in that case, the mapping $ \phi $ in \eqref{eq:pw2} is not an isometry 
since it is not injective. As a
result, the term that previously depended only on the marginal $ \mu $ in the developpement 
of the square of the Euclidean distance will now depend on $ P $. More precisely, this 
gives the following result. 

\begin{proposition}\label{prop:pw2eqf}
Let $ \mu \in \spa{W}_2(\rset^\d) $ and 
$ \nu \in \spa{W}_2(\rset^\di)$ and 
let suppose $ \d \geq \di$. Problem \eqref{eq:pw2} is equivalent to
\begin{equation}\label{eq:pw2eqf}
\inf_{\op \in \Pi(\bar{\mu},\bar{\nu})} \inf_{P \in \mathbb{V}_\di(\rset^\d)} \left( \mathrm{tr}(P^T\Sigma_x P) - 2\mathrm{tr}(P^T K_{\op}) \right) \eqsp,
\end{equation}
where $ \Sigma_x = \int_{\rset^\d \times \rset^\d}xx^T \rmd \bar{\mu}(x) $, $ K_{\op} = \int_{\rset^\d \times \rset^\di} xy^T \rmd \op(x,y) $,
and where $ \bar{\mu} $ and $ \bar{\nu} $ are the centered measures associated with $ \mu $ and $ \nu $.
\end{proposition}

\begin{proof}[Proof of \Cref{prop:pw2eqf}]
  First observe that using \Cref{lem:centerediw2}, we can consider without any loss generality 
  that $ \mu $ and $ \nu $ are centered and omit $ b $. Using \Cref{lem:mappingop}, it follows
  \begin{align}
    PW_2^2(\mu,\nu) &= \inf_{P \in \mathbb{V}_\di(\rset^\d)} \inf_{\op' \in \Pi(P^T_{\#}\mu,\nu)} \int_{\rset^\di \times \rset^\di} \|z - y\|^2 \rmd \op'(z,y) \\
    &=  \inf_{P \in \mathbb{V}_\di(\rset^\d)} \inf_{\op \in \Pi(\mu,\nu)} \int_{\rset^\d \times \rset^\di} \|P^Tx - y\|^2 \rmd \op(x,y) \\ 
    &=  \inf_{P \in \mathbb{V}_\di(\rset^\d)} \left( \int_{\rset^\d} \|P^Tx\|^2 \rmd \mu(x) + \int_{\rset^\di}\|y\|^2 \rmd \nu(y) - 2\sup_{\op \in \Pi(\mu,\nu)} \int_{\rset^\d \times \rset^\di} (P^Tx)^Ty \rmd \op(x,y) \right) \eqsp,
    \end{align}
    and so the problem is equivalent to 
    \begin{equation}
    \inf_{P \in \mathbb{V}_\di(\rset^\d)} \left( \int_{\rset^\d} \|P^Tx\|^2 \rmd \mu(x) - 2\sup_{\op \in \Pi(\mu,\nu)} \int_{\rset^\di \times \rset^\di} (P^Tx)^Ty \rmd \op(x,y) \right) \eqsp,
    \end{equation}
    which is itself equivalent to \eqref{eq:pw2eqf}, which concludes the proof. 
  \end{proof}
Observe that Problem \eqref{eq:pw2eqf} can be interpreted as a regularization in $ P $ of the $ EW_2 $ problem since we have seen above that this latter was equivalent 
to the following problem
\begin{equation}
\sup_{\op \in \Pi(\bar{\mu},\bar{\nu})} \sup_{P \in \mathbb{V}_\di(\rset^\d)} \mathrm{tr}(P^TK_\op) \eqsp.
\end{equation}
It can also be interpreted as a $ W_2 $ problem between $ \nu $ and a
measure $ \mu' $ which has a different second-order moment than $ \mu $.

\end{document}